\documentclass[twoside]{article}

\usepackage[preprint]{style}
\usepackage[round]{natbib} 

\usepackage[T1]{fontenc}
\usepackage{url}            
\usepackage{microtype}      
\usepackage{xcolor}         

\usepackage{amssymb}
\usepackage{amsthm}
\usepackage{mathtools}   
\usepackage{amsfonts}       
\usepackage[most]{tcolorbox}
\usepackage{thmtools}      
\usepackage{dsfont}        
\usepackage{nicefrac}       

\usepackage{algorithm}     
\usepackage{algorithmic}   
\usepackage{booktabs}   
\usepackage{listings}
\usepackage{multirow}   
\usepackage{tikz}
\usepackage{pgfplots}
\pgfplotsset{compat=1.18}
\usepackage{wrapfig}
\usepackage{array}
\usepgfplotslibrary{fillbetween}
\usepackage{subcaption}
\usepackage{enumitem}
\usepackage{array}
\usepackage{makecell}
\newcolumntype{C}[1]{>{\centering\arraybackslash}p{#1}}

\usepackage{hyperref}       
\hypersetup{hidelinks}
\usepackage[capitalize,noabbrev]{cleveref}

\theoremstyle{plain}
\newtheorem{theorem}{Theorem}[section]

\newtheorem{lemma}[theorem]{Lemma}
\newtheorem{claim}[theorem]{Claim}

\newtheorem{corollary}[theorem]{Corollary}
\theoremstyle{definition}
\newtheorem{definition}[theorem]{Definition}
\newtheorem{assumption}{Assumption}
\theoremstyle{remark}
\newtheorem{remark}[theorem]{Remark}

\newcount\Comments  
\newcommand{\kibitz}[2]{\ifnum\Comments=1{\textcolor{#1}{{#2}}}\fi}

\DeclareMathOperator*{\argmin}{arg\,min}

\begin{document}

%

%
\runningauthor{Shrivastava, Akshintala, Reddy,  Manwani, Gujar, Ghalme}

\twocolumn[

\aistatstitle{Linear Strategic Classification with Endogenous Improvements}

\vspace{-0.5cm}
\aistatsauthor{Siddharth Shrivastava$^1$ \And  Mahvith Akshintala$^1$ \And   B Vamsha Vardhan Reddy$^2$}
\vspace{0.1cm}
\aistatsauthor{Naresh Manwani$^2$ \And Sujit Gujar$^2$ \And Ganesh Ghalme$^1$}
\vspace{0.2cm}
\aistatsaddress{ $^1$IIT Hyderabad \And $^2$IIIT Hyderabad} ]

\begin{abstract}
  Classical models of strategic classification treat agents' responses to a deployed classifier as cosmetic: features change, but true labels do not. We study an improvement-aware variant in which strategic responses can change outcomes: agents choose post-deployment features strategically at a cost, and labels are then drawn from a fixed conditional outcome law. We formalize this problem for linear classifiers under a single-index qualification model and linear-decomposable costs. Under these assumptions, agents' best responses reduce to a closed-form shift of the decision boundary, and the resulting improvement-aware classification problem is learnable with linear sample complexity. As post-deployment labels are 
  unavailable a priori, we begin by establishing a generalization bound with oracle access to the outcome law. One of our main contributions is a practical plug-in algorithm that replaces the oracle with an estimate, together with a bound on its improvement-aware error. We also show that the optimal classifier for standard strategic classification is a better proxy for the improvement-aware objective than the non-strategic Bayes-optimal classifier.
  Experiments on semi-synthetic and synthetic data support these findings and the plausibility of our assumptions.
\end{abstract}

\section{Introduction}
{\em Strategic classification}~\citep{hardt2016strategic} studies prediction settings in which agents respond to a deployed classifier by modifying observable features at a cost. Such behavior arises in credit scoring, admissions \citep{hardt2016strategic, dong2018strategic}, and hiring \citep{kleinberg2020classifiers}, applicants may adjust credit utilization, obtain certifications, or enroll in test-preparation programs to improve their chances of receiving a favorable decision. Classical strategic classification treats these responses as manipulation: the feature vector changes, but the underlying label or qualification remains fixed.

This immutability assumption is often too restrictive. In many applications, the same actions that change observable features can also improve the underlying outcome of interest. Coursework may improve academic readiness, professional training may increase job productivity, and financial discipline may improve repayment probability. Thus, a classifier can both induce strategic feature changes and incentivize genuine improvement. We call this setting {\em improvement-aware} strategic classification.

In this paper, we study improvement-aware strategic classification with {\em endogenous improvement} for binary label prediction tasks. In particular, we assume that the original population is 
governed by a fixed conditional outcome law
$\eta(x)=\mathbb{P}(Y=1\mid X=x)$. The deployed classifier affects agents by changing their incentives over feasible feature modifications. Thus, the classifier {induces a change in} the marginal distribution of features, but not the structural relation between features and outcomes. In this sense, improvement is endogenous; it is not externally imposed by the analyst but arises because agents move to regions of the feature space that, under the original data-generating process, are associated with better outcomes.


Our framework makes two assumptions: $1)$ the class-probability function $\eta$ follows a single-index qualification model, $\eta(x) = g({w^{\star}}^\top x)$, for a non-decreasing link $g: \mathbb{R} \to [0,1]$ and non-negative weights $w^{\star} \in \mathbb{R}^d_{\geq 0}$, and $2)$ the manipulation cost is linear-decomposable, i.e., the total cost is a weighted sum of feature-wise improvement costs. Assumption \ref{ass:singleIndex} guarantees that non-strategic Bayes optimal classifier is linear, and subsumes the monotonicity of $\eta$ in features. Monotonicity of the outcome in improvement-relevant features is a standard assumption in the improvement-aware strategic classification literature \citep{kleinberg2020classifiers, miller2020strategic}. The single-index qualification model also encompasses many class-probability models commonly used in the literature, including linear and sigmoidal models \citep{foster2018uniform, audibert2007fast, yara2026nonparametric}. We will discuss these assumptions in detail in Section \ref{sec:setting}.


\textbf{Paper Organization.} Section \ref{sec:setting} introduces the model and assumptions. Section \ref{sec:structural} gives preliminary structural results. Section \ref{sec:Algorithm} studies learnability under oracle access to  $\eta$, and Section~\ref{sec:alg_strat_imp_aware} gives an oracle-free plug-in algorithm with generalization guarantees. Section \ref{sec:exp} reports experiments on synthetic and semi-synthetic data, and Section~\ref{sec:outcome_eta_uncertainty} proposes a decision rule for setting where $\eta$ is difficult to learn.


\subsection{Related work}


\citet{hardt2016strategic} formalized strategic classification as a Stackelberg game in which agents manipulate observable features at a cost. Follow-up work studies the social cost of manipulation \citep{milli2019social, liu2022strategic, hossain2025strategic}, information asymmetry \citep{tang2021linear, ghalme2021strategic, shao2023strategic}, sequential settings \citep{cohen2023sequential, shekhtman2024strategic} and differentiable surrogates for end-to-end training \citep{levanon2021strategic, eilat2022strategic}. Classical strategic classification treats manipulation as feature movement that leaves the label unchanged, an assumption we relax.


Another line lets agents genuinely improve rather than game, closely related to performative prediction \citep{perdomo2020performative, mendler2020stochastic}. Work here designs classifiers that incentivize effort over gaming \citep{kleinberg2020classifiers, ahmadi2022classification, levanon2022generalized}, studies fairness under improvement \citep{efthymiou2025desirable, alhanouti2025anticipating}, grounds the improvement/gaming distinction in causal feature structure \citep{miller2020strategic, horowitz2023causal, gois2026role}, and analyzes dynamics, persistence, and welfare \citep{miller2020strategic, horowitz2023causal, gois2026role}. \citet{chen2023learning} and
\citet{xie2024learning} ($1$-D thresholds with active experimentation) study incentivizing improvement under deterministic qualification change.

Learning-theoretic work on strategic classification establishes PAC-learnability and strategic VC-dimension results \citep{sundaram2023pac, lechner2022learning}, incentive-aware sample complexity bounds \citep{zhang2021incentive}, online regret bounds \citep{ahmadi2023fundamental}, and learnability gaps between strategic and standard PAC learning \citep{cohen2024learnability}. Closest to ours, \citet{attias2025pac} prove PAC bounds under exogenously specified improvement sets. \citet{sharma2026conservative} characterize proper learnability with improvements, give positive results for Euclidean-ball improvement sets, and extend to bounded label noise and online learning. In both, improvement sets are fixed independently of the classifier, post-improvement labels follow a fixed target concept (up to label noise in \cite{sharma2026conservative}), and ties are broken in the worst case.

Prior work thus treats qualification change as deterministic or classifier-independent \citep{chen2023learning, xie2024learning, attias2025pac, sharma2026conservative}. It also separates improvement from manipulation through exogenous structure, such as causal graphs, improvable/non-improvable feature labels, or classifier-independent improvement sets, which demands domain knowledge often unavailable in practice. We instead model post-response qualification as $\eta$ evaluated at the adapted features, that is the observational conditional $\mathbb{P}(Y=1 \mid X)$ remains
valid under change in agent features. Under this modeling choice, manipulation and probabilistic improvement coexist in a single model without a prespecified causal/non-causal feature partition, and the manipulation region depends on the deployed classifier. This yields a PAC learning guarantee under $\eta$, with sample complexity linear in the dimension $d$ (Corollary~\ref{coro:pac_learnability}), and generalization guarantee (Theorem~\ref{thm:alg_guarantee}), with no direct analogue in \cite{attias2025pac, sharma2026conservative}.

\section{Settings}
\label{sec:setting} 
Let $\mathcal{X} \subseteq \mathbb{R}^d$ denote the feature space, and   $\mathcal{D}$ be a distribution over $\mathcal{X}$. 
Also let  
\(\mathcal X_0:=\operatorname{supp}( \mathcal D)\subseteq\mathcal X\) denote the support of the original 
feature distribution. Agents may report feature vectors in a feasible report set 
\(\mathcal A\subseteq\mathcal X\), with \(\mathcal X_0\subseteq\mathcal A\). 
For simplicity we take \(\mathcal A\) to be common across agents; the model can also allow 
agent-dependent feasible sets \(\mathcal A(x)\subseteq\mathcal X\). We define the class-probability function as $\eta(x) = \mathds{P}(y=1|x)$, where the label for $x$ is drawn as $y \mid x  \sim \operatorname{Bernoulli}(\eta(x))$. The training dataset $D$ consists of unaltered points $(x_i, y_{i})_{i=1}^n$, where the $x_{i}$'s are independently sampled from $\mathcal{D}$ and $y_{i} $ are independent realizations of $\eta(x_i)$. For brevity, we write $y_x$ to denote the label $y$ conditioned on $x$, throughout the paper. Further, we define $\texttt{err}(\cdot)$ as the standard expected $0-1$ loss, $\texttt{err}(f) = \mathbb{E}_{x \sim \mathcal{D}}[\mathds{1}( f(x) \neq y_x )]$.


\begin{definition}[Agent Best Response]
For a classifier \(f\), an agent with original feature vector \(x\in\mathcal X_0\) reports 
\[
    x^f\in\arg\max_{z\in\mathcal A}\{\beta f(z)-c(x,z)\}, 
\]
$\beta>0$ is the benefit from positive classification and $c(x,x')$ is the cost of changing features from $x$ to $x'$.  
\end{definition}

Strategic reporting induces a mismatch between the training distribution and the distribution observed at deployment. A strategically robust classifier accounts for this best-response-induced distribution shift and adjusts its decision rule accordingly. The goal in standard strategic classification setting is to minimize the strategic error given below. 

\begin{definition}[Strategic Error~\citep{hardt2016strategic}] 
The strategic error of a classifier $f \in \mathcal{F}$ is defined as 
\begin{equation}
\texttt{err}_s(f) = \mathbb{E}_{x \sim \mathcal{D}} \left[ \mathds{1}[f(x^f) \ne y_x] \right].
\end{equation} 
\end{definition}
This definition implicitly assumes labels are immutable. We adapt the above objective to reflect the resulting change in qualification/labels in improvement-aware settings as follows. 

\begin{definition}[Strategic Improvement Error] 
\label{def: imp aware error}
The strategic improvement error of a classifier $f \in \mathcal{F}$ is defined as 
\begin{equation}
\texttt{err}_{\mathrm{Imp}}(f) = \mathbb{E}_{x \sim \mathcal{D}} \left[ \mathds{1}[f(x^f) \ne y_{x^f}] \right].
\end{equation} 
\end{definition}  

We consider endogenous improvement in qualification.  This modeling choice captures strategic behavior as inducing true qualification changes. In particular,  the probabilistic gain in qualification is a function of distance from classifier boundary, in contrast to models that either partition agents exogenously into gaming versus improving types \citep{horowitz2023causal, chen2023learning} or treat improvement as deterministic \citep{attias2025pac}. To characterize when and how feature improvements translate to qualification improvements, we assume the following condition on the model.

\begin{assumption}[Single-index qualification model]
\label{ass:singleIndex}
There exists a continuous non-decreasing function \(g:\mathbb R\to[0,1]\) and a vector
\(w^\star\in\mathbb R_{\ge 0}^d \setminus \{0\}\) such that \(\eta(x)=g({w^\star}^\top x)\) for all
\(x\in\mathcal X\). Moreover, the Bayes threshold 
$b^\star:=\inf\{t\in\mathbb R:g(t)\ge 1/2\}$ is finite.
\end{assumption}

The single-index qualification model ensures  that the true data-generating process respects a natural alignment between feature improvements and qualification improvements, enabling us to design classifiers that leverage this structure. This class of models covers several standard probabilistic classification models, including logistic, probit, and other monotone-link generalized linear models; the sigmoid link is a common example \citep{foster2018uniform, yara2026nonparametric}. Furthermore, Assumption \ref{ass:singleIndex}   is particularly justified in domains where features represent positive attributes, skill investments, or substantive credentials rather than mere signals. We validate this assumption on real-world datasets in Appendix \ref{app:exp}. Next we observe that  Assumption \ref{ass:singleIndex} guarantees that the optimal classifier is linear. 

\begin{restatable}[Existence of a linear Non-strategic Bayes optimal classifier]{observation}{ObsOne}\label{obs:bayes-linear}
Under $0$--$1$ loss in the Single-index qualification model (Assumption~\ref{ass:singleIndex}), there exists a Non-strategic Bayes optimal classifier of the form $f^\star(x)=\mathbf 1\{{w^{\star}}^\top x\ge b^\star\} 
$ 
for some threshold $b^\star\in \mathbb R$.
In particular, if 
$b^\star := \inf\{t\in\mathbb R: g(t)\ge 1/2\},$ 
then the classifier above is Bayes-optimal. 
\end{restatable}

This observation (proof in Appendix~\ref{app:setting}) allows us to restrict the space of classifiers to the set of linear classifiers with non-negative weights and hence, throughout the paper, we will assume that 
\[
    \mathcal{F} \coloneqq \bigl\{  \mathcal{X} \to \{0,1\}: f_{w, b}(z) = \mathds{1}(w^{\top}z \geq b) \bigr\},
\]
where $w_i \geq 0$ for all $i \in [d]$, $w \neq 0$, and $b \in \mathbb{R}$, writing $f$ and $f_{w,b}$ interchangeably.

\begin{definition}[Linear-decomposable cost] 
A cost function $c: \mathcal X_0  \times \mathcal{A} \to \mathbb{R}_+$ is decomposable if there exist $\alpha_i > 0$ for $i \in [d]$  such that 
\[
    c(x, z) = \sum_{i=1}^d \alpha_i  (z_i - x_i )_+. 
\]
Here,  $(z_i - x_i )_+ = \max(0, z_i - x_i)$.
\end{definition}

The cost model is a weighted one-sided $\ell_1$ action cost, where the parameters $\alpha_i$ denotes the per-unit cost or  difficulty of improving a given  feature $i$.  It is closely related to the linear/separable manipulation costs used in strategic classification \citep{hardt2016strategic, ahmadi2021strategic} and to coordinate-wise action costs used in algorithmic recourse \citep{ustun2019actionable}.
This structure captures settings where each feature can be improved independently at some per-unit cost, \emph{e.g.}, hours spent on test preparation, tuition for certification courses, or effort invested in building work experience.

\begin{assumption} \label{assump: cost function} 
Cost function is linear-decomposable. 
\end{assumption}

First, we prove certain structural results for the improvement-aware strategic classification. 

\section{Preliminaries}
\label{sec:structural}
All proofs for this section are deferred to Appendix~\ref{app:structural}. Throughout, we assume feasibility: for every $x \in \mathcal{X}_0$ and $i \in [d]$, $x + te_i \in \mathcal{A}$ for all $t \in [0, \beta / \alpha_i]$. Our first result characterizes the agents' best response to a linear classifier and shows that, with non-negative weights, agents who respond do so by improving their feature values, \emph{i.e.}, the best-response feature vector is coordinate-wise no smaller than the original one.





\begin{restatable}{lemma}
    {bestresponse}\label{lem:best_response} 
    For $f_{w,b} \in \mathcal{F}$, let $S(w) \coloneqq \beta \max_{j \in [d]} w_j / \alpha_j$ and let $j^{\star}$ be the smallest index in $\arg \max_{j \in [d]} w_{j} / \alpha_j$. Then for every $x \in \mathcal{X}_0$, the selected best response is 
    \[
        x^f = \begin{cases}
            x + \dfrac{b-w^{\top}x}{w_{j^{\star}}}{e_{j^{\star}}}, & \text{if } b - S(w) \leq w^{\top}x < b \\ 
            x, & \text{otherwise.}
        \end{cases}
    \]
\end{restatable}

Lemma~\ref{lem:best_response} shows that, under $f_{w,b}$ exactly the agents with $w^{\top}x \in [ b - S(w), b)$ move, and they land on the decision boundary; since $w \neq 0$, $w_{j^{\star}} > 0$ and $x^f$ is well defined. This is the linear-decomposable analogue of the Euclidean-cost best response used by \citet{levanon2022generalized}; here, agents move along the most cost-effective coordinate.  

\begin{remark}
Lemma~\ref{lem:best_response} implies $x^f \ge x$ coordinate-wise: agents never decrease a feature. Since $\eta$ is non-decreasing in each coordinate (Assumption~\ref{ass:singleIndex}), $\eta(x^f) \geq \eta(x)$, \emph{i.e.}, best responses can only raise an agent's qualification, which is a central theme of the paper.

\end{remark}


\begin{restatable}{theorem}{strategiccloseform}
\label{thm:strategic-shift}
    For $f_{w,b} \in \mathcal{F}$, let $f_s(x)=\mathbf 1\{w^\top x\ge b_s\}$ with $b_s \coloneqq b + \beta \max_{j \in [d]} w_j / \alpha_j$. Then $f_s(x^{f_s}) = f_{w,b}(x)$ for every $x \in \mathcal{X}_0$. Consequently, $\texttt{err}_s(f_s)=\texttt{err}(f_{w,b})$. 
\end{restatable}

Each unit of cost raises the score $w^{\top}x$ by at most $S(w)/\beta$, attained by moving along $e_{j^{\star}}$. Since positive classification is worth $\beta$, an agent can profitably raise its score by at most $S(w)$. Under $f_s$ (the strategic counterpart of $f_{w,b}$), agents with $w^{\top}x \geq b$ can therefore reach $b_s = b + S(w)$, while agents with $w^{\top}x < b$ cannot, so $f_s(x^{f_s}) = f_{w,b}(x)$.


\begin{restatable}{corollary}{RelationFAndFStar}
\label{cor:f_s_relation with f*} 
Let $f_{w^{\star}, b^{\star}}\in \arg \min_{f \in \mathcal{F}} \texttt{err}(f)$, i.e., it is a non-strategic Bayes-optimal classifier (Observation 2.4). Then, its shift $f^{\star}_s \coloneqq f_{w^{\star}, b^{\star}_s}$ is strategic optimal among linear classifiers in $\mathcal{F}$.


\end{restatable}



In one dimension, with $\eta$ strictly increasing, let $\theta^{\star}$, $\theta_s$, $\theta_{\mathrm{Imp}}$ be the thresholds minimizing $\texttt{err}$, $\texttt{err}_s$,  $\texttt{err}_{\mathrm{Imp}}$, respectively. Then $\theta^{\star} \leq \theta_{\mathrm{Imp}} \leq \theta_s$. (Lemma~\ref{thm:imp in r1})




We proved important structural properties of an optimal improvement-aware classifier. However, learning an optimal improvement-aware classifier requires post-deployment labels, which depend on agents' responses $x^f$ (hence on $f$ itself) and are unobserved before deployment. In Section~\ref{sec:Algorithm}, we assume oracle access to these labels, as in a controlled deployment; learning remains nontrivial since $\mathcal{D}$ is unknown and both $x^f$ and its label depend on the classifier $f$, and we establish PAC Learnability. In Section~\ref{sec:alg_strat_imp_aware}, we replace the oracle with a plug-in estimate $\widehat{\eta}$ and derive generalization bounds depending on its estimation error.

\section{PAC Learnability with Oracle Access}
\label{sec:Algorithm}

We establish PAC learnability of improvement-aware strategic classification under oracle access by reducing it to uniform convergence over an induced loss class. Throughout, we fix the deterministic tie-breaking rule of Section~\ref{sec:structural}, so that the best response $x^f$ is uniquely determined for every $x \in \mathcal{X}_0$ and $f \in \mathcal{F}$. In this section, we assume oracle access to the true conditional law $\eta(\cdot)$, which allows the learner to construct, for any $f \in \mathcal{F}$, the improvement-aware labels of Definition~\ref{def:oracle_labels}. 

\begin{definition}[Oracle labels]
\label{def:oracle_labels}
    Let $(x_i, y_i)_{i=1}^n$ be $\text{i.i.d.}$, with $x_i \sim \mathcal{D}$ supported on $\mathcal{X}_0$ and $y_i \mid x_i \sim \text{Bernoulli}(\eta(x_i))$, and let $u_1, \dots, u_n \sim \text{Unif}[0,1]$ be $\text{i.i.d.}$, independent of $(x_i, y_i)_{i=1}^n$, drawn once and shared across all $f \in \mathcal{F}$. For $f \in \mathcal{F}$, define 
    \[
        y^f_i \coloneqq \mathds{1}\{y_i =1 \vee u_i \le p^f (x_i)\},
    \]
    where $p^f(x) \coloneqq \tfrac{\eta(x^f) - \eta(x)}{1 - \eta(x)} \in [0,1]$ is the true improvement probability (Appendix~\ref{sec:imp_probability_formulation}). 
\end{definition}
Since $u_i$ is independent of $(x_i,y_i)$, $\mathbb{P}(y^f_i = 1 \mid x_i) = \eta(x_i) + (1-\eta(x_i))p^f (x_i) = \eta(x^f _i)$; that is for every fixed $f$, the posterior distribution of $y^f_i$ given $x_i$ is \text{Bernoulli}$(\eta(x_i^f))$. Hence the empirical improvement-aware error 
\[
    \widehat{\texttt{err}}_\mathrm{Imp} (f) \coloneqq \frac{1}{n} \sum_{i=1}^n \mathds{1}\left[ f(x^f_i) \neq y^f_i \right]
\]
satisfies $\mathbb{E}[\widehat{\texttt{err}}_\mathrm{Imp} (f)]= \texttt{err}_\mathrm{Imp}(f)$ for any fixed $f \in \mathcal{F}$.


\begin{restatable}{theorem}{Convergence}
 \label{thm:uniform_convergence}
Let the sample and oracle labels be as in Definition~\ref{def:oracle_labels}, suppose the induced oracle-loss class $\mathcal{H} \coloneqq \{h_f : f \in \mathcal{F}\}$, where 
\[
    h_f(x,y,u) \coloneqq \mathds{1}[f(x^f) \neq \mathds{1}\{y=1 \vee u \le p^f(x)\}]
\]
has finite $\text{VC}$ dimension $\tilde{d} \ge 1$. Then for every $\delta \in (0,1)$ and $n \ge \tilde{d}$, with probability at least $1-\delta$ over sample $(x_i, y_i, u_i)_{i=1}^n$, 
\[
    \sup_{f\in\mathcal F}\left|\texttt{err}_{\mathrm{Imp}}(f)-\widehat{\texttt{err}}_{\mathrm{Imp}}(f)\right| \le \varepsilon_{n}(\delta),
\]
where $\varepsilon_{n}(\delta) = \frac{2}{\sqrt n} + 4\sqrt{\frac{2\tilde{d}\log(en/\tilde{d})} {n}} + \sqrt{\frac{\log(1/\delta)}{2n}}.$
\end{restatable} 

The main difficulty is that the post-response label is random and its law depends on $f$. Independently resampling labels across classifiers introduces classifier-specific randomness, which can prevent uniform convergence over infinite classes. We therefore couple the labels across classifiers through a single auxiliary variable $u_i$ per point (Definition~\ref{def:oracle_labels}), so that the loss of every $f$ is a fixed function $h_f$ of the same augmented point $z_i = (x_i, y_i, u_i)$, and uniform convergence can be controlled through the $\text{VC}$ dimension of $\mathcal{H}$. The coupling preserves each classifier's conditional post-response label distribution, $y^f_i \mid x_i \sim \text{Bernoulli}(\eta(x^f_i))$, so $\mathbb{E}[h_f(z)] = \texttt{err}_{\mathrm{Imp}}(f)$, while $\widehat{\texttt{err}}_{\mathrm{Imp}}(f) = \tfrac{1}{n}\sum_{i=1}^n h_f(z_i)$ by definition. The problem thus reduces to uniform convergence for the binary class $\mathcal{H}$ over $\text{i.i.d.}$ sample: symmetrization together with Massart's lemma and the Saucer's lemma bounds the expected deviation by the first two terms of $\varepsilon_n(\delta)$, and McDiarmid's inequality yields the stated high-probability term. The full proof is in Appendix~\ref{app:PAC}. The following lemma yields $\tilde{d} = \mathcal{O}(d)$ in our setting.

\begin{restatable}{lemma}{VCdimensioninducedoracle} 
\label{lem:vc_dim_induced_oracle} 
Let $\mathcal{X} \subseteq \mathbb{R}^d$ and let $\mathcal{F}$ be the class of linear classifiers of Section~\ref{sec:setting}. Under Assumptions~\ref{ass:singleIndex} and~\ref{assump: cost function}, the induced oracle-loss class $\mathcal{H}$ of Theorem~\ref{thm:uniform_convergence} has $\text{VC}$ dimension $\tilde{d} < 11(d+1) = \mathcal{O}(d)$.
\end{restatable}
Write $\theta = (w,b)$. Under Assumption~\ref{assump: cost function} and the tie-breaking rule, an agent who moves does so 
along the single coordinate $j^{\star} = \min \arg \max_j w_j / \alpha_j$, so the parameter space splits into $d$ regions $\Theta_j = \{\theta: j^{\star} = j \}$ that do not depend on the input. On each $\Theta_j$, Section~\ref{sec:structural} gives the best response in closed form: whether the agent moves, and the post-response decision $f(x^f)$, are determined by the signs of two affine functions of $\theta$. For the oracle label, monotonicity of the link $g$ (Assumption~\ref{ass:singleIndex}) rewrites the event $\{u \leq p^f(x)\}$ as the threshold condition ${w^{\star}}^{\top}x^f \geq \kappa(x,\tau)$, where $\kappa$ depends only on the input through the generalized inverse of $g$. Substituting the closed form $x^f$ turn this into the sign of a third affine function of $\theta$, while the degenerate cases (no move, $\eta(x) =1$, or infinite $\kappa$) reduce to events that do not depend on $\theta$. Hence, on each region, $h_f(x,y,u)$ is a fixed Boolean function of the signs of three affine functions of $\theta \in \mathbb{R}^{d+1}$ with coefficients depending only on $(x,u)$, and a sign-pattern counting bound (Lemma~\ref{lem:counting}) gives $\tilde{d} < 11(d+1)$. Notably, the bound does not depend on the form of $g$; only its monotonicity and continuity are used. The full proof is in Appendix~\ref{app:PAC}.

\begin{corollary}[Agnostic PAC learnability]
    \label{coro:pac_learnability} 
    Under the conditions of Theorem~\ref{thm:uniform_convergence} and Lemma~\ref{lem:vc_dim_induced_oracle}, let $\widehat{f} \in \arg \min_{f \in \mathcal{F}} \widehat{\texttt{err}}_{\mathrm{Imp}}(f)$. Then for every $\delta \in (0,1)$ and $n \ge 11(d+1)$, with probability at least $1-\delta$, 
    \[
        \texttt{err}_{\mathrm{Imp}}(\widehat{f}) \leq \inf_{f \in \mathcal{F}} \texttt{err}_{\mathrm{Imp}}(f) + 2 \epsilon_{n}(\delta), 
    \]
    where $2 \varepsilon_{n}(\delta) = \mathcal{O}\left(\sqrt{\frac{d \log(n/d)}{n}} + \sqrt{\frac{\log(1/\delta)}{n}} \right)$.
\end{corollary}

\begin{remark}
Oracle access to $\eta$ specifies how labels respond to strategic movement but not the population distribution $\mathcal{D}$; since $\texttt{err}_{\mathrm{Imp}}(f)$ couples $f$ with the induced response $x^f$, its minimizer over $\mathcal{F}$ depends upon $\mathcal{D}$ and must be learned from samples. In Section~\ref{sec:alg_strat_imp_aware}, when \(\eta\) is replaced by an estimator \(\widehat\eta\), the resulting plug-in analysis incurs an additional approximation term of order \( \sup_{z\in\mathcal A}|\widehat\eta(z)-\eta(z)|\).
\end{remark}

\begin{algorithm}[tb]
    \caption{\textsc{Strat-Imp-Aware}}
    \label{alg:train_loop}
    \small
    \begin{algorithmic}[1]
        \raggedright
        \STATE {\textbf{Input:}} Training Set $\mathcal{D}_{tr}$, Validation Set $\mathcal{D}_{val}$, Epochs $T$, Learning rate $\gamma$, Utility $\beta$, Cost-coeff. $\vec{\alpha}$, Differentiable Estimate $\widehat\eta$, Guard $\delta_g$
        \STATE Initialize $w \in \mathbb{R}^d_{\ge 0} \setminus \{0\},~b \in \mathbb{R}$; $\mathcal{L}_{best} \gets \infty$
        \FOR{$epoch = 1 \dots T$}
            \FOR{batch $(X, Y) \in \mathcal{D}_{tr}$}
                \STATE $\widehat{\boldsymbol{\pi}} \gets 
                \textsc{Sim-Imp-Prob}(X,w,b,\beta,\vec{\alpha},
                \widehat{\eta},\delta_g)$
                \STATE $\mathcal{L}_{batch} \gets \textsc{Strat-Imp-Loss}(X, Y, \widehat{\boldsymbol{\pi}}, w,b,\beta, \vec{\alpha})$
                \STATE $w \gets \bigl(w - \gamma \nabla_w \mathcal{L}_{\mathrm{batch}}\bigr)_+$ 
                \STATE $b \gets b - \gamma \nabla_b \mathcal{L}_{\mathrm{batch}}$
            \ENDFOR
            \STATE $\mathcal{L}_{val} \gets 0$; $m \gets 0$ 
            \FOR{batch $(X, Y) \in \mathcal{D}_{val}$}
                \STATE $\widehat{\boldsymbol{\pi}} \gets \textsc{Sim-Imp-Prob}(X, w, b, \beta, \vec{\alpha}, \widehat{\eta}, \delta_g)$
                \STATE \mbox{$\mathcal{L}_{val} \gets \mathcal{L}_{val}\,+$} $\textsc{Strat-Imp-Loss}(X, Y, \widehat{\boldsymbol{\pi}}, w, b, \beta, \vec{\alpha})$ 
                \STATE $m \gets m+1$
            \ENDFOR
            \STATE $\mathcal{L}_{val} \gets \mathcal{L}_{val} / m$
            \IF{$\mathcal{L}_{val} < \mathcal{L}_{best}$}
                \STATE $\mathcal{L}_{best} \gets \mathcal{L}_{val}$, $w_{best} \gets w, b_{best} \gets b$
            \ENDIF
        \ENDFOR
        \STATE \textbf{return} Best Model $(w_{best}, b_{best})$
    \end{algorithmic}
\end{algorithm}

\begin{algorithm}[tb]
    \caption{\textsc{Sim-Imp-Prob}}
    \label{alg:sim_oracle}
    \small
    \begin{algorithmic}[1]
        \raggedright
        \STATE \textbf{Input:} Feature $x$, Model $(w,b)$, Utility $\beta$, Cost-coeff. $\vec{\alpha}$, Differentiable Estimate $\widehat\eta$, Guard $\delta_g$
        \STATE $j^{\star} \gets \min \arg\max_{j \in [d]} w_j / \alpha_j$ 
        \STATE $S(w) \gets \beta\, w_{j^{\star}} / \alpha_{j^{\star}}$ 
        \STATE $x^f \gets x$  
        \IF{$0 < b - w^{\top}x \leq S(w)$} 
            \STATE $x^f \gets x + \tfrac{b- w^{\top}x}{w_{j^{\star}}}e_{j^{\star}}$
        \ENDIF
        \STATE \textbf{return} $\widehat{\pi}_{\mathrm{Imp}} \gets \mathrm{clamp}_{[0,1]}\!\left(
        \tfrac{\widehat{\eta}(x^{f}) - \widehat{\eta}(x)}
              {\max\{\,1 - \widehat{\eta}(x),\ \delta_g\,\}}\right)$
    \end{algorithmic} 
\end{algorithm}

\begin{algorithm}[tb]
    \caption{\textsc{Strat-Imp-Loss}}
    \label{alg:strat_imp_error}
    \small
    \begin{algorithmic}[1]
        \raggedright
        \STATE {\textbf{Input:}} Point $(x,y)$, Imp-probability $\widehat{\pi}_{\mathrm{Imp}}$, Model $(w,b)$, Utility $\beta$, Cost-coeff. $\vec{\alpha}$
        \STATE $q \gets y + (1 - y) \widehat{\pi}_{\mathrm{Imp}}$
        \STATE $S(w) \gets \beta \max_{j \in [d]} \left( \tfrac{w_j}{\alpha_j} \right)$
        \STATE $\text{score} \gets w^{\top}x - b + S(w)$
        \STATE \textbf{return} $\ell \gets q(1 - \text{score})_+\,+\, (1 -q)(1+\text{score})_+$
    \end{algorithmic}
\end{algorithm}

\section{Oracle free Learning with Estimated \texorpdfstring{\raisebox{1.2pt}{\boldmath$\eta$}}{eta}} 
\label{sec:alg_strat_imp_aware}
We now propose an improvement-aware strategic learning algorithm \textsc{Strat-Imp-Aware} that targets the improvement-aware classification error under endogenous feature manipulation and label improvement, without the need of oracle access. The algorithm extends   \mbox{\citet{levanon2022generalized}} in three ways: $(i)$ decomposable costs replace $L_2$ costs, yielding a closed-form strategic shift $S(w)= \beta \cdot \max_{j \in [d]} w_j/\alpha_j$; $(ii)$ label improvement is modeled through each agent's improvement probability under the deployed model, computed from a plug-in estimate $\widehat{\eta}$; and $(iii)$ training differentiates through both the strategic shift and these estimated improvement probabilities.

\textbf{Improvement probability.} Given a sample $\{x_i, y_i\}_{i=1}^n$ from the pre-deployment distribution and the current model $(w,b)$, $\textsc{Sim-Imp-Prob}$ (Algorithm~\ref{alg:sim_oracle}) first computes the best response $x^f$, which under non-negative weights and decomposable costs has a closed form (Section~\ref{sec:structural}): if $0< b - w^{\top}x \leq S(w)$, the agent moves along the most cost-effective coordinate $j^{\star}$ exactly onto the decision boundary; otherwise (already positive, or too far to move profitably) $x^f = x$. Since $\eta$ is unknown, $\textsc{Sim-Imp-Prob}$ uses a plug-in estimate $\widehat{\eta}$ to compute and return the improvement probability $\widehat{\pi}_{\mathrm{Imp}} = \mathrm{clamp}_{[0,1]}\!\left( (\widehat{\eta}(x^{f}) - \widehat{\eta}(x)) / \max\{\,1 - \widehat{\eta}(x),\ \delta_g\,\}\right)$. $\textsc{Strat-Imp-Loss}$ (Algorithm~\ref{alg:strat_imp_error}) then forms $q = y + (1-y)\widehat{\pi}_{\mathrm{Imp}}$, the probability that post-response label is positive: a point with $y=1$ remains positive, and a point with $y=0$ becomes positive with probability $\widehat{\pi}_{\mathrm{Imp}}$. Equivalently, $q = \mathbb{P}(y_{\mathrm{Imp}} =1 \mid x,y)$ for the improvement-aware label $y_{\mathrm{Imp}} = \mathds{1}(y = 1 \vee u \leq \widehat{\pi}_{\mathrm{Imp}})$ with $u \sim \text{Unif}[0,1]$. With $\eta$ in place of $\widehat{\eta}$, $\mathbb{E}[q \mid x] = \eta(x^f)$ under Assumption~\ref{ass:singleIndex} whenever $1 - \eta(x) \geq \delta_g$, so that $y_{\mathrm{Imp}} \mid x \sim \text{Bernoulli}(\eta(x^f))$ (Appendix~\ref{sec:imp_probability_formulation}).

\textbf{Loss.} $\textsc{Strat-Imp-Loss}$ applies a hinge surrogate to the shifted score $M(x) = w^{\top}x -b + S(w)$; since $f(x^f) = \mathds{1}(w^{\top}x \geq b- S(w))$, the shift accounts for strategic responses without computing $x^f$ inside the loss. Rather than sampling $y_{\mathrm{Imp}}$, we take the expectation over $u$ in closed form: 
\begin{equation}
\label{eq:exp_strat_imp_loss}
    \mathbb{E}_{u}\left[ (1 - \tilde{y}M(x))_+ \right] = q(1-M(x))_+\,+\,(1-q)(1+M(x))_+
\end{equation}
where $\tilde{y} = 2y_{\mathrm{Imp}} -1$. This is the conditional expectation over $u$ of the sampled-label empirical risk analyzed in Section~\ref{sec:generation_bounds}: it has the same population risk, deterministic function of $(w,b)$ and removes the sampling noise in $u$.

\textbf{Training.} 
At each update, $q$ is recomputed from $(w,b)$, and gradients propagate through the differentiable estimator $\widehat{\eta}$. For a manipulating point,
\[
x^f=x+\frac{b-w^\top x}{w_{j^\star}}e_{j^\star},
\quad
\frac{\partial x^f_{j^\star}}{\partial b}=\frac{1}{w_{j^\star}},
\quad
\nabla_w x^f_{j^\star}=-\frac{x^f}{w_{j^\star}}.
\]
The coordinate $j^\star$ and manipulation status are held fixed during differentiation, yielding exact gradients wherever the corresponding branch is differentiable. 
On the movement branch, the bound $|\partial\ell/\partial q|\leq 2\beta w_{j^\star}/\alpha_{j^\star}$ cancels the inverse factor in the response derivatives; elsewhere, $q=y$ locally.
The loss is continuous across manipulation boundaries but may jump across the tie hyperplanes $\{w_j/\alpha_j=w_k/\alpha_k\}$, where the selected response coordinate changes. Algorithm~\ref{alg:train_loop} uses projected (sub)gradient steps on $\mathbb{R}_{\geq0}^d\times\mathbb{R}$, as required by Theorem~\ref{thm:strategic-shift}, and returns the model with the lowest validation loss across epochs, which also guards against degradation from crossing the tie hyperplanes. Algorithms~\ref{alg:sim_oracle} and~\ref{alg:strat_imp_error} are applied row-wise, with per-point losses averaged within each batch and validation loss averaged across batches. Validation improvement probabilities are recomputed under the current model, and $(\cdot)_+$ is applied entrywise. More details are given in Appendix~\ref{app:exp}. 

\begin{lemma}[Gradient bound, Informal]
\label{lem:grad_bound_main}
Under mild boundedness conditions, the gradients of the loss of Algorithm~\ref{alg:strat_imp_error} are bounded by a constant that depends neither on the classifier's parameters nor on the training data, so they do not blow up as the weights shrink. 

\end{lemma}

The formal statement, its assumptions, and the proof are given in Appendix~\ref{sec:generation_bounds}. Lemma~\ref{lem:grad_bound_main} keeps the iterates bounded: since per-point (sub)gradients are bounded by constants $G_w, G_b$ that do not depend on the training sample, and the projection onto $\mathbb{R}^d_{\ge 0}$ is non-expansive, after $N$ updates every iterate satisfies $\|w\|_2 \leq \|w^{(0)}\|_2  + \gamma N G_w$ and $|b| \leq |b^{(0)}| + \gamma N G_b$. Hence the returned model lies in $\mathcal{W}_{k,B} \coloneqq \{(w,b): w \in \mathbb{R}^d_{\ge 0}, \|w \|_2 \le k, |b| \le B \}$ with $k \coloneqq \|w^{(0)}\|_2  + \gamma N G_w$ and $B \coloneqq |b^{(0)}| + \gamma N G_b$, both fixed before training; our generalization analysis is uniform over this class. 

We next derive a generalization bound for \textsc{Strat-Imp-Aware} when the outcome law $\eta$ is unknown. Whereas Theorem~\ref{thm:uniform_convergence} assumes oracle access to $\eta$, Algorithm~\ref{alg:train_loop} replaces $\eta$ by an independently trained estimator $\widehat{\eta}$: it  uses $\widehat{\eta}$ to compute each agent's improvement probability $\widehat{\pi}_{\mathrm{Imp}}$ under the deployed model (Algorithm~\ref{alg:sim_oracle}), and minimizes the plug-in strategic improvement loss, a hinge surrogate that weights each point by the resulting probability of a positive post-response label (Algorithm~\ref{alg:strat_imp_error}). The bound separates the resulting error into a generalization term and a plug-in estimation term. Throughout, $\widehat{\eta} = g(\widehat{w}^{\top} \cdot)$ with $\widehat{w} \in \mathbb{R}^d_{\ge 0}$ is estimated independently of the training sample, of size $n$, and the guard satisfies $\delta_g \leq 1 - \sup_{z \in \mathcal{A}} \widehat{\eta}(z)$. Let $r \coloneqq \sup_{z \in \mathcal{A}}\|z\|_2< \infty$, 
$\rho_{\mathrm{Imp}}:=r+\frac{\beta}{\min_{j\in[d]}\alpha_j}, C_{k,B} \coloneqq k \rho_{\mathrm{Imp}} + B,  
 \varepsilon_\eta:=\sup_{z\in\mathcal A}|\widehat\eta(z)-\eta(z)|$.


\begin{restatable}{lemma}{dfracUpper} 
\label{lem:eff_dim}
Let $\mathfrak{d} \coloneqq \mathrm{Pdim}(\Pi_{k,B}) + \mathrm{Pdim}(\mathcal{M}_{k,B})$, where $\mathrm{Pdim}$ denotes pseudo-dimension \cite{anthony2009neural} and $\Pi_{k,B}, \mathcal{M}_{k,B}$ be the classes of maps $x \mapsto \widehat{\pi}_{\mathrm{Imp}}(x; w,b)$ and $x \mapsto M_{w,b}(x)$ over $ \mathcal{W}_{k,B}$. Then $\mathfrak{d} < 12(d+1)$, and the loss class of Algorithm~\ref{alg:strat_imp_error} over $\mathcal{W}_{k,B}$ has the covering numbers of a class of pseudo-dimension $\mathfrak{d}$, up to absolute constants.
\end{restatable}

Detailed proof of Lemma~\ref{lem:eff_dim} is given in Appendix~\ref{sec:generation_bounds}.  



\begin{restatable}{theorem}{AlgGuarantee}
\label{thm:alg_guarantee} 

Let $(\tilde{w},\tilde{b}) \in \mathcal{W}_{k,B}$ be the model returned by \textsc{Strat-Imp-Aware}.  Under Assumptions~\ref{ass:singleIndex} and~\ref{assump: cost function}, for every \(\delta\in(0,1)\), with probability at least \(1-\delta\),
\begin{align*}
    \texttt{err}_{\mathrm{Imp}}(f_{\tilde{w},\tilde{b}})
  &\le
  \widehat L_{\mathrm{hinge}}^{\widehat q}(\tilde{w},\tilde{b})
 + c (1 + C_{k,B})\sqrt{\mathfrak{d}/n} \nonumber \\
&\quad + (1+ C_{k,B})\sqrt{\frac{\log(1/\delta)}{2n}}
+ 4C_{k,B}\varepsilon_\eta,  
\end{align*}
where $\mathfrak{d}$ is as in Lemma~\ref{lem:eff_dim} and $c>0$ is an absolute constant.
\end{restatable} 
\begin{proof}[Proof overview of Theorem~\ref{thm:alg_guarantee}] The proof combines three bounds, each uniform over $\mathcal{W}_{k,B}$.  

\smallskip
\textbf{1.} By Theorem~\ref{thm:strategic-shift}, $f_{w,b}(x^f) = \mathds{1}(M_{w,b}(x) \geq 0)$ with $|M_{w,b}(x)| \leq C_{k,B}$, so $\texttt{err}_{\mathrm{Imp}}$ is bounded by a hinge risk weighted by $\eta(x^f)$. \\ 
\textbf{2.} The weights $q_{w,b}$ used by Algorithm~\ref{alg:strat_imp_error} have conditional mean within $2\varepsilon_{\eta}$ of $\eta(x^f)$ (Lemma~\ref{lem:inducedlabeldeviation} in Appendix~\ref{sec:generation_bounds}), so, replacing $\eta(x^f)$ by conditional mean of $q_{w,b}$ in  true loss $L^{\eta}_{\text{hinge}}$ costs at most $4C_{k,B}\varepsilon_{\eta}$ and $\widehat{L}^{\widehat{q}}$ is an unbiased estimate of the resulting risk.  \\ 
\textbf{3.} Since the loss is Lipschitz in $\widehat{\pi}_{\mathrm{Imp}}$ and $M_{w,b}$, Lemma~\ref{lem:eff_dim} controls the covering numbers of the loss class, and Dudley's entropy integral and McDiarmid's inequality give the two $\sqrt{\cdot / n}$ terms.
\end{proof}
Detailed proof of Theorem~\ref{thm:alg_guarantee} is given in Appendix~\ref{sec:generation_bounds}.


\begin{remark}
    Beyond the empirical risk $\widehat{L}^{\widehat{q}}_{\mathrm{hinge}}$, the bound consists of two terms of order $\mathcal{O}(n^{-1/2})$, with $\mathfrak{d} = \mathcal{O}(d)$, and the estimation term $4C_{k,B}\varepsilon_{\eta}$. The estimator enters the guarantee only through $\varepsilon_{\eta}$, and not through the loss used to obtain it: any independently estimated $\widehat{\eta}$ with $\varepsilon_{\eta} = \mathcal{O}(m^{-1/2})$ yields the same conclusion. We give explicit theoretical upper bound  under additional conditions in  Appendix~\ref{sec:oracle_free} for completeness. 
\end{remark}

\section{ Experiments}
\label{sec:exp}
 We now validate our theoretical results and assumptions, and evaluate performance of our proposed $\textsc{Strat-Imp-Aware}$ on both synthetic and semi-synthetic data. Our code is available \href{https://github.com/Siddharth-Shrivastava7/Linear-Strategic-Classification-with-Endogenous-Improvements}{[here]}.

\begin{table*}[t]
\centering
\small
\setlength{\tabcolsep}{4pt}
\renewcommand{\arraystretch}{1}
\begin{tabular}{lp{3cm}cccc}
\toprule
\textbf{Dataset} & \textbf{Metric} 
& \textsc{SVM} 
& \textsc{SERM} 
& \textsc{Strat-Imp-Aware}  
& \citet{attias2025pac} \\
\midrule
\multirow{3}{*}{\textbf{Adult}}
& Manip. agents   & $\textbf{1168}^{\dagger}$ & $197^{\dagger}$ & 147 & 	183\\
& Improved agents & $\textbf{153}^{\dagger}$  & $128^{\dagger}$  & 104  & 73 \\
& Imp Error (\%) $\downarrow$ & $62.40^{\dagger}$ & $\textbf{15.50}^{\dagger}$ & 15.80 & 17.90 \\
\midrule
\multirow{3}{*}{\textbf{HELOC}}
& Manip. agents   & 141 & 44 & \textbf{92} & 66 \\
& Improved agents & 35  & 17  & \textbf{47} & 25 \\
& Imp Error (\%) $\downarrow$ & 38.90 & 32.50 & \textbf{22.0} & 25.70 \\
\midrule
\multirow{3}{*}{\textbf{Law School}}
& Manip. agents   & $62^{\dagger}$ & $13^{\dagger}$ & \textbf{69} & 60 \\
& Improved agents & $31^{\dagger}$  & $9^{\dagger}$  & \textbf{41} & 31 \\
& Imp Error (\%) $\downarrow$ & $5.28^{\dagger}$ & $30.07^{\dagger}$ & \textbf{4.71} & 5.48 \\
\midrule
\multirow{3}{*}{\textbf{ACS Income}}
& Manip. agents   & \textbf{4091} & 794 & 2640 & 135 \\
& Improved agents & 1330 & 725  & \textbf{2310} & 85 \\
& Imp Error (\%) $\downarrow$ & 40.60 & 19.50 & \textbf{6.4} & 36.1 \\
\midrule
\multirow{3}{*}{\textbf{Synthetic}}
& Manip. agents  & \textbf{412} & 38 & 221 & 293 \\
& Improved agents & 148 & 35 & 192 & \textbf{238} \\
& Imp Error (\%) $\downarrow$ & 36.60 & 21.10 & \textbf{6.30} & 8.10 \\
\bottomrule
\end{tabular}
\vspace{2mm}
\caption{\small
\textsc{Strat-Imp-Aware} achieves the lowest strategic improvement error and induces more responses and label improvements than \textsc{SERM} on four of five datasets. Relative to \textsc{SVM}, it induces fewer responses on four of five datasets and achieves a higher improvement rate among responding agents on all five. This  suggests that improvement-aware training encourages feature changes that enhance qualifications. ${\dagger}$  class-weighted results addressing severe class imbalance.
}
\label{tab:merged_cls}
\end{table*}

\begin{figure*}[ht]
    \centering

    \begin{tikzpicture}
    \begin{axis}[
        hide axis,
        xmin=0, xmax=1, ymin=0, ymax=1,
        width=\linewidth, height=1.5cm, 
        scale only axis,
        legend columns=4,
        legend style={
            draw=gray!60,
            fill=white,
            font=\small, 
            /tikz/every even column/.append style={column sep=0.2cm},
            at={(0.5,0)}, anchor=north, 
        },
    ]
    \addlegendimage{color=blue!70!black, mark=*, mark size=2pt, line width=1pt}
    \addlegendentry{Non-strategic Classifier (SVM)}
    \addlegendimage{color=orange!90!black, mark=square*, mark size=1.8pt, line width=1pt}
    \addlegendentry{Strategic Classifier (SERM)}
    \addlegendimage{color=red!80!black, mark=diamond*, mark size=2.2pt, line width=1pt}
    \addlegendentry{\citet{attias2025pac}}
    \addlegendimage{color=green!50!black, mark=triangle*, mark size=2.2pt, line width=1pt}
    \addlegendentry{\textsc{Strat-Imp-Aware}}
    \end{axis}
    \end{tikzpicture}

    \par \medskip
    \begin{subfigure}[b]{0.19\linewidth}
        \centering
        \begin{tikzpicture}
        \begin{axis}[
             width=1.37\linewidth, height=4cm,
             xmin=-200, xmax=6500, ymin=0.12, ymax=0.76,
             xtick={0,3000,6000},
             ytick={0.2,0.4,0.6},
             scaled x ticks=false,
             scaled y ticks=false,
             xticklabel style={/pgf/number format/fixed, /pgf/number format/1000 sep={}, font=\tiny},
             yticklabel style={/pgf/number format/fixed, /pgf/number format/precision=2, font=\tiny, rotate=45},
             minor tick num=0, grid=major, grid style={dashed, gray!40},
             ylabel={Strategic Improvement Error},
             ylabel style={font=\scriptsize, yshift=-6pt, overlay},
        ]

        \addplot[draw=none, name path=au, forget plot] coordinates {(100,0.734) (250,0.716) (500,0.717) (1000,0.675) (2000,0.678) (4000,0.651) (6000,0.642)};
        \addplot[draw=none, name path=al, forget plot] coordinates {(100,0.632) (250,0.633) (500,0.640) (1000,0.628) (2000,0.599) (4000,0.606) (6000,0.634)};
        \addplot[blue!20, fill opacity=0.5, forget plot] fill between[of=au and al];
        \addplot[color=blue!70!black, mark=*, mark size=1pt, line width=0.7pt]
        coordinates {(100,0.683) (250,0.674) (500,0.675) (1000,0.650) (2000,0.629) (4000,0.637) (6000,0.637)};

        \addplot[draw=none, name path=bu, forget plot] coordinates {(100,0.236) (250,0.173) (500,0.175) (1000,0.167) (2000,0.191) (4000,0.164) (6000,0.171)};
        \addplot[draw=none, name path=bl, forget plot] coordinates {(100,0.158) (250,0.156) (500,0.148) (1000,0.154) (2000,0.150) (4000,0.144) (6000,0.154)};
        \addplot[orange!25, fill opacity=0.5, forget plot] fill between[of=bu and bl];
        \addplot[color=orange!90!black, mark=square*, mark size=0.8pt, line width=0.7pt]
        coordinates {(100,0.200) (250,0.165) (500,0.163) (1000,0.162) (2000,0.163) (4000,0.156) (6000,0.161)};

        \addplot[draw=none, name path=nnu, forget plot] coordinates {(100,0.239) (250,0.211) (500,0.206) (1000,0.205) (2000,0.208) (4000,0.219) (6000,0.209)};
        \addplot[draw=none, name path=nnl, forget plot] coordinates {(100,0.171) (250,0.162) (500,0.170) (1000,0.189) (2000,0.178) (4000,0.177) (6000,0.180)};
        \addplot[red!20, fill opacity=0.5, forget plot] fill between[of=nnu and nnl];
        \addplot[color=red!80!black, mark=diamond*, mark size=1.5pt, line width=0.7pt]
        coordinates {(100,0.207) (250,0.193) (500,0.194) (1000,0.196) (2000,0.190) (4000,0.201) (6000,0.194)};

        \addplot[draw=none, name path=cu, forget plot] coordinates {(100,0.243) (250,0.238) (500,0.237) (1000,0.236) (2000,0.235) (4000,0.233) (6000,0.234)};
        \addplot[draw=none, name path=cl, forget plot] coordinates {(100,0.162) (250,0.179) (500,0.177) (1000,0.196) (2000,0.173) (4000,0.167) (6000,0.164)};
        \addplot[green!25, fill opacity=0.5, forget plot] fill between[of=cu and cl];
        \addplot[color=green!50!black, mark=triangle*, mark size=1.2pt, line width=0.7pt]
        coordinates {(100,0.214) (250,0.216) (500,0.221) (1000,0.224) (2000,0.220) (4000,0.206) (6000,0.192)};
        \end{axis}
        \end{tikzpicture}
        \caption{Adult}
        \label{fig:scimp_adult}
    \end{subfigure}%
    \hfill%
    \begin{subfigure}[b]{0.19\linewidth}
        \centering
        \begin{tikzpicture}
        \begin{axis}[
             width=1.37\linewidth, height=4cm,
             xmin=-200, xmax=6500, ymin=0.205, ymax=0.46,
             xtick={0,3000,6000},
             ytick={0.25,0.35,0.45},
             scaled x ticks=false,
             scaled y ticks=false,
             xticklabel style={/pgf/number format/fixed, /pgf/number format/1000 sep={}, font=\tiny},
             yticklabel style={/pgf/number format/fixed, /pgf/number format/precision=2, font=\tiny, rotate=45},
             minor tick num=0, grid=major, grid style={dashed, gray!40},
        ]

        \addplot[draw=none, name path=au, forget plot] coordinates {(100,0.453) (250,0.422) (500,0.414) (1000,0.442) (2000,0.412) (4000,0.414) (6000,0.403)};
        \addplot[draw=none, name path=al, forget plot] coordinates {(100,0.362) (250,0.352) (500,0.358) (1000,0.358) (2000,0.372) (4000,0.383) (6000,0.385)};
        \addplot[blue!20, fill opacity=0.5, forget plot] fill between[of=au and al];
        \addplot[color=blue!70!black, mark=*, mark size=1pt, line width=0.7pt]
        coordinates {(100,0.411) (250,0.392) (500,0.388) (1000,0.402) (2000,0.396) (4000,0.396) (6000,0.395)};

        \addplot[draw=none, name path=bu, forget plot] coordinates {(100,0.362) (250,0.383) (500,0.309) (1000,0.292) (2000,0.298) (4000,0.307) (6000,0.358)};
        \addplot[draw=none, name path=bl, forget plot] coordinates {(100,0.278) (250,0.259) (500,0.259) (1000,0.263) (2000,0.270) (4000,0.267) (6000,0.290)};
        \addplot[orange!25, fill opacity=0.5, forget plot] fill between[of=bu and bl];
        \addplot[color=orange!90!black, mark=square*, mark size=0.8pt, line width=0.7pt]
        coordinates {(100,0.320) (250,0.314) (500,0.286) (1000,0.275) (2000,0.287) (4000,0.284) (6000,0.326)};

        \addplot[draw=none, name path=nnu, forget plot] coordinates {(100,0.290) (250,0.290) (500,0.276) (1000,0.251) (2000,0.282) (4000,0.270) (6000,0.280)};
        \addplot[draw=none, name path=nnl, forget plot] coordinates {(100,0.239) (250,0.253) (500,0.220) (1000,0.237) (2000,0.239) (4000,0.245) (6000,0.259)};
        \addplot[red!20, fill opacity=0.5, forget plot] fill between[of=nnu and nnl];
        \addplot[color=red!80!black, mark=diamond*, mark size=1.5pt, line width=0.7pt]
        coordinates {(100,0.264) (250,0.270) (500,0.257) (1000,0.243) (2000,0.255) (4000,0.262) (6000,0.267)};

        \addplot[draw=none, name path=cu, forget plot] coordinates {(100,0.294) (250,0.276) (500,0.257) (1000,0.233) (2000,0.226) (4000,0.226) (6000,0.233)};
        \addplot[draw=none, name path=cl, forget plot] coordinates {(100,0.249) (250,0.251) (500,0.222) (1000,0.220) (2000,0.216) (4000,0.218) (6000,0.220)};
        \addplot[green!25, fill opacity=0.5, forget plot] fill between[of=cu and cl];
        \addplot[color=green!50!black, mark=triangle*, mark size=1.2pt, line width=0.7pt]
        coordinates {(100,0.268) (250,0.265) (500,0.242) (1000,0.229) (2000,0.224) (4000,0.223) (6000,0.225)};
        \end{axis}
        \end{tikzpicture}
        \caption{HELOC}
        \label{fig:scimp_heloc}
    \end{subfigure}%
    \hfill%
    \begin{subfigure}[b]{0.19\linewidth}
        \centering
        \begin{tikzpicture}
        \begin{axis}[
             width=1.37\linewidth, height=4cm,
             xmin=-200, xmax=6500, ymin=0.02, ymax=0.41,
             xtick={0,3000,6000},
             ytick={0.1,0.2,0.3,0.4},
             scaled x ticks=false,
             scaled y ticks=false,
             xticklabel style={/pgf/number format/fixed, /pgf/number format/1000 sep={}, font=\tiny},
             yticklabel style={/pgf/number format/fixed, /pgf/number format/precision=2, font=\tiny, rotate=45},
             minor tick num=0, grid=major, grid style={dashed, gray!40},
        ]

        \addplot[draw=none, name path=au, forget plot] coordinates {(100,0.065) (250,0.073) (500,0.061) (1000,0.057) (2000,0.060) (4000,0.053) (6000,0.054)};
        \addplot[draw=none, name path=al, forget plot] coordinates {(100,0.053) (250,0.047) (500,0.051) (1000,0.051) (2000,0.053) (4000,0.048) (6000,0.046)};
        \addplot[blue!20, fill opacity=0.5, forget plot] fill between[of=au and al];
        \addplot[color=blue!70!black, mark=*, mark size=1pt, line width=0.7pt]
        coordinates {(100,0.059) (250,0.057) (500,0.056) (1000,0.054) (2000,0.056) (4000,0.051) (6000,0.051)};

        \addplot[draw=none, name path=bu, forget plot] coordinates {(100,0.394) (250,0.393) (500,0.349) (1000,0.329) (2000,0.318) (4000,0.341) (6000,0.362)};
        \addplot[draw=none, name path=bl, forget plot] coordinates {(100,0.218) (250,0.261) (500,0.221) (1000,0.252) (2000,0.239) (4000,0.252) (6000,0.256)};
        \addplot[orange!25, fill opacity=0.5, forget plot] fill between[of=bu and bl];
        \addplot[color=orange!90!black, mark=square*, mark size=0.8pt, line width=0.7pt]
        coordinates {(100,0.285) (250,0.319) (500,0.276) (1000,0.277) (2000,0.280) (4000,0.290) (6000,0.299)};

        \addplot[draw=none, name path=nnu, forget plot] coordinates {(100,0.070) (250,0.084) (500,0.074) (1000,0.064) (2000,0.069) (4000,0.065) (6000,0.061)};
        \addplot[draw=none, name path=nnl, forget plot] coordinates {(100,0.057) (250,0.044) (500,0.047) (1000,0.047) (2000,0.046) (4000,0.051) (6000,0.049)};
        \addplot[red!20, fill opacity=0.5, forget plot] fill between[of=nnu and nnl];
        \addplot[color=red!80!black, mark=diamond*, mark size=1.5pt, line width=0.7pt]
        coordinates {(100,0.063) (250,0.063) (500,0.059) (1000,0.055) (2000,0.053) (4000,0.055) (6000,0.055)};

        \addplot[draw=none, name path=cu, forget plot] coordinates {(100,0.061) (250,0.067) (500,0.058) (1000,0.052) (2000,0.049) (4000,0.050) (6000,0.051)};
        \addplot[draw=none, name path=cl, forget plot] coordinates {(100,0.045) (250,0.046) (500,0.042) (1000,0.040) (2000,0.041) (4000,0.041) (6000,0.041)};
        \addplot[green!25, fill opacity=0.5, forget plot] fill between[of=cu and cl];
        \addplot[color=green!50!black, mark=triangle*, mark size=1.2pt, line width=0.7pt]
        coordinates {(100,0.051) (250,0.052) (500,0.049) (1000,0.046) (2000,0.045) (4000,0.044) (6000,0.045)};
        \end{axis}
        \end{tikzpicture}
        \caption{Law School}
        \label{fig:scimp_lawschool}
    \end{subfigure}%
    \hfill%
    \begin{subfigure}[b]{0.19\linewidth}
        \centering
        \begin{tikzpicture}
        \begin{axis}[
             width=1.37\linewidth, height=4cm,
             xmin=-200, xmax=6500, ymin=0.05, ymax=0.5,
             xtick={0,3000,6000},
             ytick={0.1,0.2,0.3,0.4},
             scaled x ticks=false,
             scaled y ticks=false,
             xticklabel style={/pgf/number format/fixed, /pgf/number format/1000 sep={}, font=\tiny},
             yticklabel style={/pgf/number format/fixed, /pgf/number format/precision=2, font=\tiny, rotate=45},
             minor tick num=0, grid=major, grid style={dashed, gray!40},
        ]

        \addplot[draw=none, name path=au, forget plot] coordinates {(100,0.480) (250,0.452) (500,0.408) (1000,0.430) (2000,0.420) (4000,0.415) (6000,0.412)};
        \addplot[draw=none, name path=al, forget plot] coordinates {(100,0.349) (250,0.375) (500,0.346) (1000,0.353) (2000,0.367) (4000,0.393) (6000,0.400)};
        \addplot[blue!20, fill opacity=0.5, forget plot] fill between[of=au and al];
        \addplot[color=blue!70!black, mark=*, mark size=1pt, line width=0.7pt]
        coordinates {(100,0.406) (250,0.401) (500,0.386) (1000,0.401) (2000,0.403) (4000,0.404) (6000,0.407)};

        \addplot[draw=none, name path=bu, forget plot] coordinates {(100,0.412) (250,0.385) (500,0.288) (1000,0.180) (2000,0.224) (4000,0.236) (6000,0.186)};
        \addplot[draw=none, name path=bl, forget plot] coordinates {(100,0.073) (250,0.083) (500,0.117) (1000,0.075) (2000,0.076) (4000,0.081) (6000,0.084)};
        \addplot[orange!25, fill opacity=0.5, forget plot] fill between[of=bu and bl];
        \addplot[color=orange!90!black, mark=square*, mark size=0.8pt, line width=0.7pt]
        coordinates {(100,0.226) (250,0.206) (500,0.171) (1000,0.124) (2000,0.156) (4000,0.149) (6000,0.119)};

        \addplot[draw=none, name path=nnu, forget plot] coordinates {(100,0.412) (250,0.412) (500,0.396) (1000,0.359) (2000,0.290) (4000,0.284) (6000,0.306)};
        \addplot[draw=none, name path=nnl, forget plot] coordinates {(100,0.103) (250,0.123) (500,0.121) (1000,0.108) (2000,0.190) (4000,0.177) (6000,0.195)};
        \addplot[red!20, fill opacity=0.5, forget plot] fill between[of=nnu and nnl];
        \addplot[color=red!80!black, mark=diamond*, mark size=1.5pt, line width=0.7pt]
        coordinates {(100,0.282) (250,0.264) (500,0.239) (1000,0.257) (2000,0.238) (4000,0.225) (6000,0.252)};

        \addplot[draw=none, name path=cu, forget plot] coordinates {(100,0.200) (250,0.183) (500,0.145) (1000,0.116) (2000,0.105) (4000,0.091) (6000,0.086)};
        \addplot[draw=none, name path=cl, forget plot] coordinates {(100,0.125) (250,0.146) (500,0.130) (1000,0.106) (2000,0.085) (4000,0.075) (6000,0.070)};
        \addplot[green!25, fill opacity=0.5, forget plot] fill between[of=cu and cl];
        \addplot[color=green!50!black, mark=triangle*, mark size=1.2pt, line width=0.7pt]
        coordinates {(100,0.172) (250,0.162) (500,0.135) (1000,0.110) (2000,0.096) (4000,0.083) (6000,0.078)};
        \end{axis}
        \end{tikzpicture}
        \caption{ACSIncome}
        \label{fig:scimp_acsincome}
    \end{subfigure}%
    \hfill%
    \begin{subfigure}[b]{0.19\linewidth}
        \centering
        \begin{tikzpicture}
        \begin{axis}[
             width=1.37\linewidth, height=4cm,
             xmin=-200, xmax=6500, ymin=0.05, ymax=0.56,
             xtick={0,3000,6000},
             ytick={0.1,0.3,0.5},
             scaled x ticks=false,
             scaled y ticks=false,
             xticklabel style={/pgf/number format/fixed, /pgf/number format/1000 sep={}, font=\tiny},
             yticklabel style={/pgf/number format/fixed, /pgf/number format/precision=2, font=\tiny, rotate=45},
             minor tick num=0, grid=major, grid style={dashed, gray!40},
        ]

        \addplot[draw=none, name path=au, forget plot] coordinates {(100,0.545) (250,0.514) (500,0.430) (1000,0.424) (2000,0.423) (4000,0.389) (6000,0.391)};
        \addplot[draw=none, name path=al, forget plot] coordinates {(100,0.429) (250,0.418) (500,0.391) (1000,0.392) (2000,0.377) (4000,0.366) (6000,0.361)};
        \addplot[blue!20, fill opacity=0.5, forget plot] fill between[of=au and al];
        \addplot[color=blue!70!black, mark=*, mark size=1pt, line width=0.7pt]
        coordinates {(100,0.483) (250,0.450) (500,0.406) (1000,0.408) (2000,0.395) (4000,0.380) (6000,0.376)};

        \addplot[draw=none, name path=bu, forget plot] coordinates {(100,0.134) (250,0.184) (500,0.173) (1000,0.187) (2000,0.226) (4000,0.209) (6000,0.205)};
        \addplot[draw=none, name path=bl, forget plot] coordinates {(100,0.088) (250,0.099) (500,0.126) (1000,0.118) (2000,0.136) (4000,0.130) (6000,0.139)};
        \addplot[orange!25, fill opacity=0.5, forget plot] fill between[of=bu and bl];
        \addplot[color=orange!90!black, mark=square*, mark size=0.8pt, line width=0.7pt]
        coordinates {(100,0.114) (250,0.139) (500,0.154) (1000,0.163) (2000,0.166) (4000,0.174) (6000,0.179)};

        \addplot[draw=none, name path=nnu, forget plot] coordinates {(100,0.188) (250,0.204) (500,0.138) (1000,0.124) (2000,0.108) (4000,0.103) (6000,0.099)};
        \addplot[draw=none, name path=nnl, forget plot] coordinates {(100,0.087) (250,0.122) (500,0.099) (1000,0.080) (2000,0.070) (4000,0.075) (6000,0.074)};
        \addplot[red!20, fill opacity=0.5, forget plot] fill between[of=nnu and nnl];
        \addplot[color=red!80!black, mark=diamond*, mark size=1.5pt, line width=0.7pt]
        coordinates {(100,0.125) (250,0.144) (500,0.120) (1000,0.105) (2000,0.088) (4000,0.085) (6000,0.087)};

        \addplot[draw=none, name path=cu, forget plot] coordinates {(100,0.110) (250,0.102) (500,0.093) (1000,0.080) (2000,0.078) (4000,0.075) (6000,0.073)};
        \addplot[draw=none, name path=cl, forget plot] coordinates {(100,0.069) (250,0.069) (500,0.069) (1000,0.067) (2000,0.063) (4000,0.065) (6000,0.061)};
        \addplot[green!25, fill opacity=0.5, forget plot] fill between[of=cu and cl];
        \addplot[color=green!50!black, mark=triangle*, mark size=1.2pt, line width=0.7pt]
        coordinates {(100,0.085) (250,0.084) (500,0.078) (1000,0.075) (2000,0.070) (4000,0.069) (6000,0.068)};
        \end{axis}
        \end{tikzpicture}
        \caption{Synthetic}
        \label{fig:scimp_synthetic}
    \end{subfigure}

    \par\smallskip
    {\scriptsize Number of training samples}
    \caption{Strategic improvement error ($\widehat{\texttt{err}}_{\mathrm{Imp}})$ vs.\ number of training samples. Lines show the mean over 5 runs; shaded regions show the min--max range.}
    \label{fig:scimp_error_vs_n}
\end{figure*}

 \textbf{Experimental Setting.} We evaluate our algorithm on four datasets Adult \citep{Becker1996}, HELOC \citep{fico2018heloc}, Law School \citep{wightman1998lsac} and ACS Income \citep{ding2021retiring} and on Synthetic dataset. We partition each dataset in a $45/45/5/5$ ratio: the first $45\%$ is used to estimate $\eta$ and the rest as  training, validation and test sets. We perform two major experiments on these datasets. Additional details on  experimental setup  with ablation studies are given in the Appendix \ref{app:exp}. Dataset details are summarized in Table~\ref{tab:datasets} in Appendix~\ref{sec:compare_algos}. The class-probability function $\eta(x)$ is modeled using Logistic Regression, following the Assumption \ref{ass:singleIndex}. On all the datasets, we consider the features for which $\eta(x)$ is monotonically increasing, more details can be found in Appendix \ref{app:exp}. We use  a single-layer neural network with a $d$-dimensional input and a biased linear output and  optimize the model parameters using a \textsc{Strat-Imp-Loss} objective. The model is trained using stochastic gradient descent (SGD).

 We  implement two benchmark algorithms namely $\textsc{SERM}$ \citep{levanon2022generalized} and \citet{attias2025pac}. The classifier is trained using an objective that anticipates agent manipulation. During training, we adjust the loss by adding a strategic gain term, which represents the maximum score change within their budget $\beta$. For more details, refer Appendix~\ref{sec:compare_algos}.
 

 \textbf{Comparison with \citet{attias2025pac}}: The comparison is necessarily model-mismatched because the two frameworks model improvement differently. \citet{attias2025pac} use classifier-independent improvement sets, whereas our model induces improvement endogenously through strategic best responses. We therefore use their method as a representative improvement-aware baseline rather than as a strictly model-matched comparison.

 \textbf{Robustness to Training Set Size.} Figure \ref{fig:scimp_error_vs_n} evaluates the improvement-aware strategic error on test data across the datasets as a function of the number of training samples and shows that the improvement-aware linear classifier $f^{\star}_{\mathrm{Imp}}$ consistently achieves lowest improvement error  regardless of training set size. Furthermore, the strategic linear classifier $f^{\star}_s$ generally outperforms the standard optimal linear classifier $f^{\star}$, while achieving comparable performance on the Law School dataset. This clear hierarchy empirically demonstrates that $f^{\star}_s$ serves as a  better proxy for improvement-aware settings than the naive $f^{\star}$, even when training data is limited.


\textbf{Improvement Rates.}   
We evaluate the agent's label improvement across the four classifiers: \textsc{SVM}, {\sc SERM} \citep{levanon2022generalized}, \textsc{Strat-Imp-Aware} and \citet{attias2025pac}. For each deployed classifier, we count an agent as manipulating if its selected best response differs from its original feature. Among manipulating agents, we count an agent as improved if the sampled post-response label is positive after manipulation when the original label was negative. The improvement rate reported in Table \ref{tab:merged_cls}   is the fraction of negative labeled manipulating agents who improve. Total counts of manipulating and improved agents are quantified using the test data.


\section{ Outcome Law Uncertainty}
\label{sec:outcome_eta_uncertainty}

Access to $\eta$ determines the classifier for deployment under improvement aware setting.  For instance, we show in Section~\ref{sec:Algorithm} that  if the  learner has  {\em oracle access} to $\eta$ the optimal improvement aware classifier is learnable. However, oracle access may not be a practical assumption in many real world  settings, except  where multiple controlled trials are feasible, since post-deployment labels are function of the deployed classifier. Another interesting idea we explored in this paper, is to simulate an $\eta$ labels from the input data to evaluate the classifiers under consideration. Since the outcome law $\eta$ governing  pre- and post deployment label generation is fixed, we first learn $\widehat{\eta}$, an estimate of $\eta$ and use it to learn optimal improvement aware classifier with an additive error in $\eta$ estimate.  Learning  $ \eta$ comes with additional complication that it may not be always possible. We discuss this aspect more in Appendix~\ref{sec:oracle_free} and  give sufficient conditions under which we can learn $\eta$.   

Our next result (proof in Appendix~\ref{app:proxy})  provides a recommendation in cases where these conditions are difficult to verify or fail to hold.  In particular, we compare two classifiers namely, Bayes optimal classifier, which accounts for neither gaming nor improvement (naive viewpoint), and the optimal strategic classifier which incorporates only strategic gaming (pessimistic viewpoint). We show that when post-deployment labels generation process is difficult to learn it is always safe to choose the optimal strategic classifier.    

\begin{restatable}{theorem}{Proxy}\label{thm:proxy}
Under Assumptions~\ref{ass:singleIndex} and~\ref{assump: cost function}, the feasibility and tie-breaking conditions from Section~\ref{sec:structural}, the strategic-optimal classifier \(f_s^\star\) and the Bayes-optimal classifier $f^{\star}$  satisfy   \[\texttt{err}_{\mathrm{Imp}}(f_s^\star)\le \texttt{err}_{\mathrm{Imp}}(f^\star).\]
\end{restatable} 



\section{Conclusion}
When improvements are endogenous, strategic classification is no longer a purely defensive problem: agents may respond to the deployed classifier by  improving their qualifications, not only by gaming it. Our results show that this richer setting remains statistically tractable: under a single-index qualification model and linear-decomposable costs, best responses act as a closed-form shift of the decision boundary, and improvement-aware learning is PAC learnable under oracle access to $\eta$, with sample complexity linear in the dimension. Experiments on synthetic and semi-synthetic data support the plausibility of our assumptions and the effectiveness of $\textsc{Strat-Imp-Aware}$ when $\eta$ is estimated. 
Our guarantees rest on two  choices: a single-index structure, which lets $\widehat{\eta}$ extrapolate to best-response points beyond the support of the pre-deployment data, and independent agent responses under a fixed conditional outcome law. Future work  includes non-decomposable costs e.g.,  $\ell_2$ costs, heterogeneous agents, repeated interactions and causal models that separate genuine improvement from gaming.

\subsection*{AI use statement}

In this work, we used generative AI tools in three limited capacities: $(i)$ identifying grammatical and typographical errors in the text, $(ii)$ assisting with the derivation of proofs, and $(iii)$ debugging and writing code. All LLM-assisted proofs were verified line by line by the authors, and all LLM-generated code was reviewed and tested for correctness. Generative AI tools were not used to formulate research questions, hypotheses or claims, and the ideas and contributions presented here are our own. We conducted a manual literature survey to ensure that prior work is properly attributed. We take responsibility for the final content of this work, including all text, claims, proofs, and code, whether or not generative AI tools were used in producing them. 




\bibliographystyle{plainnat}
\bibliography{references}






\section*{Checklist}



\begin{enumerate}

  \item For all models and algorithms presented, check if you include:
  \begin{enumerate}
    \item A clear description of the mathematical setting, assumptions, algorithm, and/or model. [Yes]
    \item An analysis of the properties and complexity (time, space, sample size) of any algorithm. [Yes]
    \item (Optional) Anonymized source code, with specification of all dependencies, including external libraries. [Yes]
  \end{enumerate}

  \item For any theoretical claim, check if you include:
  \begin{enumerate}
    \item Statements of the full set of assumptions of all theoretical results. [Yes]
    \item Complete proofs of all theoretical results. [Yes]
    \item Clear explanations of any assumptions. [Yes]
  \end{enumerate}

  \item For all figures and tables that present empirical results, check if you include:
  \begin{enumerate}
    \item The code, data, and instructions needed to reproduce the main experimental results (either in the supplemental material or as a URL). [Yes]
    \item All the training details (e.g., data splits, hyperparameters, how they were chosen). [Yes]
    \item A clear definition of the specific measure or statistics and error bars (e.g., with respect to the random seed after running experiments multiple times). [Yes]
    \item A description of the computing infrastructure used. (e.g., type of GPUs, internal cluster, or cloud provider). [Yes]
  \end{enumerate}

  \item If you are using existing assets (e.g., code, data, models) or curating/releasing new assets, check if you include:
  \begin{enumerate}
    \item Citations of the creator If your work uses existing assets. [Yes]
    \item The license information of the assets, if applicable. [Yes]
    \item New assets either in the supplemental material or as a URL, if applicable. [Yes]
    \item Information about consent from data providers/curators. [Not Applicable]
    \item Discussion of sensible content if applicable, e.g., personally identifiable information or offensive content. [Not Applicable]
  \end{enumerate}

  \item If you used crowdsourcing or conducted research with human subjects, check if you include:
  \begin{enumerate}
    \item The full text of instructions given to participants and screenshots. [Not Applicable]
    \item Descriptions of potential participant risks, with links to Institutional Review Board (IRB) approvals if applicable. [Not Applicable]
    \item The estimated hourly wage paid to participants and the total amount spent on participant compensation. [Not Applicable]
  \end{enumerate}

\end{enumerate}

\clearpage
\appendix
\thispagestyle{empty}

\onecolumn
\aistatstitle{Linear Strategic Classification with Endogenous Improvements: \\
Supplementary Materials}

\section{Notation}
\allowdisplaybreaks 

\begin{table}[ht]
\centering
\label{tab:notation}
\small
\begin{tabular}{p{0.39\linewidth}p{0.6\linewidth}}
\toprule
\textbf{Notation} & \textbf{Description} \\
\midrule
\(\mathcal X\subseteq\mathbb R^d\) & Valid feature space; both original features and reported/post-response features lie in \(\mathcal X\). \\

\(\mathcal X_0=\operatorname{supp}(D)\subseteq\mathcal X\) & Support of the original/pre-deployment feature distribution. \\

\(\mathcal A\subseteq\mathcal X\) & Feasible report/action space, with \(\mathcal X_0\subseteq\mathcal A\). \\

\(D\) & Distribution over original features, supported on \(\mathcal X_0\). \\

\(x\in\mathcal X_0\) & Original feature vector sampled from \(D\). \\

\(z\in\mathcal A\) & Reported/post-response feature vector. \\

\(Y=\{0,1\}\) & Binary label space. \\

\(\eta:\mathcal X\to[0,1]\) & Class-probability function, \(\eta(x)= \mathbb{P}(Y=1\mid X=x)\). \\

\(y_x\sim\operatorname{Bernoulli}(\eta(x))\) & Conditional label generated at feature vector \(x\in\mathcal X\). \\

\(\mathcal F\) & Hypothesis class of linear classifiers \(f_{w,b}:\mathcal X\to\{0,1\}\). \\

\(f_{w,b}(z)=\mathds 1\{w^\top z\ge b\}\) & Linear classifier with weight vector \(w\in\mathbb R^d_{\ge0}\) and threshold \(b\in\mathbb R\). \\

\(c:\mathcal X_0\times\mathcal A\to\mathbb R_+\) & Cost of changing original feature \(x\in\mathcal X_0\) to report \(z\in\mathcal A\). \\

\(c(x,z)=\sum_{i=1}^d\alpha_i(z_i-x_i)_+\) & Linear-decomposable one-sided cost, for some $i \in [d]$, \(\alpha_i>0\). \\

\((a)_+=\max\{a,0\}\) & Positive part of \(a\in\mathbb R\). \\

\(\beta>0\) & Utility gain from receiving positive classification. \\

\(u_f(x,z)=\beta f(z)-c(x,z)\) & Utility of reporting \(z\in\mathcal A\) for an agent with original feature \(x\in\mathcal X_0\). \\

\(x^f\in\arg\max_{z\in\mathcal A}u_f(x,z)\) & Best response/report under classifier \(f_{w,b}\). \\

$\texttt{err}(f) = \mathbb E_{x\sim D}[\mathds 1\{f(x)\neq y_x\}]$ & Standard non-strategic classification error. \\

$\texttt{err}_s(f) = \mathbb E_{x\sim D}[\mathds 1\{f(x^f)\neq y_x\}]$ & Strategic error with immutable labels. \\

\(\texttt{err}_{\mathrm{Imp}}(f)=\mathbb E_{x\sim D}[\mathds 1\{f(x^f)\neq y_{x^f}\}]\) & Strategic improvement error with post-response labels. \\

\(\widehat{\texttt{err}}_{\mathrm{Imp}}(f)=\frac1n\sum_{i=1}^n\mathds 1\{f(x_i^f)\neq y_{x_i^f}\}\) & Empirical strategic improvement error. \\

\(g:\mathbb R\to[0,1]\) & Monotone link function in the single-index model. \\

\(w^\star\in\mathbb R^d_{\ge0}\) & True single-index direction, satisfying \(\eta(x)=g((w^\star)^\top x)\). \\

\(b^\star=\inf\{t \in \mathbb{R} :g(t)\ge1/2\}\) & Bayes threshold under the single-index model. \\

\(f^\star\) & Non-strategic Bayes optimal classifier. \\

\(f_s^\star\) & Strategic-optimal shifted classifier. \\

\(f_{\mathrm{Imp}}^\star\) & Improvement-aware optimal classifier. \\

\(S(w)=\beta\max_{i\in[d]}w_i/\alpha_i\) & Strategic shift induced by classifier \(f_{w,b}\). \\

\(b_s=b+S(w)\) & Shifted threshold corresponding to \(f_{w,b}\). \\

\(M_{w,b}(x)=w^\top x-b+S(w)\) & Strategic shifted margin used in the hinge surrogate. \\

\(\rho_{\mathrm{Imp}}=r+\beta/\min_{j\in[d]}\alpha_j\) & Radius parameter used in the generalization bound, where \(r=\sup_{z\in\mathcal A}\|z\|_2\). \\

\(\widehat\eta\) & Plug-in estimator of \(\eta\). \\

\(\varepsilon_\eta=\sup_{z\in\mathcal A}|\widehat\eta(z)-\eta(z)|\) & Uniform plug-in estimation error over feasible reports. \\

\(\widehat L_{\mathrm{hinge}}^{\widehat q}(w,b)\) & Empirical plug-in strategic improvement loss minimized by \textsc{Strat-Imp-Aware}. \\

\(\mathcal W_{k,B}\) & Norm- and bias-bounded parameter class \(\{(w,b):w\in\mathbb R_{\ge0}^d,\|w\|_2\le k, |b|\le B\}\). \\

\(C_{k,B}=k\rho_{\mathrm{Imp}}+B\) & Uniform margin bound used in Theorem~\ref{thm:alg_guarantee}. \\
\bottomrule
\end{tabular}
\caption{Notation used in the paper.}
\end{table}

\section{Proofs from Section \ref{sec:setting}}
\label{app:setting}

\ObsOne*
\begin{proof}
By Assumption~\ref{ass:singleIndex}, 
\[
\eta(x)=g(w^{\star\top} x).
\]
Since $g$ is nondecreasing, the set
\[
\{x:\eta(x)\ge 1/2\}
=
\{x:g(w^{\star\top} x)\ge 1/2\}
\]
is a threshold set of the form $\{x:w^{\star\top} x\ge b^\star\}$, where
\[
b^\star := \inf\{t\in\mathbb R: g(t)\ge 1/2\}.
\]
Hence the classifier $f^\star(x)=\mathbf 1\{w^{\star\top} x\ge b^\star\}$ is Bayes optimal.
\end{proof}

\section{Proofs from Section \ref{sec:structural}}
\label{app:structural}

In this section, we present proofs from Section~\ref{sec:structural}. 

\begin{figure}[ht!] 
    \centering
    \includegraphics[width=0.7\columnwidth]{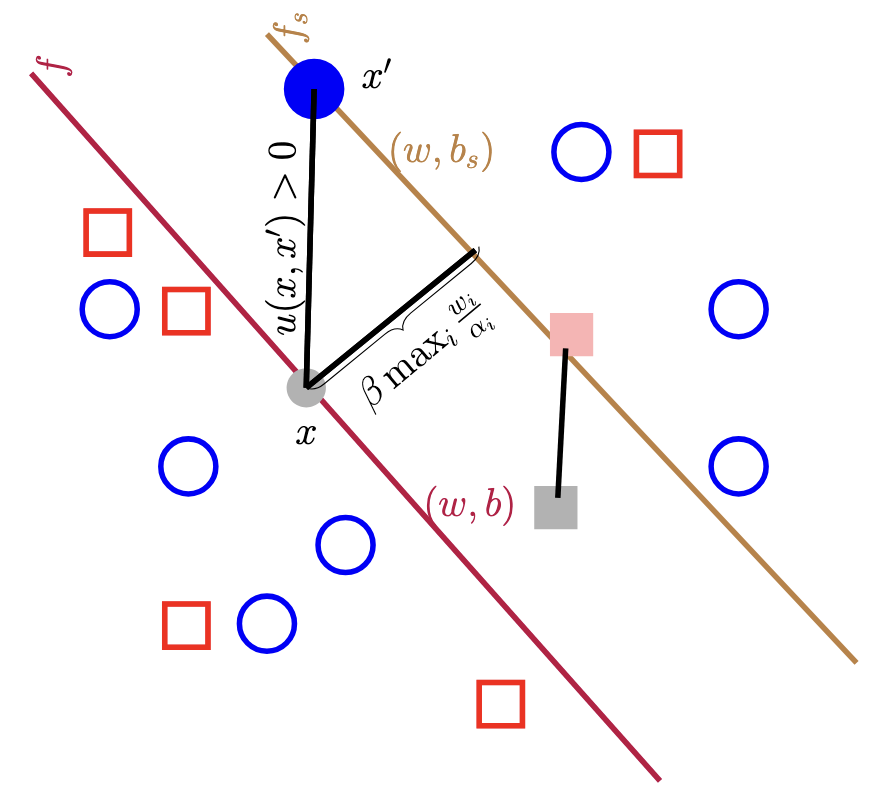}
    \caption{The strategic classifier $f_s$ (brown) is obtained by shifting the linear classifier $f$ (red) by a margin of $\max_{i \in [d]} \left( \frac{w_i}{\alpha_i} \right) \beta$. This translation anticipates utility-maximizing feature manipulation from $x$ to $x'$ by the agents.}
    \label{fig: f_s relation with f*}
\end{figure}

\bestresponse*
\begin{proof}
    Let $\Delta \coloneqq b - w^{\top}x$. Since $w \neq 0$ and $\alpha_j > 0$, we have $S(w) > 0$ and $w_{j^{\star}}>0$. For any $z \in \mathcal{A}$ with $w^{\top} z  \geq b$, using $\alpha_j \geq \beta w_j / S(w)$, $(t)_+ \geq t$ and $w \geq 0$, 
    \[
        c(x,z) \geq \frac{\beta}{S(w)} \sum_{j=1}^d w_{j} (z_j - x_j)_+ \geq \frac{\beta}{S(w)} w^{\top}(z-x) \geq \frac{\beta \Delta}{S(w)}.
    \]
    Hence every positively classified $z$ has $u_f(x,z) \leq \beta (1 - \Delta/S(w))$, where $u_f(x,z) \coloneqq \beta \mathds{1}(w^{\top}z \geq b) - c(x,z)$, while every negatively classified $z$ has $u_{f}(x,z) \leq 0$.

    \medskip
    \noindent 
    \emph{Case $\Delta \leq 0$.} Staying at $x$ is classified positively at zero cost, so $u_f(x,z) = \beta$ is maximal and $x^f = x$. 

    \medskip 
    \noindent 
    \emph{Case $0 < \Delta \leq S(w)$.} Let $z^{\star} \coloneqq x + (\Delta / w_{j^{\star}})e_{j^{\star}}$. Since $\Delta/w_{j^{\star}} \leq S(w)/w_{j^{\star}} = \beta / \alpha_{j^{\star}}$, $z^{\star} \in \mathcal{A}$ by feasibility. Moreover $w^{\top}z^{\star} = b$ and $c(x,z) = \alpha_{j^{\star}} \Delta/ w_{j^{\star}} = \beta \Delta/ S(w)$, so $z^{\star}$ attains the bound above with $u_f(x,z^{\star}) \ge 0$. Thus, $z^{\star}$ is a positive best response; when $\Delta = S(w)$ it ties with staying, and tie-breaking selects it. Among point $x + t e_{j^{\star}}$, only $t = \Delta / w_{j^{\star}}$ is both positive and cost-minimal, so $x^f = z^{\star}$.

    \medskip
    \noindent 
    \emph{Case $\Delta > S(w)$.} Every positively classified $z$ has $u_f(x,z) < 0 = u_f(x,x)$, so no positive best response exists and $x^f = x$.  
\end{proof}

\strategiccloseform*
\begin{proof} 
Let $r(w) = \max_{j \in [d]} w_j / \alpha_j$, fix $f_{w,b}\in\mathcal F$. Let $x^{f_s}$ denote the selected tie-broken best response to the shifted classifier $f_s$. For any selected report $z$ in the range of the best-response map, define $\mathcal Z_z:=\{x\in\mathcal X_0: z \text{ is the selected best response to } f_s \text{ for agent } x\}$. These sets partition $\mathcal X_0$ according to the selected best-response rule. It is therefore enough to show that, for every selected report $z$ and every $x\in\mathcal Z_z$, $f_s(z)=f_{w,b}(x)$.

We first record a useful bound. For any $x\in\mathcal X_0, y \in \mathcal A$, 
\begin{align*}
    w^\top(y-x) \le \sum_{i=1}^d w_i(y_i-x_i)_+ &\le r(w)\sum_{i=1}^d\alpha_i(y_i-x_i)_+ \\
    &= r(w)c(x,y). 
\end{align*}
We now consider three cases.

\noindent\textbf{Case 1: $w^\top z>b_s$.}
Then $f_s(z)=1$. Suppose, for contradiction, that $f_{w,b}(x)=0$ for some $x\in\mathcal Z_z$. Then $w^\top x<b$, and hence $f_s(x)=0$. Since staying at $x$ gives utility $0$, while reporting $z$ gives utility $\beta-c(x,z)$, optimality of $z$ implies $c(x,z)\le\beta$. Using the bound above, $w^\top z\le w^\top x+r(w)c(x,z)\le w^\top x+\beta r(w)<b+\beta r(w)=b_s$, contradicting $w^\top z>b_s$. Hence $f_{w,b}(x)=1=f_s(z)$.

\noindent\textbf{Case 2: $w^\top z=b_s$.}
Again $f_s(z)=1$. Suppose, for contradiction, that $f_{w,b}(x)=0$ for some $x\in\mathcal Z_z$. Then $w^\top x<b$ and $f_s(x)=0$. Since $z$ is a selected best response, we must have $c(x,z)\le\beta$. Therefore, $b_s=w^\top z\le w^\top x+r(w)c(x,z)\le w^\top x+\beta r(w)<b+\beta r(w)=b_s$, a contradiction. Thus $f_{w,b}(x)=1=f_s(z)$.

\noindent\textbf{Case 3: $w^\top z<b_s$.}
Then $f_s(z)=0$. Suppose, for contradiction, that $f_{w,b}(x)=1$ for some $x\in\mathcal Z_z$, i.e., $w^\top x\ge b$. If $w^\top x\ge b_s$, then staying at $x$ gives positive classification with zero cost and utility $\beta$, whereas reporting $z$ gives utility $-c(x,z)\le0$, contradicting optimality of $z$. Hence we may assume $b\le w^\top x<b_s$. Let $\delta:=b_s-w^\top x\in(0,\beta r(w)]$. Choose $i^\star\in\arg\max_i w_i/\alpha_i$. Since $w\neq0$ and $\alpha_i>0$, we have $w_{i^\star}>0$. By the feasibility assumption, the point $y:=x+\frac{\delta}{w_{i^\star}}e_{i^\star}$ belongs to $\mathcal \mathcal A$. Moreover, $w^\top y=b_s$, so $f_s(y)=1$, and $c(x,y)=\alpha_{i^\star}\frac{\delta}{w_{i^\star}}=\frac{\delta}{r(w)}\le\beta$. Thus reporting $y$ gives utility $\beta-c(x,y)\ge0$. On the other hand, since $f_s(z)=0$, reporting $z$ gives utility $-c(x,z)\le0$. If $\beta-c(x,y)>0$, this contradicts optimality of $z$. If $\beta-c(x,y)=0$, then $y$ is a positive best response with utility equal to that of any zero-utility negative response, so the tie-breaking rule selects a positive response rather than $z$. This again contradicts $z$ being the selected report. Therefore $f_{w,b}(x)=0=f_s(z)$.

Combining the three cases, for every selected report $z$ and every $x\in\mathcal Z_z$, $f_s(z)=f_{w,b}(x)$. Equivalently, for every $x\in\mathcal X_0$, $f_s(x^{f_s})=f_{w,b}(x)$. Taking expectation over the data distribution yields $\texttt{err}_s(f_s)=\texttt{err}(f_{w,b})$.
\end{proof}

\RelationFAndFStar* 
\begin{proof} Let $f^{\star}$ be the Non-strategic Bayes optimal linear  classifier and $f_s^{\star}$ be the corresponding strategic classifier satisfying $\texttt{err}_{s}(f_s^{\star}) = \texttt{err}(f^{\star})$. For contradiction, assume that there is a classifier $f_s^{'}$ with parameters $(w', b')$ satisfying 
\begin{equation} 
\label{eq: contradiction eq}
\texttt{err}_{s}(f_s^{'}) < \texttt{err}_s(f_{s}^{\star}) . \end{equation}  and consider    classifier $f^{'}$ with parameters $w', b^{'} - \max_{i \in [d]}\left(\frac{w_i^{'}}{\alpha_i}\right) \beta$. The following set of inequalities result in a contradiction.   \begin{align*}  
\texttt{err}_s(f_s^{\star})  = \texttt{err}(f^{\star})  \leq \texttt{err}(f^{'}) = \texttt{err}_s(f_s^{'}) < \texttt{err}_s(f_s^{\star}).   \end{align*} 
The first inequality above follows from Bayes optimality of $f^\star$ and the second inequality follows from Equation \ref{eq: contradiction eq}. 
\end{proof}

\begin{restatable}{lemma}{CharacterizationInOneD}[Improvement threshold lies between Non-strategic Bayes and Strategic thresholds]
\label{thm:imp in r1}
Let $x \sim \mathcal{D}$ with $x \in \mathcal{X} \subseteq \mathbb{R}$, and the label associated with input $x$ be $y_{x} \sim \operatorname{Bernoulli}(\eta(x))$ where $\eta: \mathcal{X} \rightarrow [0,1]$ is a strictly increasing function. Consider threshold classifiers of the form $f_\theta(x) = \mathds{1}(x \geq \theta)$. 
Denote by $\theta^\star$ the Non-strategic Bayes optimal threshold. Let $\theta_s$ be the optimal threshold under strategic manipulation with  $\exists \alpha \in \mathbb{R}^{+}$ then suppose cost function $c(x,x') = max(0, \alpha(x'-x))$ and let $\theta_{\mathrm{Imp}}$ be the optimal threshold under our improvement-aware model. Then the thresholds satisfy: 
\[
\theta^\star \;\;\le\;\; \theta_{\mathrm{Imp}} \;\;\le\;\; \theta_s.
\]
\end{restatable} 

\begin{proof}
We establish the inequality chain $\theta^\star \leq \theta_{\mathrm{Imp}} \leq \theta_s$ via proof by contradiction. We rely on the monotonicity of the class posterior $\eta(x)$ under the generalized linear model assumption. We proceed in two parts, First, suppose that $\theta_{\mathrm{Imp}} > \theta_s$. We will show that this violates the optimality condition of the improvement-aware classifier. Next, suppose that $\theta_{\mathrm{Imp}} < \theta^\star$. We will again reach a contradiction by showing that the improvement risk can be strictly reduced by raising the threshold. For simplicity in notation, let $x' = x^{f_{\theta_{\mathrm{Imp}}}}$ be strategic manipulated point under $f_{\theta_{\mathrm{Imp}}}$ threshold classifier.\\
We know that,
\begin{align*}
    \theta_{\mathrm{Imp}} &= \argmin_{\theta \in \mathbb{R}} \texttt{err}_{\mathrm{Imp}}(f_{\theta}) \\
    &= \argmin_{\theta \in \mathbb{R}} \mathbb{E}[\mathds{1}(f_{\theta}(x^{f_{\theta}}) \neq y_{x^{f_{\theta}}})]
\end{align*} 

We begin with a few supporting technical results. 

\begin{claim}\label{lemma: lower bound for improvement threshold}
Let us assume $\theta^{\star} > \theta_{\mathrm{Imp}}$. Then,
\begin{align*}
\left\{
\begin{aligned}
    &\mathbb{P}(y_{\theta_{\mathrm{Imp}}} = 0) (F(\theta_{\mathrm{Imp}}) - F(\theta_{\mathrm{Imp}} -\frac{\beta}{\alpha})) - \mathbb{P}(y_{\theta^\star} = 0)(F(\theta^\star) - F(\theta^\star -\frac{\beta}{\alpha}))\\
    &\quad + \int_{\theta_{\mathrm{Imp}}}^{\theta^\star}\,\mathbb{P}(y_x = 0)\, p_X(x) dx \, - \int_{\theta_{\mathrm{Imp}} -\frac{\beta}{\alpha}}^{\theta^\star -\frac{\beta}{\alpha}}\, \mathbb{P}(y_x = 1)\, p_X(x)\,dx 
\end{aligned}
\right\}    
    > 0 
\end{align*}
\end{claim}

\begin{proof}
\begin{align*}
    &\mathbb{P}(y_{\theta_{\mathrm{Imp}}} = 0) (F(\theta_{\mathrm{Imp}}) - F(\theta_{\mathrm{Imp}} -\frac{\beta}{\alpha})) - \mathbb{P}(y_{\theta^\star} = 0)(F(\theta^\star) - F(\theta^\star -\frac{\beta}{\alpha}))\\
    &\quad + \int_{\theta_{\mathrm{Imp}}}^{\theta^\star}\,\mathbb{P}(y_x = 0)\, p_X(x) dx \, - \int_{\theta_{\mathrm{Imp}} -\frac{\beta}{\alpha}}^{\theta^\star -\frac{\beta}{\alpha}}\, \mathbb{P}(y_x = 1)\, p_X(x)\,dx  \\[1em]
    &\geq \,\mathbb{P}(y_{\theta_{\mathrm{Imp}}} = 0) (F(\theta_{\mathrm{Imp}}) - F(\theta_{\mathrm{Imp}} -\frac{\beta}{\alpha})) + \mathbb{P}(y_{\theta^\star} = 0)(F(\theta^\star -\frac{\beta}{\alpha}) - F(\theta_{\mathrm{Imp}}))\\
    &\quad - \int_{\theta_{\mathrm{Imp}} -\frac{\beta}{\alpha}}^{\theta^\star -\frac{\beta}{\alpha}}\, \mathbb{P}(y_x = 1)\, p_X(x)\,dx \\
    \tag{$\because \forall x \in [\theta_{\mathrm{Imp}}, \theta^\star]$, $\quad \mathbb{P}(y_{x} = 0) \leq \mathbb{P}(y_{\theta_{\mathrm{Imp}}} = 0)$} \\[1em] 
    &> \mathbb{P}(y_{\theta^\star} = 0)\,(F(\theta^\star -\frac{\beta}{\alpha}) - F(\theta_{\mathrm{Imp}} -\frac{\beta}{\alpha})) - \int_{\theta_{\mathrm{Imp}} -\frac{\beta}{\alpha}}^{\theta^\star -\frac{\beta}{\alpha}}\, \mathbb{P}(y_x = 1)\, p_X(x)\,dx  \\
    \tag{$\because \theta_{\mathrm{Imp}} < \theta^\star \implies \mathbb{P}(y_{\theta_{\mathrm{Imp}}} = 0) > \mathbb{P}(y_{\theta^\star} = 0)$} \\[1em]
    &=\,(F(\theta^\star -\frac{\beta}{\alpha}) - F(\theta_{\mathrm{Imp}} -\frac{\beta}{\alpha}))(\mathbb{P}(y_{\theta^\star} = 0) - 1) + \int_{\theta_{\mathrm{Imp}} -\frac{\beta}{\alpha}}^{\theta^\star -\frac{\beta}{\alpha}} \mathbb{P}(y_x = 0)\, p_X(x) dx \\[1em]
    &\geq (F(\theta^\star -\frac{\beta}{\alpha}) - F(\theta_{\mathrm{Imp}} -\frac{\beta}{\alpha}))(\mathbb{P}(y_{\theta^\star} = 0) + \mathbb{P}(y_{\theta^{\star} -\frac{\beta}{\alpha}} = 0) - 1)\\
    \tag{$\because \forall x \in [\theta_{\mathrm{Imp}} -\frac{\beta}{\alpha}, \theta^{\star} -\frac{\beta}{\alpha}]$, $\quad \mathbb{P}(y_{x} = 0) \geq \mathbb{P}(y_{\theta^{\star} -\frac{\beta}{\alpha}} = 0)$} \\[1em]
    &> (F(\theta^{\star} -\frac{\beta}{\alpha}) - F(\theta_{\mathrm{Imp}} -\frac{\beta}{\alpha}))(2\,\mathbb{P}(y_{\theta^{\star}} = 0) - 1)\\[1em]
    &= 0 \tag{ $\because \mathbb{P}(y_{\theta^{\star}} = 0) = \frac{1}{2}$}
\end{align*}
Hence proved. 
\end{proof}

\begin{claim}\label{lemma: upper bound on improvement threshold}
Let us assume $\theta_{\mathrm{Imp}} > \theta_s$. Then,  
\begin{align*}
\left\{
\begin{aligned}
    &\mathbb{P}(y_{\theta_{\mathrm{Imp}}} = 0)(F(\theta_{\mathrm{Imp}}) - F(\theta_{\mathrm{Imp}} -\frac{\beta}{\alpha})) - \mathbb{P}(y_{\theta_{s}} = 0)(F(\theta_{s}) - F(\theta_{s} -\frac{\beta}{\alpha}))\\
    &\quad - \int_{\theta_{s}}^{\theta_{\mathrm{Imp}}}\,\mathbb{P}(y_x = 0)\, p_X(x) dx \, + \int_{\theta_{s} -\frac{\beta}{\alpha}}^{\theta_{\mathrm{Imp}} -\frac{\beta}{\alpha}}\, \mathbb{P}(y_x = 1)\, p_X(x)\,dx 
\end{aligned}
\right\}    
    > 0 
\end{align*}
\end{claim}

\begin{proof}
\begin{align*}
    &\mathbb{P}(y_{\theta_{\mathrm{Imp}}} = 0)(F(\theta_{\mathrm{Imp}}) - F(\theta_{\mathrm{Imp}} -\frac{\beta}{\alpha})) - \mathbb{P}(y_{\theta_{s}} = 0)(F(\theta_{s}) - F(\theta_{s} -\frac{\beta}{\alpha}))\\
    &\quad - \int_{\theta_{s}}^{\theta_{\mathrm{Imp}}}\,\mathbb{P}(y_x = 0)\, p_X(x) dx \, + \int_{\theta_{s} -\frac{\beta}{\alpha}}^{\theta_{\mathrm{Imp}} -\frac{\beta}{\alpha}}\, \mathbb{P}(y_x = 1)\, p_X(x)\,dx  \\[1em]
    &\geq \mathbb{P}(y_{\theta_{\mathrm{Imp}}} = 0)(F(\theta_{\mathrm{Imp}}) - F(\theta_{\mathrm{Imp}} -\frac{\beta}{\alpha}))\,+ \, \mathbb{P}(y_{\theta_s} = 0)(F(\theta_s -\frac{\beta}{\alpha}) - F(\theta_{\mathrm{Imp}})) \\
    &\quad + \int_{\theta_{s} -\frac{\beta}{\alpha}}^{\theta_{\mathrm{Imp}} -\frac{\beta}{\alpha}}\, \mathbb{P}(y_x = 1)\, p_X(x)\,dx \\
    \tag{$\because \forall x \in [\theta_{s}, \theta_{\mathrm{Imp}}]$, $\quad - \mathbb{P}(y_{x} = 0) \geq - \mathbb{P}(y_{\theta_{s}} = 0)$} \\[1em] 
    &> \mathbb{P}(y_{\theta_{\mathrm{Imp}}}= 0)(F(\theta_{s} -\frac{\beta}{\alpha}) - F(\theta_{\mathrm{Imp}} -\frac{\beta}{\alpha})) + \int_{\theta_{s} -\frac{\beta}{\alpha}}^{\theta_{\mathrm{Imp}} -\frac{\beta}{\alpha}}\, \mathbb{P}(y_x = 1)\, p_X(x)\,dx \\
    \tag{$\because \theta_{s} < \theta_{\mathrm{Imp}} \implies \mathbb{P}(y_{\theta_{s}} = 0) > \mathbb{P}(y_{\theta_{\mathrm{Imp}}} = 0)$}\\[1em] 
    &= (F(\theta_{\mathrm{Imp}} -\frac{\beta}{\alpha}) - F(\theta_{s} -\frac{\beta}{\alpha}))(1 - \mathbb{P}(y_{\theta_{\mathrm{Imp}}} = 0)) - \int_{\theta_s -\frac{\beta}{\alpha}}^{\theta_{\mathrm{Imp}} -\frac{\beta}{\alpha}}\mathbb{P}(y_x = 0)\,p_X(x)dx \\[1em]
    &\geq (F(\theta_{\mathrm{Imp}} -\frac{\beta}{\alpha}) - F(\theta_{s} -\frac{\beta}{\alpha}))(1 - \mathbb{P}(y_{\theta_{\mathrm{Imp}}} = 0) - \mathbb{P}(y_{\theta_s -\frac{\beta}{\alpha}} = 0))\\
    \tag{$\because \forall x \in [\theta_{s} -\frac{\beta}{\alpha}, \theta_{\mathrm{Imp}} -\frac{\beta}{\alpha}]$, $\quad - \mathbb{P}(y_{x} = 0) \geq - \mathbb{P}(y_{\theta_{s} -\frac{\beta}{\alpha}} = 0)$}\\[1em]
    &> (F(\theta_{\mathrm{Imp}} -\frac{\beta}{\alpha}) - F(\theta_{s} -\frac{\beta}{\alpha}))(1 - 2 \mathbb{P}(y_{\theta^{\star}} = 0)) 
    \tag{$\because - \mathbb{P}(y_{\theta_{\mathrm{Imp}}}=0) > - \mathbb{P}(y_{\theta_{s} -\frac{\beta}{\alpha}}=0) = -\mathbb{P}(y_{\theta^{\star}}=0)$}\\[1em]
    &= 0 \tag{ $\because \mathbb{P}(y_{\theta^{\star}} = 0) = \frac{1}{2}$}
\end{align*}
\end{proof}

\paragraph{Case 1: Lower bound ($\theta_{\mathrm{Imp}} \geq \theta^\star$)} 
Now, let us assume that ${\theta_{\mathrm{Imp}} < \theta^{\star}}$. 
\begin{align*}
    \texttt{err}_{\mathrm{Imp}}(f_{\theta_{\mathrm{Imp}}}) &= \mathbb{E}[\mathds{1}(f_{\theta_{\mathrm{Imp}}}(x') \neq y_{x'})]  \\
    &= \mathbb{E}[\mathds{1}(f_{\theta_{\mathrm{Imp}}}(x') \neq y_{x'}, x \in R_{\theta_{\mathrm{Imp}}})] + \mathbb{E}[\mathds{1}(f_{\theta_{\mathrm{Imp}}}(x') \neq y_{x'}, x \notin R_{\theta_{\mathrm{Imp}}})] \\ 
    &= \mathbb{P}({f_{\theta_{\mathrm{Imp}}}(x') \neq y_{x'}, x \in R_{\theta_{\mathrm{Imp}}}}) + \mathbb{P}({f_{\theta_{\mathrm{Imp}}}(x') \neq y_{x'}, x \notin R_{\theta_{\mathrm{Imp}}}}) \\ 
    \intertext{$R_{\theta_{\mathrm{Imp}}}:= [\theta_{\mathrm{Imp}} -\frac{\beta}{\alpha}, \theta_{\mathrm{Imp}}]$ is the region such that agents in this region modify their features to get positive classification. } 
    &= \mathbb{P}({y_{\theta_{\mathrm{Imp}}} = 0 , x \in R_{\theta_{\mathrm{Imp}}}}) + \mathbb{P}({f_{\theta_{\mathrm{Imp}}}(x') \neq y_{x'}, x \notin R_{\theta_{\mathrm{Imp}}}}) \\ 
    &= \mathbb{P}(y_{\theta_{\mathrm{Imp}}} = 0,\, x \in R_{\theta_{\mathrm{Imp}}}\big)  
   + \mathbb{P}(y_x = 0,\, x > \theta_{\mathrm{Imp}}) 
   + \mathbb{P}(y_x = 1,\, x < \theta_{\mathrm{Imp}} -\frac{\beta}{\alpha}) \\[1em]
    &= \int_{\theta_{\mathrm{Imp}}-\frac{\beta}{\alpha}}^{\theta_{\mathrm{Imp}}} 
          \mathbb{P}(y_{\theta_{\mathrm{Imp}}}=0)\, p_X(x)\, dx 
       + \int_{\theta_{\mathrm{Imp}}}^{\infty} 
          \mathbb{P}(y_x=0)\, p_X(x)\, dx \\ 
    &\quad + \int_{-\infty}^{\theta_{\mathrm{Imp}}-\frac{\beta}{\alpha}} 
          \mathbb{P}(y_x=1)\, p_X(x)\, dx \\[1em]
    &= \texttt{err}_{\mathrm{Imp}}(f_{\theta^{\star}}) + \mathbb{P}(y_{\theta_{\mathrm{Imp}}} = 0)(F(\theta_{\mathrm{Imp}}) - F(\theta_{\mathrm{Imp}} -\frac{\beta}{\alpha})) \\
    &\quad - \mathbb{P}(y_{\theta^{\star}} = 0)(F(\theta^{\star}) - F(\theta^{\star} -\frac{\beta}{\alpha})) + \int_{\theta_{\mathrm{Imp}}}^{\theta^{\star}} \mathbb{P}(y_x = 0)\, p_X(x)\, dx\\
    &\quad - \int_{\theta_{\mathrm{Imp}} -\frac{\beta}{\alpha}}^{\theta^{\star} -\frac{\beta}{\alpha}}\mathbb{P}(y_x = 1)\, p_X(x), dx \\[1em]
    &> \texttt{err}_{\mathrm{Imp}}(f_{\theta^{\star}}) \tag{Follows from the \textit{Lemma \ref{lemma: lower bound for improvement threshold}} }
\end{align*}
which contradicts improvement optimality condition of classifier, our assumption is false. Thus $\theta_{\mathrm{Imp}} \geq \theta^{\star}$ \\[1em] 
\paragraph{Case 2: Upper bound ($\theta_{\mathrm{Imp}} \leq \theta_s$)} Now, let us assume that ${\theta_{\mathrm{Imp}} > \theta_s}$.  
\begin{align*}
    \texttt{err}_{\mathrm{Imp}}(f_{\theta_{\mathrm{Imp}}}) &= \mathbb{E}[\mathds{1}(f_{\theta_{\mathrm{Imp}}}(x') \neq y_{x'})]  \\
    &= \mathbb{P}(y_{\theta_{\mathrm{Imp}}} = 0,\, x \in R_{\theta_{\mathrm{Imp}}}\big)  
   + \mathbb{P}(y_x = 0,\, x > \theta_{\mathrm{Imp}}) 
   + \mathbb{P}(y_x = 1,\, x < \theta_{\mathrm{Imp}} -\frac{\beta}{\alpha}) \\[1em]
    &= \int_{\theta_{\mathrm{Imp}}-\frac{\beta}{\alpha}}^{\theta_{\mathrm{Imp}}} 
          \mathbb{P}(y_{\theta_{\mathrm{Imp}}}=0)\, p_X(x)\, dx 
       + \int_{\theta_{\mathrm{Imp}}}^{\infty} 
          \mathbb{P}(y_x=0)\, p_X(x)\, dx \\ 
    &\quad + \int_{-\infty}^{\theta_{\mathrm{Imp}}-\frac{\beta}{\alpha}} 
          \mathbb{P}(y_x=1)\, p_X(x)\, dx \\[1em]
    &= \texttt{err}_{\mathrm{Imp}}(f_{\theta_s}) + \mathbb{P}(y_{\theta_{\mathrm{Imp}}} = 0)(F(\theta_{\mathrm{Imp}}) - F(\theta_{\mathrm{Imp}} -\frac{\beta}{\alpha})) \\
    &\quad - \mathbb{P}(y_{\theta_{s}} = 0)(F(\theta_{s}) - F(\theta_{s} -\frac{\beta}{\alpha})) - \int_{\theta_s}^{\theta_{\mathrm{Imp}}} \mathbb{P}(y_x = 0)\, p_X(x)\, dx\\
    &\quad - \int_{\theta_s -\frac{\beta}{\alpha}}^{\theta_{\mathrm{Imp}} -\frac{\beta}{\alpha}}\mathbb{P}(y_x = 1)\, p_X(x) dx \\[1em]
    &> \texttt{err}_{\mathrm{Imp}}(f_{\theta_s}) \tag{Follows from the \textit{Lemma} \ref{lemma: upper bound on improvement threshold} }
\end{align*}
which contradicts the improvement optimality condition, our assumption is false. Thus $\theta_{\mathrm{Imp}} \leq \theta_s$.
\end{proof}

\section{Generalization Bounds with Oracle Access}
\label{app:PAC}

\begin{lemma}[Uniform deviation bound for a VC class]
\label{lem:vc-uniform-deviation}
Let \(\mathcal{H}\) be a class of binary functions \(h:\mathcal Z\to\{0,1\}\) with
\(\mathrm{VC}(\mathcal{H})\le d\). Let \(z_1,\ldots,z_n\overset{\mathrm{i.i.d.}}{\sim}\mathcal{P}\), and define
\[
\Gamma(S):=\sup_{h\in \mathcal{H}}\left|\mathbb E[h(z)]-\frac1n\sum_{i=1}^n h(z_i)\right|.
\]
Then
\[
\mathbb E_S[\Gamma(S)]
\le
\frac{2}{\sqrt n}
+
4\sqrt{\frac{2d\log(en/d)}{n}}.
\]
Moreover, with probability at least \(1-\delta\),
\[
\Gamma(S)
\le
\frac{2}{\sqrt n}
+
4\sqrt{\frac{2d\log(en/d)}{n}}
+
\sqrt{\frac{\log(1/\delta)}{2n}}.
\]
\end{lemma}

\begin{proof}
Let \(S'=\{z'_1,\ldots,z'_n\}\) be an independent ghost sample drawn from \(\mathcal{P}\), and let
\(\sigma_1,\ldots,\sigma_n\) be independent Rademacher random variables. By symmetrization,
\[
\mathbb E_S[\Gamma(S)]
\le
2\mathbb E_{S,\sigma}\left[
\sup_{h\in H}\left|\frac1n\sum_{i=1}^n\sigma_i h(z_i)\right|
\right]
=
2\mathbb E_S[\widehat{\mathfrak R}_S(\mathcal{H})].
\]
Since \(\mathcal{H}\) is binary and has VC dimension \(d\), Sauer's lemma gives
\[
\pi_H(n)\le \left(\frac{en}{d}\right)^d, \quad n \ge d 
\]
where \(\pi_H(n)\) is the growth function of \(\mathcal{H}\). Combining this with the standard Rademacher bound using Massart's lemma for binary VC classes yields
\[
\mathbb E_S[\widehat{\mathfrak R}_S(\mathcal{H})]
\le
\frac{1}{\sqrt n}
+
2\sqrt{\frac{2d\log(en/d)}{n}}.
\]
Therefore,
\[
\mathbb E_S[\Gamma(S)]
\le
2\left(
\frac{1}{\sqrt n}
+
2\sqrt{\frac{2d\log(en/d)}{n}}
\right)
=
\frac{2}{\sqrt n}
+
4\sqrt{\frac{2d\log(en/d)}{n}}.
\]
This proves the expectation bound.

For the high-probability part, observe that changing one sample point \(z_i\) changes
\(\frac1n\sum_{i=1}^n h(z_i)\) by at most \(1/n\) for every \(h\in H\), and therefore changes
\(\Gamma(S)\) by at most \(1/n\) (Lemma~\ref{lem:bounded-difference}). Hence \(\Gamma(S)\) satisfies the bounded-difference condition. By McDiarmid's inequality, with probability at least \(1-\delta\),
\[
\Gamma(S)\le \mathbb E_S[\Gamma(S)]+\sqrt{\frac{\log(1/\delta)}{2n}}.
\]
Substituting the expectation bound above proves the claim.
\end{proof}

\begin{lemma}[Bounded difference for uniform deviations]
\label{lem:bounded-difference}
Let \(\mathcal{H}\) be a class of functions \(h:\mathcal Z\to[0,1]\), and let
\[
\Gamma(S):=\sup_{h\in H}\left|\mathbb E[h(z)]-\frac1n\sum_{i=1}^n h(z_i)\right|,
\]
where \(S=(z_1,\ldots,z_n)\). If \(S\) and \(S'\) differ in exactly one coordinate, then
\[
|\Gamma(S)-\Gamma(S')|\le \frac1n.
\]
\end{lemma}
\begin{proof}
Suppose \(S\) and \(S'\) differ only in the \(j\)-th coordinate. For every \(h\in \mathcal{H}\),
\[
\left|\frac1n\sum_{i=1}^n h(z_i)-\frac1n\sum_{i=1}^n h(z'_i)\right|
=
\frac1n |h(z_j)-h(z'_j)|
\le \frac1n,
\]
because \(h\in[0,1]\). Taking the supremum over \(h\in \mathcal{H}\) preserves the same bound.
\end{proof}

\begin{lemma}[Counting lemma]
\label{lem:counting}
Let $\{h_{\theta}: \theta \in \Theta\}$ be a class of $\{0,1\}$-valued functions on a set $\mathcal{Z}$, parameterized by $\Theta \subseteq \mathbb{R}^D$, and let $\Theta = \Theta_1 \cup \cdots \cup \Theta_K$ be a partition that does not depend on the input. Suppose that for every $z \in \mathcal{Z}$ and $k \in [K]$ there are functions $\lambda_{z,k,1}, \dots, \lambda_{z,k,s}: \mathbb{R}^D \to \mathbb{R}$, each affine in $\theta$, and a Boolean function $B_{z,k}$ such that $h_{\theta}(z) = B_{z,k}(\text{sgn } \lambda_{z,k,1}(\theta), \dots, \text{sgn } \lambda_{z,k,s}(\theta))$ for all $\theta \in \Theta_k$, where $\text{sgn}$ takes values in $\{-, 0, +\}$. If $m$ points are shattered and $sm \geq D$, then $2^m \leq K (8esm/D)^D$. In particular, if $D =d+1$, $K \leq d$ and $s\leq 3$, then the $\text{VC}$-dimension of the class is less than $11(d+1)$.
\end{lemma}
\begin{proof}
    Let $z_1, \dots, z_m$ be shattered. Padding with constant functions if necessary, we may assume that exactly $s$ functions are attached to each pair $(z_i, k)$. Fix $k$. For $\theta \in \Theta_k$, the labeling $(h_{\theta}(z_1 ), \dots, h_{\theta}(z_m))$ is determined by the sign assignment to the $sm$ affine functions $\{\lambda_{z_i, k,r}\}_{i \in [m], r \in [s]}$ at $\theta$. Since $sm \ge D$, Corollary $2.1$ of \citet{goldberg1993bounding}, applied to polynomials of degree one in $D$ variables, bounds the number of consistent sign assignments by $(8esm/D)^D$. Summing over $k \in [K]$, at most $K(8esm/D)^D$ labelings are realizable, and shattering requires $2^m$ of them. \\
    For the second claim, pad with constant functions so that $s=3$. It suffices to consider $m \ge d+1$, so that $3m \ge D$. With logarithms to base $2$, the first claim gives 
    \[
        m \le \log d + (d+1) \log \frac{24em}{d+1}.
    \]
    Since $\log d \le d+1$, the ratio $u \coloneqq m/ (d+1)$ satisfies $u \le 1 + \log (24eu)$. The function $u \to 1 + \log (24eu) -u$ is decreasing for $u \ge 2$ and negative at $u=11$, so $u<11$, \emph{i.e.}, $m < 11(d+1)$.
\end{proof}

\Convergence*
\begin{proof} 
For every \(f\in\mathcal F\), define the induced oracle-loss function \(h_f:\mathcal X_0 \times \{0,1\} \times [0,1]\to\{0,1\}\) by
\[
h_f(x,y,u):=\mathds 1\{f(x^f)\neq \mathds 1[ y = 1 \vee u\le p^f(x)]\}.
\]
The tie-breaking rule ensures that \(x^f\) is well defined, and hence \(h_f\) is a well-defined binary function. Let \(z=(x,y,u) \sim \mathcal{P}\),  where \(x\sim \mathcal D\), $y \mid x \sim \mathrm{Bernoulli}(\eta(x))$, and \(u \sim\mathrm{Unif}[0,1]\) independently, and let \(z_i=(x_i, y_i, u_i)\) for \(i\in[n]\). Since \(u\sim\mathrm{Unif}[0,1]\) independently of $y$, and $\mathbb{P}(y=1 \mid x) = \eta(x)$, conditioning on $x$ gives $\mathbb{P}(y = 1 \vee u \leq p^f(x) \mid x) = \eta(x) + (1 - \eta(x))p^f(x) = \eta(x^f)$, by the improvement model of Appendix~\ref{sec:imp_probability_formulation}. Hence $\mathds{1}(y =1 \vee u \le p^f(x)) \sim \mathrm{Bernoulli}(\eta(x^f))$ conditionally on $x$. Therefore,
\[
\mathbb E_{z \sim \mathcal{P}}[h_f(z)]
=
\mathbb E_{x \sim \mathcal{D}}\!\left[\mathds 1\{f(x^f)\neq y_{x^f}\}\right]
=
\texttt{err}_{\mathrm{Imp}}(f).
\]
Moreover, by the Definition~\ref{def:oracle_labels} oracle-label construction  \(y^f_{i}:=\mathbf 1\{y_i = 1 \vee u_i\le p^f(x_i)\}\), we have
\[
\frac1n\sum_{i=1}^n h_f(z_i)
=
\frac1n\sum_{i=1}^n \mathbf 1[f(x_i^f)\neq y^f_{i}]
=
\widehat{\texttt{err}}_{\mathrm{Imp}}(f).
\]
Hence the uniform generalization gap in the theorem is exactly
\[
\Gamma(S):=
\sup_{f\in\mathcal F}
\left|
\mathbb E_{z \sim \mathcal{P}}[h_f(z)]
-
\frac1n\sum_{i=1}^n h_f(z_i)
\right|,
\]
where \(S=(z_1,\ldots,z_n)\). The induced oracle-loss class \(H=\{h_f:f\in\mathcal F\}\) has VC dimension \(\tilde{d}\). Applying Lemma~\ref{lem:vc-uniform-deviation} to this class gives
\[
\mathbb E_S[\Gamma(S)]
\le
\frac{2}{\sqrt n}
+
4\sqrt{\frac{2 \tilde{d} \log(en/\tilde{d})}{n}}.
\]
Next, by Lemma~\ref{lem:bounded-difference}, changing one augmented sample point \(z_i=(x_i,y_i, u_i)\) changes \(\Gamma(S)\) by at most \(1/n\). Therefore, McDiarmid's inequality implies that, with probability at least \(1-\delta\),
\[
\Gamma(S)
\le
\mathbb E_S[\Gamma(S)]
+
\sqrt{\frac{\log(1/\delta)}{2n}}.
\]
Combining the last two displays yields
\[
\sup_{f\in\mathcal F}
\left|
\texttt{err}_{\mathrm{Imp}}(f)
-
\widehat{\texttt{err}}_{\mathrm{Imp}}(f)
\right|
\le
\frac{2}{\sqrt n}
+
4\sqrt{\frac{2 \tilde{d} \log(en/\tilde{d})}{n}}
+
\sqrt{\frac{\log(1/\delta)}{2n}}.
\]
This proves the theorem.
\end{proof}

\VCdimensioninducedoracle*
\begin{proof}
    Write $\theta = (w,b)$ and $\eta(x) = g({w^{\star}}^{\top}x)$, where $w^{\star} \in \mathbb{R}^d_{\ge 0}$ is the single index parameter of Assumption~\ref{ass:singleIndex}; since $w^{\star}$, $\eta$ and $g$ are fixed, they enter only as coefficients. Let $g^{-1}(u) \coloneqq \inf \{t \in \mathbb{R}: g(t) \ge u \}$ denote the generalized inverse of the single-index link function. Partition the parameter set into the regions $\Theta_j \coloneqq \{\theta: j^{\star} = j \}, j \in [d]$, where $j^{\star}$ is the minimal maximizer of $w_j/\alpha_j$. These are disjoint by the tie-breaking rule and do not depend on the input. Under Assumption~\ref{assump: cost function}, on the move branch,  $w_{j^{\star}} > 0$ (Section~\ref{sec:structural}).  \\
    Fix $j, \theta \in \Theta_j$ and input $(x,y,u)$. On $\Theta_j$, $S(w) = \beta w_j / \alpha_j$, so by Theorem~\ref{thm:strategic-shift}, $f(x^f) = \mathds{1}\{\lambda_2(\theta) \ge 0\}$ with $\lambda_2(\theta) \coloneqq w^{\top}x - b + \beta w_j /\alpha_j$. The agent moves $\text{iff}$ $\lambda_1(\theta) \coloneqq b - w^{\top}x > 0$ and $\lambda_2(\theta) \geq 0$, in which case $x^f = x +  \tfrac{b- w^{\top}x}{w_j}e_j$; otherwise $x^f = x$. \\
    If $\eta(x) = 1$, then $p^f(x) = 0$ by the convention of Appendix~\ref{sec:imp_probability_formulation}, so the event $\{u \le p^f(x)\}$ reduces to $\{ u \le 0\}$, which does not depend upon $\theta$. Assume now that $\eta(x) < 1$. If the agent does not move, then $x^f = x$ and $p^f(x) = 0$, and the event again reduces to $\{u \leq 0 \}$. If the agent moves, then ${w^{\star}}^{\top}x^f \ge {w^{\star}}^{\top}x$ because $w^{\star}_j \geq 0$, so $p^f(x) \in [0,1]$. Since $1-\eta(x) > 0$ and $g$ is non-decreasing and continuous (Assumption~\ref{ass:singleIndex}), 
    $u \leq p^f(x) \iff g({w^{\star}}^{\top}x^f) \ge u(1 - \eta(x)) + \eta(x) \iff \kappa(x,u) \le {w^{\star}}^{\top}x^f$,
    where $\kappa(x,u) \coloneqq g^{-1}\bigl(u(1-\eta(x)) + \eta(x)\bigr)$. If $\kappa(x,u) = + \infty$ (resp. $-\infty$), this condition is false (resp. true) for every $\theta$. Otherwise, substituting ${w^{\star}}^{\top}x^f ={w^{\star}}^{\top}x + w^{\star}_j(b-w^{\top}x)/w_j$ and multiplying by $w_j > 0$, it is equivalent to $\lambda_3(\theta) \ge 0$, with $\lambda_3(\theta) \coloneqq w^{\star}_j (b - w^{\top}x) - w_j (\kappa(x,u) - {w^{\star}}^{\top}x)$. \\
    Consequently, on $\Theta_j$, both $f(x^f)$ and the oracle label $\mathds{1}\{ y =1 \vee u \le p^f(x) \}$, and hence $h_f(x,y,u)$ are fixed Boolean functions of the signs of $\lambda_1(\theta),\lambda_2(\theta),\lambda_3(\theta)$ and of the input-only quantities $y$ and $\mathds{1}\{u \leq 0\}$. Each $\lambda_i$ is affine in $\theta \in \mathbb{R}^{d+1}$, with coefficients depending only on $(x,u)$. Lemma~\ref{lem:counting}, applied with $D=d+1$, $K=d$ and $s=3$, gives $\text{VC}(\mathcal{H}) < 11(d+1)$.
\end{proof}

\section{Generalization Bounds for \textsc{Strat-Imp-Aware}} 
\label{sec:generation_bounds}
In this section we prove Theorem~\ref{thm:alg_guarantee} for the plug-in version of $\textsc{Strat-Imp-Aware}$, in which the unknown class-probability function $\eta$ is replaced by a plug-in estimator $\widehat{\eta}$: each agent's improvement probability $\widehat{\pi}_{\mathrm{Imp}}$ under the deployed model is computed from $\widehat{\eta}$ by Algorithm~\ref{alg:sim_oracle}, and the loss of Algorithm~\ref{alg:strat_imp_error} is minimized. Throughout, let $r \coloneqq \sup_{z \in \mathcal{A}} \|z\|_2 < \infty$, $\rho_{\mathrm{Imp}} \coloneqq r + \beta / \min_{j \in [d]} \alpha_j$, $\varepsilon_{\eta} \coloneqq \sup_{z \in \mathcal{A}} | \widehat{\eta}(z) - \eta(z)|$, and $\mathcal{W}_{k,B} \coloneqq \{(w,b): w \in \mathbb{R}^d_{\geq  0}, \|w\|_2 \leq k, |b| \leq B\}$; for $(w,b) \in \mathcal{W}_{k, B}$ define $S(w) \coloneqq \beta \max_{j \in [d]} w_j / \alpha_j$, $M_{w,b}(x) \coloneqq w^{\top}x - b + S(w)$, and $C_{k,B} \coloneqq k\rho_{\mathrm{Imp}} + B$. We take the radius over the feasible report set $\mathcal{A}$ rather than over $\mathcal{X}_0$ so that the same constant serves the estimation analysis of Appendix~\ref{sec:oracle_free}, where $\widehat{\eta}$ and $\eta$ are compared at reported points; since $\mathcal{X}_0 \subseteq \mathcal{A}$, it follows that $|M_{w,b}(x)| \leq |w^{\top}x| + |b| + S(w) \leq C_{k,B} $ for every $x \in \mathcal{X}_0$. Here $\varepsilon_{\eta}$ is treated as a given tolerance; a bound of order $\mathcal{O}(m^{-1/2})$ in the pre-deployment sample size $m$ is established in Appendix~\ref{sec:oracle_free}.

\paragraph{Improvement probability.} For $x \in \mathcal{X}_0$, let $x^f$ denote the selected best response to $f_{w,b}$, \emph{i.e.}, the maximizer of $\beta \mathds{1}(w^{\top}z \ge b) - \sum_{j} \alpha_j (z_j - x_j)_+$ over $z \in \mathcal{A}$ selected by the canonical tie-breaking rule of Theorem~\ref{thm:strategic-shift}: $x^f = x + \tfrac{b-w^{\top}x}{w_{j^{\star}}}e_{j^{\star}}$ with $j^{\star} = \min \arg \max_{j \in [d]} w_j / \alpha_j$ if $0< b- w^{\top}x \leq S(w)$, and $x^f = x$ otherwise. Algorithm~\ref{alg:sim_oracle} computes the improvement probability 
\[
    \widehat{\pi}_{\mathrm{Imp}}(x; w, b) = \operatorname{clamp}_{[0,1]}\left( \frac{\widehat{\eta}(x^f) - \widehat{\eta}(x)}{\max\left\{ 1 - \widehat{\eta}(x), \, \delta_{\mathrm{g}} \right\}} \right), 
\]
the guard $\delta_g >0$ prevents division by a vanishing denominator; we take $\delta_g \leq 1 - \sup_{z \in \mathcal{A}} \widehat{\eta}(z)$, which is always possible since $\widehat{\eta}(z) = g(\widehat{w}^{\top}z) \leq g(rB_{\eta}) < 1$ on $\mathcal{A}$ (see Appendix~\ref{ass:F.1_estimation_regularity} for details). Under this choice the guard is never active, so it is a safeguard for finite-precision arithmetic rather than an analytical assumption. Algorithm~\ref{alg:strat_imp_error} combines $\widehat{\pi}_{\mathrm{Imp}}$ with the agent's observed pre-manipulation label $y$ into
\[
    q_{w,b}(x,y) \coloneqq y + (1-y) \widehat{\pi}_{\mathrm{Imp}}(x; w, b) \in [0,1],
\]
the probability that the post-response label is positive. Equivalently, $q_{w,b}(x,y) = \mathbb{P}(y_{\mathrm{Imp}}=1 \mid x,y)$ for the improvement-aware label $y_{\mathrm{Imp}} \coloneqq \mathds{1}(\{y=1\} \vee \{u \leq \widehat{\pi}_{\mathrm{Imp}}(x; w,b)\})$ with $u \sim \text{Unif}[0,1]$ independent of $(x,y)$. This construction enforces $\mathbb{P}(y_{\mathrm{Imp}} = 1 \mid y=1, x) = 1$ exactly and assigns $\mathbb{P}(y_{\mathrm{Imp}} = 1 \mid y = 0, x) = \widehat{\pi}_{\mathrm{Imp}}(x; w, b)$, mirroring the improvement model of Appendix~\ref{sec:imp_probability_formulation}, in which the true post-response label satisfies $\mathbb{P}(y^f = 1 \mid y = 1, x) = 1$ and $\mathbb{P}(y^f = 1 \mid y=0, x) = p^f(x)$ with $p^f(x) = (\eta(x^f) - \eta(x))/ (1 - \eta(x))$. The label $y_{\mathrm{Imp}}$ serves only the analysis; the algorithm never samples it. Since $y \mid x \sim \text{Bernoulli}(\eta(x))$, 
\[
    \mathbb{E}[ q_{w,b}(x,y) \mid x ] = \widehat{q}_{w,b}(x) \coloneqq \eta(x) + (1 - \eta(x)) \widehat{\pi}_{\mathrm{Imp}}(x; w,b),
\]
the estimator thus entering only through the conditional improvement probability for agents with $y = 0$. Lemma~\ref{lem:inducedlabeldeviation} shows that $\widehat{q}_{w,b}(x)$ deviates from the target $\eta(x^f)$ by at most $2 \varepsilon_{\eta}$.

\paragraph{Plug-in strategic improvement loss.} 
Writing $z \coloneqq (x,y)$, we define the conditional loss and the per-point loss of Algorithm~\ref{alg:strat_imp_error}, a hinge surrogate of the improvement-aware error,
\begin{align*}
    \ell^{\widehat{q}}_{w,b}(x) &\coloneqq \widehat{q}_{w,b}(x)(1 - M_{w,b}(x ))_+\,+\,(1 - \widehat{q}_{w,b}(x))(1 + M_{w,b}(x))_+, \\
    \ell_{w,b}(z) &\coloneqq q_{w,b}(x,y)(1 - M_{w,b}(x ))_+\,+\,(1 - q_{w,b}(x,y))(1 + M_{w,b}(x))_+,
\end{align*}
together with the empirical plug-in strategic improvement loss
\[
    \widehat{L}^{\widehat{q}}_{\mathrm{hinge}}(w, b) \coloneqq \frac{1}{n} \sum_{i=1}^n \ell_{w,b}(z_i),
\]
where $z_i = (x_i, y_i)$ are $\text{i.i.d.}$ with $x_i \sim \mathcal{D}$ and $y_i \mid x_i \sim \text{Bernoulli}(\eta(x_i))$. Since $\ell_{w,b}(z)$ is affine in $q_{w,b}(x,y)$ and $\mathbb{E}[q_{w,b}(x,y) \mid x] = \widehat{q}_{w,b}(x)$, we have $\mathbb{E}[\ell_{w,b}(z) \mid x] = \ell^{\widehat{q}}_{w,b}(x)$, so $\widehat{L}^{\widehat{q}}_{\mathrm{hinge}}(w,b)$ is an unbiased estimate of $L^{\widehat{q}}_{\mathrm{hinge}}(w,b) \coloneqq \mathbb{E}_{x \sim \mathcal{D}}[\ell^{\widehat{q}}_{w,b}(x)]$. In addition, $\ell_{w,b}(z)$ is the expectation over $u$ of the hinge loss evaluated at the label $y_{\mathrm{Imp}}$, so weighting by $q_{w,b}$ rather than sampling $y_{\mathrm{Imp}}$ removes the sampling noise in $u$ without changing the population risk. Moreover, $q_{w,b}(x,y) \in [0,1]$ makes $\ell_{w,b}(z)$ a convex combination of the two hinges, so $0 \leq \ell_{w,b}(z)\leq 1 + | M_{w,b}(x) | \leq 1 + C_{k,B}$. Two features of this construction are used below: $\widehat{L}^{\widehat{q}}_{\mathrm{hinge}}$ is the objective minimized by $\textsc{Strat-Imp-Aware}$ and is computable, whereas $\ell^{\widehat{q}}_{w,b}$ involves the unknown $\eta$ and is not; since $\widehat{L}^{\widehat{q}}_{\mathrm{hinge}}$ involves no auxiliary randomization, it is a deterministic function of $(w,b)$ given the sample, as the uniform convergence argument requires.

\begin{lemma}[Induced label deviation]
    \label{lem:inducedlabeldeviation} 
    Let $(w,b) \in \mathcal{W}_{k,B}$ and $x \in \mathcal{X}_0$. Then
    \[
        | \widehat{q}_{w,b}(x) - \eta(x^f) | \leq 2 \varepsilon_{\eta}
    \]
\end{lemma}
\begin{proof}
    By the choice of $\delta_g$, the denominator of $\widehat{\pi}_{\mathrm{Imp}}(x; w,b)$ equals $1 - \widehat{\eta}(x)$. We can write $\widehat{\pi}_{\mathrm{Imp}}(x; w, b) = \text{clamp}_{[0,1]}(\tilde{\pi})$ with $\tilde{\pi} \coloneqq (\widehat{\eta}(x^f) - \widehat{\eta}(x)) / ( 1 - \widehat{\eta}(x) )$; under Assumption~\ref{ass:F.1_estimation_regularity}, $\widehat{\eta}(x) \leq \widehat{\eta}(x^f) < 1$, so $\tilde{\pi} \in [0,1]$ and the clamp is inactive. Decomposing over the observed label, the improvement model of Appendix~\ref{sec:imp_probability_formulation} gives $\eta(x^f) = \eta(x) + (1- \eta(x))p^f(x)$ with true improvement probability $p^f(x) = (\eta(x^f) - \eta(x)) / (1 - \eta(x)) \in [0,1]$; thus $\widehat{\pi}_{\mathrm{Imp}}$ is the plug-in estimate of $p^f(x)$. Using $\widehat{q}_{w,b}(x) = \eta(x) + (1 - \eta(x)) \tilde{\pi}$ and $(1 - \widehat{\eta}(x)) \tilde{\pi} = \widehat{\eta}(x^f) - \widehat{\eta}(x)$, 
    \begin{align*}
    \widehat{q}_{w,b}(x) - \eta(x^f) &= \underbrace{\widehat{\eta}(x) + (1 - \widehat{\eta}(x))\tilde{\pi}}_{=\widehat{\eta}(x^f)} - \eta(x^f) - (\widehat{\eta}(x) - \eta(x))(1 - \tilde{\pi}) \\
    &= [\widehat{\eta}(x^f) - \eta(x^f)] - [\widehat{\eta}(x) - \eta(x)](1 - \tilde{\pi}).
    \end{align*}
    Both $x$ and $x^f$ lie in $\mathcal{A}$, so each bracketed term is at most $\varepsilon_{\eta}$ in absolute value, and $| 1- \tilde{\pi}| \leq 1 $. The triangle inequality gives the claim.
\end{proof}

\dfracUpper*
\begin{proof}
Let $\mathcal{G}_{k,B} \coloneqq \{ \ell_{w,b}: (w,b) \in \mathcal{W}_{k,B} \}$, regarded as a class of real-valued functions on $\mathcal{Z} \coloneqq \mathcal{X}_0 \times \{0,1\}$. Condition on the estimation sample $S_{\eta}$. Since $\widehat{\eta}$ is estimated independently of the training sample, $\mathcal{G}_{k,B}$ is then a fixed class of measurable functions on $\mathcal{Z}$, and the points $z_1, \dots, z_n$ remain $\text{i.i.d}$. Furthermore, each $\ell_{w,b}(z_i)$ is a deterministic function of $(w,b)$, since the loss involves no auxiliary randomization. All bounds below hold for every realization of $\widehat{\eta}$ and therefore unconditionally. Write $R \coloneqq 1 + C_{k,B}$ so that $0 \leq g \leq R$ for every $g \in \mathcal{G}_{k, B}$ by the setup.

\medskip
\noindent
\emph{Pseudo-dimension.} By definition \citep{anthony2009neural}, the pseudo-dimension of a class $\mathcal{F}$ of real-valued functions is the $\text{VC}$-dimension of $\{(x,\tau) \mapsto \text{sgn}(f(x) - \tau): f \in \mathcal{F} \}$, with the threshold $\tau \in \mathbb{R}$ part of the input. \\
For $\mathcal{M}_{k,B}$: since $S(w)$ does not depend on $x$, $M_{w,b}(x) = w^{\top}x - (b - S(w))$, so $\mathcal{M}_{k,B}$ is contained in the $(d+1)$-dimensional vector space of affine functions on $\mathbb{R}^d$,  whose pseudo-dimension is $d+1$ using Theorem $11.4$ from \citet{anthony2009neural}. Hence $\mathrm{Pdim}(\mathcal{M}_{k,B}) \leq d + 1$. \\
For $\Pi_{k,B}$: Under decomposable costs the selected best response is $x^f = x + \Delta e_{j^{\star}}$, where $j^{\star} \in \arg\max_{j \in [d]} w_j / \alpha_j$ with the stated tie-break and $\Delta = (b- w^{\top}x)/w_{j^{\star}}$, when the agent moves and $x^f =x$ otherwise. The division is well defined as the move branch requires $w^{\top}x < b$ and $M_{w,b}(x) \geq 0$, which are incompatible at $w=0$ since $M_{0,b}(x) = -b$, and $w\neq 0$ with $w \in \mathbb{R}^d_{\geq 0}$ gives $\max_{j \in [d]} w_j /\alpha_j > 0$, hence $w_{j^{\star}} > 0$. Partition $\mathcal{W}_{k,B}$ into the regions $C_j \coloneqq \{ (w,b) \in \mathcal{W}_{k,B}: j^{\star} = j \}$, $j \in [d]$; by the tie-breaking rule these are disjoint and defined by linear inequalities in $w$ that do not involve the input. As $g$ is non-decreasing and continuous (Assumption~\ref{ass:singleIndex}), and both the guard and the clamp are inactive (proof of Lemma~\ref{lem:inducedlabeldeviation}), the test $\widehat{\pi}_{\mathrm{Imp}} \geq \tau$ is equivalent to 
\[
    \kappa(x, \tau) \leq \widehat{w}^{\top}x^f, \qquad \kappa(x,\tau) \coloneqq g^{-1}(\tau(1 - \widehat{\eta}(x)) + \widehat{\eta}(x)),
\]
where $g^{-1}(v) \coloneqq \inf\{s: g(s) \geq v \}$; conditioning on $S_{\eta}$ fixes $\widehat{\eta}$ and $\widehat{w}$, so $\kappa(x, \tau)$ and $\widehat{w}^{\top}x$ are functions of the input alone. Fix $j$ and $\theta = (w,b) \in C_j$. The agent moves $\text{iff}$ $\ell_1(\theta) \coloneqq b - w^{\top}x > 0$ and $\ell_2(\theta) \coloneqq \beta w_j / \alpha_j -b + w^{\top}x \geq 0$, in which case $\widehat{w}^{\top}x^f = \widehat{w}^{\top}x + \widehat{w}_j (b-w^{\top}x)/ w_j$ and, multiplying by $w_j > 0$, the test becomes $\ell_3 (\theta) \coloneqq \widehat{w}_j(b - w^{\top}x) - w_j(\kappa(x,\tau) - \widehat{w}^{\top}x) \geq 0$; if $\kappa(x,\tau) = \infty$ (resp. $-\infty$), the test fails (resp. holds) for every $\theta$, and we set $\ell_3 :\equiv -1$ (resp. $\ell_3 :\equiv 1$). Otherwise $x^f = x$ and the test reduces to $\tau \leq 0$. Each $\ell_i$ is affine in $\theta \in \mathbb{R}^{d+1}$ with coefficients depending only on $(x, \tau)$, so on $C_j$ the test is a fixed Boolean function of the signs of three affine functions of $\theta$. Lemma~\ref{lem:counting}, applied with $D=d+1$, $K=d$ and $s=3$, therefore gives $\mathrm{Pdim}(\Pi_{k,B}) < 11 (d+1)$, and together with $\mathrm{Pdim}(\mathcal{M}_{k,B}) \leq d+1$, yields $\mathfrak{d} = \mathrm{Pdim}(\Pi_{k,B})+\mathrm{Pdim}(\mathcal{M}_{k,B}) < 12(d+1) = \mathcal{O}(d)$.

\medskip
\noindent
\emph{Covering numbers.} For a function $f$ defined on $\mathcal{Z}$, define the empirical norms 
\[
    \| f \|_{L_p(P_n)} \coloneqq \left( \frac{1}{n} \sum_{i=1}^n | f(z_i)|^p \right)^{1/p},  \qquad p \in \{1,2\}
\]
A function $f$ defined on $\mathcal{X}_0$ is identified with a function on $\mathcal{Z}$ by setting $f(z_i) \coloneqq f(x_i)$, so that all norms below are evaluated at the same sample points $z_1, \dots, z_n$. All covering numbers are internal, $\emph{i.e.}$, cover elements belong to the class being covered, and all logarithms are natural. Write $\theta = (w,b)$, $\widehat{\pi}_{\theta} \coloneqq \widehat{\pi}_{\mathrm{Imp}}(\cdot; w,b)$, $M_{\theta} \coloneqq M_{w,b}$, $q_{\theta} \coloneqq q_{w,b}$ and $\ell_{\theta} \coloneqq \ell_{w,b}$ and abbreviate $\| \cdot \|_{L_2(P_n)}$ by $\|\cdot\|_n$. We bound $\mathcal{N}(\epsilon, \mathcal{G}_{k,B}, L_2(P_n))$ for $0<\epsilon \leq R$ below in three steps. \\
$1.$ Lipschitz decomposition: Let $\psi(q, m) \coloneqq q(1-m)_+\,+\,(1-q)(1+m)_+$ so that $\ell_{\theta}(z) = \psi (q_{\theta}(x,y), M_{\theta}(x))$ with $q_{\theta}(x,y) = y + (1-y) \widehat{\pi}_{\theta}(x)$. Since $|(1-m)_+ - (1+m)_+| \leq 2|m|$ for all $m \in \mathbb{R}$, the map $q \mapsto \psi(q,m)$ is $2R$-Lipschitz as $|m| \leq C_{k,B} \leq R$. Moreover, for every $q \in [0,1]$ the map $m \mapsto \psi(q,m)$ is $1$-Lipschitz, being the sum of a $q$-Lipschitz and a $(1-q)$-Lipschitz function. Fix $\theta, \theta' \in \mathcal{W}_{k,B}$ and apply the triangle inequality 
\[
    |\psi(q,m) - \psi(q', m')| \leq |\psi(q,m) - \psi(q',m)| + |\psi(q', m) - \psi(q', m')|
\]
with $(q,m) = (q_{\theta}(x,y), M_{\theta}(x))$ and $(q',m') = (q_{\theta'}(x,y), M_{\theta'}(x))$; the first term uses $|M_{\theta}(x)| \leq C_{k,B}$ and the second uses $q_{\theta'}(x,y) \in [0,1]$. Together with $| q_{\theta}(x,y) - q_{\theta'}(x,y) | = (1-y)|\widehat{\pi}_{\theta}(x) - \widehat{\pi}_{\theta'}(x)| \leq |\widehat{\pi}_{\theta}(x) - \widehat{\pi}_{\theta'}(x)|$, this gives, pointwise, |$\ell_{\theta}(z) - \ell_{\theta'}(z)| \leq 2R|\widehat{\pi}_{\theta}(x) - \widehat{\pi}_{\theta'}(x)| + |M_{\theta}(x) - M_{\theta'}(x)|$ and Minkowski's inequality yields 
\begin{equation}
    \label{eq:minkowski_E2}
    \| \ell_{\theta} - \ell_{\theta'} \|_n \leq 2R \|\widehat{\pi}_{\theta} - \widehat{\pi}_{\theta'}\|_n + \|M_{\theta} - M_{\theta'}\|_n
\end{equation}
$2$. Product cover: Take minimal covers $\{\pi_s\}_{s=1}^{N_{\Pi}}$ of $\Pi_{k,B}$ at scale $\epsilon/(8R)$ and $\{\mu_t\}_{t=1}^{N_\mathcal{M}}$ of $\mathcal{M}_{k,B}$ at scale $\epsilon/4$ in $\|\cdot\|_n$. By the cover property, each $\theta \in \mathcal{W}_{k,B}$ can be assigned to a pair $(s,t)$ with $\|\widehat{\pi}_{\theta} - \pi_s \|_n \leq \epsilon /(8R)$, $\|M_{\theta} - \mu_t\|_n \leq \epsilon/4$. For each occupied pair, select one representative $\theta_{s,t} \in \mathcal{W}_{k,B}$ assigned to it. If $\theta$ is assigned to $(s,t)$, the triangle inequality gives 
\[
    \|\widehat{\pi}_{\theta} - \widehat{\pi}_{\theta_{s,t}} \|_n \leq \|\widehat{\pi}_{\theta} - \pi_s \|_n + \|\pi_s - \widehat{\pi}_{\theta_{s,t}} \|_n \leq \frac{\epsilon}{4R} , \qquad \|M_{\theta} - M_{\theta_{s,t}}\|_n \leq \|M_{\theta} - \mu_t\|_n + \|\mu_t - M_{\theta_{s,t}}\|_n \leq \frac{\epsilon}{2},
\]
and hence, by~(\ref{eq:minkowski_E2}), $\| \ell_{\theta} - \ell_{\theta_{s,t}} \|_n \leq 2R \epsilon/(4R) + \epsilon/2 = \epsilon$. Thus the representative losses $\{\ell_{\theta_{s,t}}\}$ form an internal $\epsilon$-cover of $\mathcal{G}_{k,B}$ with at most $N_{\Pi}N_{\mathcal{M}}$ elements, since there are at most $N_{\Pi}N_{\mathcal{M}}$ occupied index pairs $(s,t)$, proving 
\begin{equation}
    \label{eq:covering_num_prod_E2}
    \mathcal{N}(\epsilon, \mathcal{G}_{k,B}, L_2(P_n)) \leq \mathcal{N}\left(\frac{\epsilon}{8R}, \Pi_{k,B}, L_2(P_n) \right) \mathcal{N}\left(\frac{\epsilon}{4}, \mathcal{M}_{k,B}, L_2(P_n) \right)
\end{equation}
$3$. Bounding the components: Let $\mathcal{F}$ be a class of functions taking values in the interval of length $A > 0$, with pseudo-dimension $d_{\mathcal{F}}$. Rescaling $\mathcal{F}$ to $[0,1]$ leaves its pseudo-dimension unchanged and scales distances by $1/A$, so applying Theorem $18.4$ of \citet{anthony2009neural} which bounds the $u$-packing number and hence the internal $u$-covering number (a maximal $u$-packing is an internal $u$-cover), gives $\mathcal{N}(u,\mathcal{F},L_1(P_n)) \leq e(d_{\mathcal{F}}+1)(\frac{2eA}{u})^{d_{\mathcal{F}}}$, where $0< u \leq A$. Since $|f -g| \leq A$ pointwise for $f,g \in \mathcal{F}$, we have $\|f - g\|^2_n \leq A \|f - g \|_{L_1(P_n)}$, so an internal $(\delta^2/A)$-cover in $L_1(P_n)$ is an internal $\delta$-cover in $L_2(P_n)$. Consequently, for $0 < \delta \leq A$, 
\begin{equation}
    \label{eq:covering_num_L2_E_2}
    \log \mathcal{N}(\delta, \mathcal{F}, L_2(P_n)) \leq \log(e (d_{\mathcal{F}} + 1)) + 2d_{\mathcal{F}}\log \left( \frac{A\sqrt{2e}}{\delta} \right).
\end{equation}
Apply~(\ref{eq:covering_num_L2_E_2}) to $\Pi_{k,B}$ which takes values in $[0,1]$, with $(A, \delta) = (1, \epsilon/(8R))$, and to $\mathcal{M}_{k,B}$, which takes values in $[-C_{k,B}, C_{k,B}] \subseteq [-R, R]$, with $(A, \delta) = (2R, \epsilon/4)$; both choices satisfy $\delta \leq A$ since $\epsilon \leq R$, and in both cases the argument of the logarithm equals $8R\sqrt{2e}/\epsilon$. Writing $d_{\Pi} \coloneqq \mathrm{Pdim}(\Pi_{k,B})$ and $d_{\mathcal{M}} \coloneqq \mathrm{Pdim}(\mathcal{M}_{k,B})$, so that $\mathfrak{d} = d_{\Pi} + d_{\mathcal{M}}$, taking logarithms in~(\ref{eq:covering_num_prod_E2}) and using $d_{\Pi}, d_{\mathcal{M}} \leq \mathfrak{d}$ in the constant terms gives
\begin{equation}
    \label{eq:dfrac_compared_to_pseudodim} 
    \log \mathcal{N}(\epsilon, \mathcal{G}_{k,B}, L_2(P_n)) \leq 2 \log (e(\mathfrak{d} + 1)) + 2\mathfrak{d} \log \left( \frac{8R\sqrt{2e}}{\epsilon} \right).
\end{equation}
Thus $\mathfrak{d}$ plays the role of a pseudo-dimension for $\mathcal{G}_{k,B}$: it controls the coefficient of $\log(1/\epsilon)$ in its $\log$-covering number.
\end{proof}
\begin{lemma}[Uniform convergence for plug-in strategic improvement loss]
\label{lem:plugin-hinge-uniform}
Let \(\widehat\eta\) be estimated independently of the training sample, and let $\mathfrak{d}$ be as in Lemma~\ref{lem:eff_dim}. Then, for every $\delta \in (0,1)$, with probability at least $1 - \delta$, uniformly over $(w,b) \in \mathcal{W}_{k, B}$, 
\[
    L_{\mathrm{hinge}}^{\widehat{q}}(w,b) \leq \widehat{L}_{\mathrm{hinge}}^{\widehat{q}}(w,b) + c (1 + C_{k,B}) \sqrt{\frac{\mathfrak{d}}{n}} + (1 + C_{k, B}) \sqrt{\frac{\log(1/\delta)}{2n}}
\]
where $c > 0$ is an absolute constant. 
\end{lemma}
\begin{proof}
We work in the setting of the proof of Lemma~\ref{lem:eff_dim}: conditionally on $S_{\eta}$, $\mathcal{G}_{k,B}$ is a fixed class of measurable functions on $\mathcal{Z}$, the points $z_1, \dots, z_n$ are $\text{i.i.d.}$, and every $g \in \mathcal{G}_{k,B}$ satisfies $0\leq g \leq R$, with $R=1+C_{k,B}$; all bounds below hold for every realization of $\widehat{\eta}$, and therefore unconditionally. By (\ref{eq:dfrac_compared_to_pseudodim}) in that proof, for $0 < \epsilon \le R$, 
\[
    \log \mathcal{N}(\epsilon, \mathcal{G}_{k,B}, L_2(P_n)) \leq 2 \log (e(\mathfrak{d} + 1)) + 2\mathfrak{d} \log \left( \frac{8R\sqrt{2e}}{\epsilon} \right).
\]

\medskip
\noindent
\emph{Rademacher complexity.}
By Dudley's entropy integral at $\alpha = 0$ (Lemma $A.3$ of \citet{srebro2010smoothness}), using $\sup_{g \in \mathcal{G}_{k,B}} \sqrt{\widehat{\mathbb{E}}[g^2]} \leq R$ for its upper limit, 
\[
    \widehat{\mathfrak{R}}_n(\mathcal{G}_{k,B}) \le C_0 \int_{0}^R \sqrt{\frac{\log \mathcal{N}(\epsilon, \mathcal{G}_{k,B}, L_2(P_n))}{n}} d \epsilon
\]
for an absolute constant $C_0 > 0$. Substituting $\log \mathcal{N}(\epsilon, \mathcal{G}_{k,B}, L_2(P_n))$ from (\ref{eq:dfrac_compared_to_pseudodim}), rescaling $\epsilon = Rs$, which cancels $R$ inside the logarithm, and using $\sqrt{a+b } \leq \sqrt{a} + \sqrt{b}$ twice together with $\int_{0}^1 \sqrt{\log(1/s)} ds = \sqrt{\pi}/2$, 
\[
    \int_0^R \sqrt{\log \mathcal{N}(\epsilon, \mathcal{G}_{k,B}, L_2(P_n))} d\epsilon \leq R \biggl( \sqrt{2 \log (e(\mathfrak{d} +1 ))} + \sqrt{2 \mathfrak{d}} \left( \sqrt{\log(8\sqrt{2e})} + \frac{\sqrt{\pi}}{2}\right) \biggr) \leq c_1 R \sqrt{\mathfrak{d}}, 
\]
where the last step uses $\log (e(\mathfrak{d} + 1)) \le 1 + \mathfrak{d} \le 2 \mathfrak{d}$ (as $\mathfrak{d} \ge 1$) and $c_1 >0$ is an absolute constant. Hence, with $c \coloneqq 2 C_0 c_1 > 0 $, gives 
\[
    \widehat{\mathfrak{R}}_n(\mathcal{G}_{k,B}) \leq \frac{c}{2}R\sqrt{\frac{\mathfrak{d}}{n}}
\]
This bound holds for every sample, hence also bounds $\mathfrak{R}_n(\mathcal{G}_{k,B}) \coloneqq \mathbb{E}[\widehat{\mathfrak{R}}_n(\mathcal{G}_{k,B})] \leq \frac{c}{2}R\sqrt{\frac{\mathfrak{d}}{n}} $. 

\medskip
\noindent
\emph{Uniform convergence.} Replacing one $z_i$ alters $\sup_{g \in \mathcal{G}_{k,B}} \bigl(\mathbb{E}[g(z)] - \frac{1}{n}\sum_{i \in [n]}g(z_i)\bigr)$ by at most $R/n$, so symmetrization together with McDiarmid's inequality gives, with probability at least $1-\delta$, uniformly over $g \in \mathcal{G}_{k,B}$, 
\[
    \mathbb{E}[g(z)] \leq \frac{1}{n} \sum_{i=1}^ng(z_i) + 2 \mathfrak{R}_n(\mathcal{G}_{k,B}) + R\sqrt{\frac{\log (1/\delta)}{2n}}.
\]
Applying this to $g = \ell_{w,b}$: the empirical average is $\widehat{L}^{\widehat{q}}_{\mathrm{hinge}}(w,b)$ by definition, while unbiasedness and the tower property give $\mathbb{E}_z[\ell_{w,b}(z)] = \mathbb{E}_{x \sim \mathcal{D}}[\ell^{\widehat{q}}_{w,b}(x)] = L_{\mathrm{hinge}}^{\widehat{q}}(w,b)$. Substituting $R = 1 + C_{k,B}$ completes the proof.
\end{proof}

\AlgGuarantee*
\begin{proof}[Proof of Theorem~\ref{thm:alg_guarantee}]
Recall \(\mathcal W_{k,B}:=\{(w,b):w\in\mathbb R_{\ge0}^d,\|w\|_2\le k, |b|\le B\}\) and 
\(C_{k,B}:=k\rho_{\mathrm{Imp}}+B\). Fix \((w,b)\in\mathcal W_{k,B}\), and let
\(x^{f}\in\mathcal A\) be the selected best response of \(x\in\mathcal X_0\) to \(f_{w,b}\). Write
\(S(w):=\beta\max_{j\in[d]}w_j/\alpha_j\) and
\(M_{w,b}(x):=w^\top x-b+S(w)\). By the strategic-shift characterization under decomposable 
costs (Theorem~\ref{thm:strategic-shift}), the post-response prediction satisfies
\(f_{w,b}(x^{f})=\mathds 1\{M_{w,b}(x)\ge0\}\).

Under the improvement-aware model, the true post-response label satisfies $\mathbb{P}(y^f = 1 \mid x) = \eta(x^f)$, so the conditional improvement-aware \(0/1\) loss at \(x\) is
\( \mathbb{E}[\mathds{1} \{f_{w,b}(x^f) \neq y^f\} \mid x ] = \eta(x^{f})\mathds 1\{M_{w,b}(x)<0\}+(1-\eta(x^{f}))\mathds 1\{M_{w,b}(x)\ge0\}\). Since \(\mathds 1\{M <0\}\le(1-M)_+\) and \(\mathds 1\{M \ge0\}\le(1+M)_+\), taking expectations over $x \sim \mathcal{D}$ gives,
\begin{equation}\label{eq:true_indicator_hinge}
    \texttt{err}_{\mathrm{Imp}}(f_{w,b})
\le
L_{\mathrm{hinge}}^\eta(w,b) 
\end{equation} 
where $L_{\mathrm{hinge}}^\eta(w,b) \coloneqq \mathbb E_{x\sim D}\!\left[
\eta(x^{f})(1-M_{w,b}(x))_+
+
(1-\eta(x^{f}))(1+M_{w,b}(x))_+
\right]$. Recall the plug-in strategic improvement risk
\[
L_{\mathrm{hinge}}^{\widehat q}(w,b)
=
\mathbb E_{x\sim D}\!\left[
\widehat q_{w,b}(x)(1-M_{w,b}(x))_+
+
(1-\widehat q_{w,b}(x))(1+M_{w,b}(x))_+
\right].
\]
The two risks differ only in the weight attached to the two hinges, so
\[
\left|L_{\mathrm{hinge}}^\eta(w,b)-L_{\mathrm{hinge}}^{\widehat q}(w,b)\right|
\le
\mathbb E_{x\sim D}\!\left[
|\eta(x^{f})-\widehat q_{w,b}(x)|
\left|(1-M_{w,b}(x))_+-(1+M_{w,b}(x))_+\right|
\right].
\]
By Lemma~\ref{lem:inducedlabeldeviation}, $| \eta(x^f) - \widehat{q}_{w,b}(x)| \leq 2 \varepsilon_{\eta}$. Moreover,
\(|(1-m)_+-(1+m)_+|\le2|m|\) for all \(m\in\mathbb R\). Finally, for all
\(x\in\mathcal X_0 \subseteq \mathcal{A}\),
\[
|M_{w,b}(x)|
\le
|w^\top x|+|b|+S(w)
\le
k r+B+\frac{\beta k}{\min_{j\in[d]}\alpha_j}
=
k\rho_{\mathrm{Imp}}+B
=
C_{k,B}.
\]
Hence
\begin{equation}\label{eq:true_hinge_and_empirical}
    L_{\mathrm{hinge}}^\eta(w,b)
\le
L_{\mathrm{hinge}}^{\widehat q}(w,b)
+
4C_{k,B}\varepsilon_\eta .
\end{equation}
By Lemma~\ref{lem:eff_dim} and Lemma~\ref{lem:plugin-hinge-uniform}, with probability at least \(1-\delta\), uniformly over
\((w,b)\in\mathcal W_{k,B}\),
\begin{equation}\label{eq:hinge_and_its_empirical}
    L_{\mathrm{hinge}}^{\widehat q}(w,b)
\le
\widehat L_{\mathrm{hinge}}^{\widehat q}(w,b)
+
c (1 + C_{k,B})\sqrt{\frac{\mathfrak{d}}{n}}
+
(1+C_{k,B})\sqrt{\frac{\log(1/\delta)}{2n}} .
\end{equation}
Combining \eqref{eq:true_indicator_hinge}, \eqref{eq:true_hinge_and_empirical} and \eqref{eq:hinge_and_its_empirical}, with probability at least $1 - \delta$, uniformly over \((w,b)\in\mathcal W_{k,B}\) gives, 
\[
\texttt{err}_{\mathrm{Imp}}(f_{w,b})
\le
\widehat L_{\mathrm{hinge}}^{\widehat q}(w,b)
+
c (1 + C_{k,B})\sqrt{\frac{\mathfrak{d}}{n}}
+
(1+C_{k,B})\sqrt{\frac{\log(1/\delta)}{2n}}
+
4C_{k,B}\varepsilon_\eta .
\]
Since \textsc{Strat-Imp-Aware} returns \((\tilde{w},\tilde{b})\in\mathcal W_{k,B}\), and the above holds uniformly on $\mathcal{W}_{k,B} $, it applies to the returned model irrespective of how it was selected; substituting
\((w,b)=(\tilde{w},\tilde{b})\) proves the theorem.
\end{proof}

\begin{lemma}[Gradient bound]
    \label{lem:grad_bound}
    Let $\alpha_{\min} \coloneqq \min_{j \in [d]} \alpha_j$ and $\widehat{\eta} \coloneqq g(\widehat{w}^{\top} \cdot)$. Under  Assumption~\ref{ass:F.1_estimation_regularity}, for every $(w,b)$ with $w \in \mathbb{R}^d_{\ge 0} \setminus \{0\}$ and every $z = (x,y)$ with $x \in \mathcal{X}_0$, every (sub)gradient of $\ell_{w,b}(z)$, computed with $j^{\star}$ and the manipulation status held fixed, satisfies $\| \nabla_w  \ell_{w,b} (z) \|_2 \leq G_w$ and $| \partial_b \ell_{w,b}(z)| \leq G_b$, where 
    \[
        G_w \coloneqq r + \frac{\beta}{\alpha_{\min}} + \frac{2 \beta L_g B_{\eta}r}{\alpha_{\min}\delta_g}, \qquad G_b \coloneqq 1 + \frac{2\beta L_g B_{\eta}}{\alpha_{\min} \delta_g}.
    \]
    If $g$ is the logistic function, the factor $L_g / \delta_g$ can be replaced by $1$. 
\end{lemma}
\begin{proof}
    Write $\ell_{w,b}(z) = \Phi(q,M)$ with $q = q_{w,b}(x)$, $M = M_{w,b}(x)$ and $\Phi(q,M) \coloneqq q(1-M)_+ + (1-q) (1 + M)_+$, so that $\nabla \ell = (\partial \ell / \partial M) \nabla M + (\partial \ell / \partial q) \nabla q$. Every subgradient satisfies $\partial \ell / \partial M \in [ -q, 1-q]$, hence $| \partial \ell / \ell M | \leq 1$; moreover $\nabla_w M  = x + \beta s$ for a subgradient $s$ of $w \to \max_{j} w_j / \alpha_j$, with $\|s\|_2 \leq 1/ \alpha_{\min}$, and $\partial_b M = -1$. Since $\|x\|\leq r$, the first term contributes at most $r + \beta / \alpha_{\min}$ to $\|\nabla_w \ell \|_2$ and at most $1$ to $|\partial_b \ell |$. The second term vanishes unless $y=0$, the agent manipulates, and the clamp is inactive, since otherwise $q$ is locally constant. In that case $0 \leq M < S(w) = \beta w_{j^{\star}} / \alpha_{j^{\star}}$, so $| \partial \ell / \partial q | = | (1-M)_+ - (1 + M)_+ | \leq 2M \leq 2 \beta w_{j^{\star}}/ \alpha_{j^{\star}}$. Further, $\nabla q = \tfrac{g'(\widehat{w}^{\top}x^f)\widehat{w}_{j^{\star}}}{\max\{1- \widehat{\eta}(x), \delta_g\}}\nabla x^f_{j^{\star}}$ with $\nabla_w x^f_{j^{\star}} = -x^f / w_{j^{\star}}$ and $\partial_b x^f_{j^{\star}} = 1 / w_{j^{\star}}$. From Assumption~\ref{ass:F.1_estimation_regularity}, $g' \leq L_g$, $\widehat{w}_{j^{\star}}\leq B_{\eta}$, $\|x^f\|_2 \leq r$ (as $x^f \in \mathcal{A}$) and $\alpha_{j^{\star}} \geq \alpha_{\min}$, the factor $w_{j^{\star}}$ cancels, and the second term contributes at most $2 \beta L_g B_{\eta} r/ (\alpha_{\min} \delta_g)$ to $\|\nabla_w \ell \|_2$ and at most $2 \beta L_g B_{\eta}/ (\alpha_{\min} \delta_g)$ to $|\partial_b \ell|$. For the logistic link, $g' = g(1-g)$ and under Assumption~\ref{ass:F.1_estimation_regularity}, $\widehat{\eta}(x^f) \geq \widehat{\eta}(x)$ give $g'(\widehat{w}^{\top}x^f)/(1- \widehat{\eta}(x)) = \widehat{\eta}(x^f)(1- \widehat{\eta}(x^f))/(1- \widehat{\eta}(x)) \leq 1$, which replaces $L_g / \delta_g$ by $1$. 
\end{proof}

\section{Oracle-Free Estimation Bounds} 
\label{sec:oracle_free}

Theorem~\ref{thm:uniform_convergence} establishes a generalization guarantee under access to a true post-deployment label oracle. This section removes that assumption. We replace the oracle with the labeling rule of Algorithm~\ref{alg:sim_oracle}, which is built from a class-probability estimator $\widehat{\eta}$ learned entirely from pre-deployment data and quantify the resulting price. The analysis proceeds via Theorem~\ref{thm:alg_guarantee}, which bounds $\texttt{err}_{\mathrm{Imp}}$ in terms of deviation of the label process from $\eta(\cdot^{f})$; our task is to show that this deviation is controlled by the pre-deployment estimator error, and to bound the latter.  

\paragraph{Estimation sample and independence.} Let estimation sample be
\[
    S_{\eta} = \{x_i, y_i\}_{i=1}^m, \qquad x_i \overset{\text{i.i.d}}{\sim} \mathcal{D}, \qquad y_i \mid x_i \sim \mathrm{Bernoulli}(\eta(x_i)),
\]
drawn independently of the classifier's sample $S_n = \mathcal{D}_{tr} \cup \mathcal{D}_{val}$, with $n \coloneqq |S_n|$ and $m \coloneqq |S_{\eta}|$. The estimation sample lies in $\mathcal{X}_0 \coloneqq \operatorname{supp}(\mathcal{D})$ \textemdash{} the estimator $\widehat{\eta}$ is fit only on pre-deployment data, and no post-deployment or off-support label is used at any point. It is fit once and passed to Algorithm~\ref{alg:strat_imp_error} as a frozen input, so it is not co-adapted to the the classifier being learned. Independence of $S_{\eta}$ and $S_n$ is a hypothesis of Theorem~\ref{thm:alg_guarantee}, discharged in the proof of Lemma~\ref{lem:plugin-hinge-uniform} by conditioning on $S_{\eta}$. It is used again below, when the theorem is invoked for the random estimator $\widehat{\eta}$: conditioning on $S_{\eta}$ freezes $\widehat{\eta}$ into a deterministic function while, by independence, $S_n$ remains $\text{i.i.d.}$ from $\mathcal{D}$; so the conclusion holds for every realization of $S_{\eta}$ and hence, on averaging, unconditionally.

\paragraph{The estimator.} We take $\widehat{\eta}(x) = g(\widehat{w}^{\top}x)$ with $\widehat{w} \in \arg \min_{w \in \mathcal{W}} \widehat{L}_m(w)$, the empirical cross-entropy risk on $S_{\eta}$. Writing $t \coloneqq w^{\top}x$, 
\[
\ell(w; (x,y)) = -y \log g(w^{\top} x) - (1-y)\log (1 - g(w^{\top}x)) = \varphi_y(t)
\]
where $\varphi_1(t) = - \log g(t)$ and $\varphi_0(t) = - \log (1 - g(t))$. Since the loss depends on $w$ only through the scalar $w^{\top}x$, the chain rule gives $\nabla_w \ell (w; (x,y)) = \varphi'_y(w^{\top}x)x$, with 
\[
\varphi'_{1}(t) = - \frac{g'(t)}{g(t)}, ~ \varphi'_{0}(t) = \frac{g'(t)}{1 - g(t)}, ~\text{equivalently} ~ \varphi_y'(t) = \frac{g'(t)(g(t) - y)}{g(t) (1 - g(t))},
\]
and 
\[
    \varphi''_1(t) = - \frac{g''(t)}{g(t)} + \biggl(\frac{g'(t)}{g(t)}\biggr)^2, ~ \varphi_0''(t) = \frac{g''(t)}{1 - g(t)} + \biggl( \frac{g'(t)}{1 - g(t)} \biggr)^2.
\]
We write the population risk and the empirical  risk for estimating $\eta$ as $L(w) \coloneqq \mathbb{E}_{(x,y)}[\ell(w; (x,y))]$ and  $\widehat{L}_m(w) \coloneqq \frac{1}{m} \sum_{i=1}^m \ell(w; (x_i, y_i))$, set $\Delta_{\eta}(w) \coloneqq \nabla \widehat{L}_m(w) - \nabla L(w)$, and $\Phi \coloneqq \sup_{w \in \mathcal{W_{\eta}}} \| \Delta_{\eta}(w) \|_2$. Recall from Appendix~\ref{sec:generation_bounds} that the classification-stage guarantee depends upon $\widehat{\eta}$ only through $\varepsilon_{\eta} = \sup_{z \in \mathcal{A}} | \widehat{\eta}(z) - \eta(z)|$, and not on the loss used to obtain it; cross entropy is therefore one sufficient choice, and any independently estimated $\widehat{\eta}$ with $\varepsilon_{\eta} = \mathcal{O}(m^{-1/2})$ yields the same conclusion. 

\paragraph{On-support estimation and off-support evaluation.} Two regions appear in the analysis and are not interchangeable. All the estimation lemmas below are statements about $\mathcal{D}$ alone: the sample $S_{\eta}$, the risks $L$ and $\widehat{L}_m$, and the strong-convexity constant $\alpha$ are determined by the law of $x \sim \mathcal{D}$, hence depend on $\eta$ only through its values on $\mathcal{X}_0$. Extrapolation to reported points enters at a single step, the final Lipschitz bound in the proof of Theorem~\ref{thm:F_6_oracle_free_guarantee}, which converts the on-support parameter error $\| \widehat{w} - w^{\star} \|_2$ into a bound on $\varepsilon_{\eta}$ over $\mathcal{A}$, at the cost of a factor $L_gr$ with $r$ the radius of the report set rather than of $\mathcal{X}_0$. Two ingredients license this step. Assumption~\ref{ass:singleIndex} posits the single-index form on all of $\mathcal{X}$, hence on $\mathcal{A}$, so $\widehat{\eta}$ and $\eta$ share the same link $g$ and differ off support only through the parameter $\widehat{w} - w^{\star}$; and condition \ref{ass:F.1_estimation_regularity}(5) \textemdash{} strong convexity of the population cross-entropy risk, i.e. no direction of feature variation is unobserved under $\mathcal{D}$ \textemdash{} makes the parameter identifiable from $\mathcal{D}$ (Section~\ref{sec:F_4_identifiability}). Accuracy at reported points is therefore a consequence of pre-deployment estimation rather than a separate hypothesis.

\subsection{Assumptions}
\label{ass:F.1_estimation_regularity}
Throughout, $\mathcal{X}_0 \coloneqq \operatorname{supp}(\mathcal{D})$ and $\mathcal{A}$ denotes the feasible report set, so that $\mathcal{X}_0 \subseteq \mathcal{A} \subseteq \mathcal{X}$.

\paragraph{Assumption F.1 (Estimation regularity).} In addition to Assumption~\ref{ass:singleIndex}: 
\begin{enumerate}
    \item (\emph{Bounded reports}.) $r \coloneqq \sup_{z \in \mathcal{A}} \|z\|_2 < \infty$. 
    \item (\emph{Smooth link}.) $g$ is twice differentiable on $\mathcal{I} \coloneqq [-r B_{\eta}, rB_{\eta}]$, with $g'(t) \leq L_g < \infty$ and $|g''(t)| \leq L_{g'} < \infty$ for all $t \in \mathcal{I}$. 
    \item (\emph{Bounded probabilities.}) There is $\tau \in (0, 1/2)$ such that $\tau \leq g(w^{\top}z) \leq 1 - \tau $ for all $z \in \mathcal{A}$ and $w \in \mathcal{W}_{\eta}$.
    \item (\emph{Parameter set.}) $\mathcal{W}_{\eta} = \{ w \in \mathbb{R}^d_{\geq 0}: \|w \|_2 \leq B_{\eta} \}$, and $\| w^{\star}\| \leq B_{\eta}$. 
    \item (\emph{Population strong convexity.}) Population cross-entropy risk $L$ is $\alpha$-strongly convex on $\mathcal{W}_{\eta}$.
\end{enumerate}
Conditions $1-4$ are boundedness conditions yielding explicit constants; condition $5$ is the identifiability condition (Section~\ref{sec:F_4_identifiability}), subsumes strict monotonicity of $g$ on the observed range. We note that $\mathcal{X}$ itself need not be bounded: every point at which $\eta$, $\widehat{\eta}$, or the loss is evaluated lies in $\mathcal{A}$, and $\mathcal{X}_0 \subseteq \mathcal{A}$ gives $\|x_i\| \leq r$ for the estimation sample automatically. Consequently, for all $w \in \mathcal{W}_{\eta}$ and $z \in \mathcal{A}$, $|w^{\top}z| \leq \|w\|_2 \|z\|_2 \leq r B_{\eta}$, so conditions $2$ and $3$ are required only on the interval $\mathcal{I}$ and the set $\mathcal{A}$ respectively. \\
Condition $4$ constrains the estimator's parameter set, which is distinct from the classifier's $\mathcal{W}_{k, B}$: $B_{\eta}$ must be large enough to contain true single-index parameter $w^{\star}$, whereas $k$ is the capacity-control parameter for the classifier. The sign constraint mirrors Assumption~\ref{ass:singleIndex} and ensures that $\widehat{\eta}$ is non-decreasing along the best response direction, as used in Lemma~\ref{lem:inducedlabeldeviation}. Condition $3$ supplies the remaining hypothesis of that lemma: since $\widehat{\eta}(z) = g(\widehat{w}^{\top}z) \leq 1 - \tau$ for every $z \in \mathcal{A}$, any choice $\delta_g \leq \tau$ leaves the guard in $\widehat{\pi}_{\mathrm{Imp}}$ inactive. Define 
\[
    A_1 \coloneqq \frac{L_g}{\tau},~ A_2 \coloneqq \frac{L_{g'}}{\tau} +\biggl( \frac{L_g}{\tau}\biggr)^2,~ M \coloneqq A_1 r,~ H \coloneqq A_2 r^2.
\] 

\subsection{Estimation lemmas} 
The argument proceeds in four stages. Lemmas~\ref{lem:F.1_derivative_bounds} and~\ref{lem:F_2_Rademacher_complexity_gradient_class} bound the complexity of the gradient class of the cross-entropy loss; Lemma~\ref{lem:F_3_Uniform_gradient_deviation} converts this into a uniform deviation bound between empirical and population gradients; Lemmas~\ref{lem:F_4_Fischer_consistency} and~\ref{lem:F_5_paramter_estimation} convert that into a bound on the parameter error $\| \widehat{w} - w^{\star} \|$. Theorem~\ref{thm:F_6_oracle_free_guarantee} combines them with Theorem~\ref{thm:alg_guarantee}. 

\begin{lemma}[Derivative bounds.]
    \label{lem:F.1_derivative_bounds}
    Under Assumption~\ref{ass:F.1_estimation_regularity}~(1)--(3), for all $y \in \{0,1\}$ and $t \in \mathcal{I}$:  
    \[
        (a)~|\varphi'_y(t)| \leq A_1; \quad  (b)~ |\varphi''_y(t)| \leq A_2.
    \]
    Consequently, for all $w_1, w_2 \in \mathcal{W}_{\eta}$, $z \in \mathcal{A}$: 
    \[
        (c)~\| \nabla_w \ell(w; (z,y)) \|_2\leq M \quad (d)~\| \nabla_w \ell(w_1; (z, y)) - \nabla_w \ell(w_2; (z,y)) \|_2 \leq H \|w_1 - w_2 \|_2.
    \]
\end{lemma}
\begin{proof}
    By~\ref{ass:F.1_estimation_regularity}, $t = w^{\top}z \in \mathcal{I}$ whenever $z \in \mathcal{A}$ and $w \in \mathcal{W}_{\eta}$, so conditions (2) and (3) apply. \\ 
    \emph{(a)} The two derivative formulas have the same shape and both denominators are at least $\tau$ using condition (3):  
    \[
        | \varphi'_1(t) | = \frac{g'(t)}{g(t)} \leq \frac{L_{g}}{\tau}, ~ |\varphi'_0(t)| = \frac{g'(t)}{1 - g(t)} \leq \frac{L_g}{\tau}. 
    \]
    \emph{(b)} Applying $|g''| \leq L_{g'}$, $g' \leq L_g$ and $\min \{g, 1-g \} \geq \tau$ to the two second-derivative formulas gives $|\varphi''_y(t)| \leq L_{g'} / \tau + L^2_g / \tau^2 = A_2$ for both $y$. \\
    \emph{(c)} We know $\| \nabla_w \ell(w; (z,y)) \|_2 = | \varphi'_y(t)| \|z\|_2 \leq A_1 r = M $, by \emph{(a)} and condition (1). \\
    \emph{(d)} Write $t_i = w^{\top}_i z \in \mathcal{I}$. By \emph{(b)} and the mean value theorem, $\varphi'_y$ is $A_2$-Lipschitz on $\mathcal{I}$, so by Cauchy-Schwarz and condition (1), 
    \[
       \| \varphi'_y(t_1)z - \varphi'_y(t_2)z \|_2 =  | \varphi'_y(t_1) - \varphi'_y(t_2) |~\|z\|_2 \leq A_2 |t_1 - t_2 | r \leq A_2 r^2~\|w_1 - w_2 \|_2
    \]
\end{proof}

\begin{lemma}[Rademacher complexity of the gradient class] 
\label{lem:F_2_Rademacher_complexity_gradient_class}
Let $\epsilon_1, \dots, \epsilon_m$ be $\text{i.i.d.}$ Rademacher variables and $(z_t, y_t) \in \mathcal{A} \times \{0,1\}, t \in [m]$, be arbitrary deterministic points. Under conditions (1)--(4) of Assumption~\ref{ass:F.1_estimation_regularity}~,
\[
    \widehat{\mathfrak{R}}_{\| \cdot \|_2}(\nabla \ell \circ \mathcal{W}_{\eta}) = \mathbb{E}_{\epsilon} \sup_{w \in \mathcal{W}_{\eta}} \biggl \| \sum_{t=1}^{m} \epsilon_t \nabla_w \ell (w; (z_t, y_t)) \biggr \|_2  \leq 2 (H B_{\eta} + M) \sqrt{m}.
\]

\end{lemma}

\begin{proof} 
    Write $\ell (w; (z_t, y_t)) = G_t(F_t(w))$ with $G_t \coloneqq \varphi_{y_t}$ and $F_t(w) \coloneqq \langle w, z_t \rangle$, so that $\nabla_w \ell(w; (z_t, y_t)) = G'_t(\langle w, z_t \rangle)z_t$. From Assumption~\ref{ass:F.1_estimation_regularity}, $F_t(w) \in \mathcal{I}$ for every $w \in \mathcal{W}_{\eta}$, so Lemma~\ref{lem:F.1_derivative_bounds} applies throughout. 
    
    \medskip 
    \noindent 
    \emph{Chain rule.} By Lemma~\ref{lem:F.1_derivative_bounds}(a), $\sup_{s \in \mathcal{I}}|G_t'(s)| \leq A_1 \eqqcolon L_G$, $\| \nabla F_t(w) \|_2 = \|z\|_2 \leq r \eqqcolon L_F$, both uniformly in $t$ and $w$. Applying the chain rule for normed Rademacher complexity using Theorem $1$ of  \citet{foster2018uniform} with $K = 1$, and noting that $\nabla F_t(w) = z_t$ does not depend on $w$, 
    \[
        \widehat{\mathfrak{R}}_{\| \cdot \|_2} \leq 2r \mathbb{E}_{\epsilon} \sup_{w \in \mathcal{W}_{\eta}} \sum_{t=1}^m \epsilon_t \varphi'_{y_t}(\langle w,z_t \rangle) + 2A_1 \mathbb{E}_{\epsilon} \biggl \| \sum_{t=1}^m \epsilon_t z_t \biggr \|_2
    \]

    \medskip 
    \noindent 
    \emph{Scalar contraction.} By Lemma~\ref{lem:F.1_derivative_bounds}, each $\varphi'_{y_t}$ is $A_2$-Lipschitz on $\mathcal{I}$. The Ledoux-Talagrand contraction inequality (\citet{mohri2014learning}, Lemma 8) applied for a fixed sample to the real-valued class $\{z \mapsto \langle w, z \rangle: w \in \mathcal{W}_{\eta}\}$, together with $\sup_{w \in \mathcal{W}_{\eta}} \langle w, v \rangle \le B_{\eta} \|v \|_2$, gives 
    \[
        \mathbb{E}_{\epsilon} \sup_{w \in \mathcal{W}_{\eta}} \sum_{t=1}^m \epsilon_t \varphi'_{y_t}(\langle w,z_t \rangle) \leq A_2 \mathbb{E}_{\epsilon} \sup_{w \in \mathcal{W}_{\eta}} \bigl \langle w, \sum_{t=1}^m \epsilon_t z_t \bigr \rangle \leq A_2 B_{\eta} \mathbb{E}_{\epsilon} \biggl \| \sum_{t=1}^m \epsilon_t z_t \biggr \|_2
    \]

    \medskip
    \noindent 
    \emph{Rademacher average of the data.} By Jensen's inequality, the expected norm of the Rademacher sum is at most the square root of its expected square norm. Expanding the squared norm and using independence with $\mathbb{E}_{\epsilon} [\epsilon_s \epsilon_t ] = \mathds{1}\{s=t\}$ eliminate all cross terms. Since $\epsilon^2_t = 1$, 
    \[
        \mathbb{E}_{\epsilon} \biggl \| \sum_{t=1}^m \epsilon_t z_t \biggr \|_2 \le \sqrt{\sum_{t=1}^m \|z_t\|_2^2} \leq r\sqrt{m}. 
    \]

    \medskip 
    \noindent 
    Combining the above bound yields, 
    \[
    \widehat{\mathfrak{R}}_{\| \cdot \|_2} \leq 2r \cdot A_2 B_{\eta} r \sqrt{m} + 2A_1 r \sqrt{m} = 2(A_2 r^2 B_{\eta} + A_1 r) \sqrt{m} = 2(H B_{\eta} + M) \sqrt{m}.
    \]
\end{proof}
    
\begin{lemma}[Uniform gradient deviation.]
\label{lem:F_3_Uniform_gradient_deviation}
Under Assumption~\ref{ass:F.1_estimation_regularity}~(1)--(4), for any $\delta' \in (0,1)$, with probability at least $1 - \delta'$ over $S_{\eta}$, 
\[
    \Phi \leq \Lambda_{m}(\delta') \coloneqq \frac{4(HB_{\eta} + M)}{\sqrt{m}} + M \sqrt{\frac{2 \log(1/\delta')}{m}}
\]    
\end{lemma}

\begin{proof}
    Since $x_i \in \mathcal{X}_0 \subseteq \mathcal{A}$, Lemma~\ref{lem:F.1_derivative_bounds} applies at every sample point. In particular $w \mapsto \Delta_{\eta}(w)$ is continuous on $\mathcal{W}_{\eta}$ by Lemma~\ref{lem:F.1_derivative_bounds}~(d), so the supremum defining $\Phi$ can be restricted to a countable dense subset; hence $\Phi$ is measurable. 

    \medskip 
    \noindent
    \emph{Concentration about the mean.} Fix $i \in [m]$ and replace $(x_i, y_i)$ by arbitrary $(x'_i, y'_i) \in \mathcal{A} \times \{0,1\}$, writing $\Delta^{(i)}_{\eta}$ for the resulting deviation. The population term cancels, so for every $w \in \mathcal{W}_{\eta}$, by Lemma~\ref{lem:F.1_derivative_bounds}~(c), 
    \[
        \| \Delta_{\eta}(w) - \Delta^{(i)}_{\eta} \|_2 = \frac{1}{m} \bigl \| \nabla_w \ell(w;(x_i, y_i)) - \nabla_w \ell(w; (x'_i, y'_i))   \bigr \|_2  \leq \frac{2M}{m}.
    \]
    Since $|\sup_{w} a(w) - \sup_{w} b(w)| \leq \sup_{w} |a(w) - b(w)|$ for bounded $a, b$, together with the reverse triangle inequality, $\Phi$ has bounded differences $c_i = 2M/m$. McDiarmid's inequality with $\sum_{i} c^2_i = 4M^2/m$ gives $\mathbb{P}( \Phi - \mathbb{E}[\Phi] \geq s) \leq \exp(-ms^2/2M^2)$, whence 
    \begin{equation}
        \label{eq:F3_1}
        \Phi \leq \mathbb{E}[\Phi] + M \sqrt{\frac{2\log(1/\delta')}{m}} \qquad \text{with probability at least}~1- \delta' 
    \end{equation}
    
    \medskip
    \noindent
    \emph{Symmetrization.} Let $(x'_i, y'_i)_{i \leq m}$ be an independent copy of the sample. Since $\nabla L(w) = \mathbb{E}[\frac{1}{m} \sum_{i} \nabla_w\ell(w; (x'_i, y'_i))]$, convexity of the norm and Jensen's inequality give
    \[
        \mathbb{E}[\Phi] \leq \mathbb{E} \sup_{w \in \mathcal{W}_{\eta}} \biggl \| \frac{1}{m} \sum_{i=1}^m \biggl( \nabla_w \ell(w;(x_i, y_i)) - \nabla_w \ell(w;(x'_i, y'_i))  \biggr) \biggr \|_2
    \]
    For any fixed $\epsilon \in \{\pm1\}^m$, exchanging $(x_i, y_i)$ with $(x'_i, y'_i)$ on $\{i: \epsilon_i = -1\}$ leaves the joint law unchanged and flips the corresponding differences, so the expectation is unchanged upon inserting the signs $\epsilon_i$; averaging over Rademacher $\epsilon$ and applying the triangle inequality yields two terms equal in law, giving 
    \[
        \mathbb{E}[\Phi] \leq \frac{2}{m} \mathbb{E}~\widehat{\mathfrak{R}}_{\|\cdot\|}( \nabla \ell \circ \mathcal{W}_{\eta} ).
    \]
    Since Lemma~\ref{lem:F_2_Rademacher_complexity_gradient_class} holds for arbitrary deterministic points in $\mathcal{A} \times \{0,1\}$, its bound remains valid under the outer expectation: 
    \begin{equation}
        \label{eq:F3_2}
        \mathbb{E}[\Phi] \leq \frac{4(HB_{\eta} + M)}{\sqrt{m}}.
    \end{equation}
    Substituting inequality~\ref{eq:F3_2} into~\ref{eq:F3_1} completes the proof.
\end{proof}

\begin{lemma}[Fischer Consistency]
\label{lem:F_4_Fischer_consistency}
    Under Assumption~\ref{ass:singleIndex} and~\ref{ass:F.1_estimation_regularity}~(1)--(3) conditions, $\nabla L(w^{\star}) = 0$. 
\end{lemma}

\begin{proof}
    From the setup, 
    \[
        \nabla_w \ell(w; (x,y)) = \frac{g'(w^{\top}x) (g(w^{\top}x) -y)}{g(w^{\top}x)(1- g(w^{\top}x))}x.
    \]
    Under Assumption~\ref{ass:singleIndex}, $\mathbb{E}[y \mid x] = \eta(x) = g({w^{\star}}^{\top}x)$, so at $w = w^{\star}$ the factor $g({w^{\star}}^{\top}x) - \mathbb{E}[y \mid x ]$ vanishes and $\mathbb{E}[\nabla_w \ell(w^{\star}; (x,y)) | x] = 0$. The integrand is bounded by $M$ (Lemma~\ref{lem:F.1_derivative_bounds} (c)), so the tower property gives $\nabla L(w^{\star}) = 0$.
\end{proof}

\begin{lemma}[Parameter estimation.]
    \label{lem:F_5_paramter_estimation}
    Let $\widehat{w} \in \underset{w \in \mathcal{W}_{\eta}}{\arg \min}~\widehat{L}_m(w)$, which exists since $\mathcal{W}_{\eta}$ is compact and $\ell$ is continuous in $w$. Under Assumption~\ref{ass:F.1_estimation_regularity}~(4)--(5) and $\nabla L(w^{\star}) = 0$, deterministically, 
    \[
       \| \widehat{w} - w^{\star} \| \leq \frac{\Phi}{\alpha}.
    \]
\end{lemma}

\begin{proof}
    The set $\mathcal{W}_{\eta}$ is convex, being the intersection of a Euclidean ball with the non-negative orthant, and $\widehat{w}$ minimizes the differential function $\widehat{L}_m$ over it; hence $\widehat{w}$ satisfies the variational inequality $\langle \nabla \widehat{L}_m(\widehat{w}), w - \widehat{w} \rangle \geq 0$ for all $w \in \mathcal{W}_{\eta}$ for all $w \in \mathcal{W}_{\eta}$. Taking $w = w^{\star}$, which is feasible by Assumption~\ref{ass:singleIndex} and condition \emph{(4)}, 
    \begin{equation} 
        \label{eq:F_5_1}
        \langle \nabla \widehat{L}_m(\widehat{w}), \widehat{w} - w^{\star} \rangle \leq 0.
    \end{equation}
    This requires no convexity of $\widehat{L}_m$: any global minimizer over a convex set satisfies above bound. \\
    $\alpha$-strong convexity of $L$ on $\mathcal{W}_{\eta}$ is equivalent to strong monotonicity of $\nabla L$, which with $\nabla L(w^{\star}) = 0$ gives 
    \begin{equation}
        \label{eq:F_5_2}
        \alpha \| \widehat{w} - w^{\star} \|_2^2 \leq \langle \nabla L(\widehat{w}), \widehat{w} - w^{\star} \rangle.
    \end{equation}
    Decomposing the right-hand side of~\ref{eq:F_5_2} and applying Cauchy-Schwartz to the first term and~\ref{eq:F_5_1} to the second, 
    \[
        \langle \nabla L(\widehat{w}), \widehat{w} - w^{\star} \rangle = \langle - \Delta_{\eta}(\widehat{w}), \widehat{w} - w^{\star} \rangle + \langle \nabla \widehat{L}_m (\widehat{w}), \widehat{w} - w^{\star} \rangle \leq \Phi \| \widehat{w} - w^{\star} \|_2.
    \]
    Dividing by $\|\widehat{w} - w^{\star}\|_2$ when $\widehat{w} \neq w^{\star}$, the claim being trivial otherwise. 
\end{proof} 

\subsection{Main result} 

\begin{theorem}[Oracle-free guarantee]
    \label{thm:F_6_oracle_free_guarantee} 
    Suppose the assumptions of Theorem~\ref{thm:alg_guarantee} hold, and let Assumption~\ref{ass:F.1_estimation_regularity} hold. Let $\widehat{\eta}(x) = g(\widehat{w}^{\top}x)$ with $\widehat{w} \in \arg \min_{w \in \mathcal{W}_{\eta}}\widehat{L}_m(w)$ on the estimation sample $S_{\eta}$, drawn independently of the classifier's sample $S_n$, and let $\tilde{w}, \tilde{b}$ be the model returned by \textsc{Strat-Imp-Aware}. Then, for every $\delta \in (0,1)$, with probability at least $1- \delta$, 
    \[
        \texttt{err}_{\mathrm{Imp}}(f_{\tilde{w}, \tilde{b}}) \leq \widehat{L}^{\widehat{q}}_{\text{hinge}}(\tilde{w}, \tilde{b}) + c(1 + C_{k,B})\sqrt{\frac{\mathfrak{d}}{n}} + (1 + C_{k,B})\sqrt{\frac{2  \log(2/\delta)}{n}} + \frac{4 C_{k,B}L_gr}{\alpha}\Lambda_m(\delta/2),
    \]
    where $\Lambda_m(\delta/2) = \frac{4(HB_{\eta} + M)}{\sqrt{m}} + M\sqrt{\frac{2 \log(2/\delta)}{m}}.$
\end{theorem}

\begin{proof}
    Let $E_1$ denote the event on which the conclusion of Theorem~\ref{thm:alg_guarantee} holds at level $\delta/2$, and $E_2$ the event of Lemma~\ref{lem:F_3_Uniform_gradient_deviation} with $\delta' = \delta/2$. 

    \medskip 
    \noindent
    \emph{Bounding $\mathbb{P}(E^c_1)$.}  Theorem~\ref{thm:alg_guarantee} is a statement about a fixed estimator, whereas $\widehat{\eta}$ depends upon $S_{\eta}$. Conditioning on $S_{\eta}$ renders $\widehat{\eta}$, and hence $\varepsilon_{\eta}$, deterministic, while by independence $S_n$ remains $\text{i.i.d.}$ from $\mathcal{D}$; the theorem therefore applies for every realization of $S_{\eta}$, giving $\mathbb{P}(E^c_1 \mid S_{\eta}) \leq \delta/2$ pointwise, and hence $\mathbb{P}(E_1^c) \leq \delta /2$ by the law of total probability. On $E_1$, 
    \begin{equation}
        \label{eq:F_6_1}
        \texttt{err}_{\mathrm{Imp}}(f_{\tilde{w}, \tilde{b}}) \leq \widehat{L}^{\widehat{q}}_{\text{hinge}}(\tilde{w}, \tilde{b}) + c(1 + C_{k,B})\sqrt{\frac{\mathfrak{d}}{n}} + (1 + C_{k,B})\sqrt{\frac{2  \log(2/\delta)}{n}} + 4C_{k,B} \varepsilon_{\eta},
    \end{equation}
    the confidence term arising from Theorem~\ref{thm:alg_guarantee} at level $\delta /2$.

    \medskip
    \noindent 
    \emph{Bounding $\varepsilon_{\eta}$ on $E_2$.} By Lemma~\ref{lem:F_3_Uniform_gradient_deviation}, $\Phi \leq \Lambda_m(\delta/2)$; by Lemmas~\ref{lem:F_4_Fischer_consistency} and~\ref{lem:F_5_paramter_estimation}, $\|\widehat{w} - w^{\star} \| \leq \Lambda_m(\delta/2) / \alpha$. For any $z \in \mathcal{A}$, Assumption~\ref{ass:singleIndex} gives $\eta(z) = g({w^{\star}}^{\top}z)$ and $\widehat{\eta}(z) = g({\widehat{w}}^{\top} z)$ by construction; both arguments lie in $\mathcal{I}$ by Assumption~\ref{ass:F.1_estimation_regularity}, so $g$ is $L_g$-Lipschitz there, and by Cauchy-Schwartz with $\|z \|_2 \leq r$, 
    \[
        | \widehat{\eta}(z) - \eta(z) | = |g(\widehat{w}^{\top}z) - g({w^{\star}}^{\top}z)| \leq L_g |(\widehat{w} - w^{\star})^{\top}z| \leq L_g r \| \widehat{w} - w^{\star} \|_2.
    \]
    Taking the supremum over $z \in \mathcal{A}$ gives 
    \begin{equation}
        \label{eq:F_6_2}
        \varepsilon_{\eta} \leq \frac{L_g r}{\alpha} \Lambda_{m}(\delta/2).
    \end{equation}
    \emph{Combining the bounds.} By union bound, $\mathbb{P}(E^c_1 \cup E^c_2) \leq \delta/2 + \delta/2 = \delta$, so $E_1 \cap E_2$ holds with probability at least $1-\delta$. Substituting~\ref{eq:F_6_2} into~\ref{eq:F_6_1} proves the theorem. 
\end{proof}

\begin{remark}[Rate]
    \label{rem:F_7_rate}
    Treating $r, B_{\eta}, L_g, L_{g'}, \tau^{-1}$ and $\alpha^{-1}$ as constants, fixing $\delta$, we have $\Lambda_m(\delta/2) = \mathcal{O}(m^{-1/2})$. Thus, the additional error due to oracle-free estimation is $\mathcal{O}(m^{-1/2})$, and the overall rate is $\mathcal{O}(\sqrt{\mathfrak{d}/n} + m^{-1/2})$, where  $\mathfrak{d}  = \mathcal{O}(d)$. In particular, $m \gtrsim n$, the estimation term is bounded by a constant multiple of the classifier term, yielding an overall rate of $\mathcal{O}(\sqrt{\mathfrak{d}/n})$. Hence, oracle-free estimation preserves the asymptotic rate of oracle-based bound.
\end{remark}

\subsection{Condition 5 as an identifiability condition} \label{sec:F_4_identifiability}
Under Assumption~\ref{ass:singleIndex}~, $\eta(x) = g({w^{\star}}^{\top}x)$, so evaluating the Hessian of $L$ at $w = w^{\star}$ gives 
\[
    \nabla^2 L(w^{\star}) = \mathbb{E}_{x \sim \mathcal{D}} \biggl[ \frac{g'({w^{\star}}^{\top}x)^2}{g({w^{\star}}^{\top}x) (1- g({w^{\star}}^{\top}x))}~ xx^{\top} \bigg], 
\]
a product of a factor measuring whether $\eta$ responds to feature variation and one measuring whether the features vary at all. Condition 5 asserts that this $g'$-weighted design is non-degenerate, and therefore excludes both ways in which $w^{\star}$ can fail to be identified from $\mathcal{D}$: if there is a vector $a \neq 0$ with $a^{\top}x = 0$ and $x \sim \mathcal{D}$-\emph{a.s.}, then $g(({w^{\star}} + \lambda a)^{\top}x) = \eta(x)$ for every $\lambda \in \mathbb{R}$ and every $x \in \mathcal{X}_0$, so $w^{\star}$ and $w^{\star} + \lambda a$ are observationally indistinguishable at any sample size; and if $g' \equiv 0$ on a region of the observed range, the corresponding points contribute nothing to the Hessian. Quantitatively, since $g' \leq L_g$ and $g(1-g) \geq \tau (1 - \tau)$ on $\mathcal{I}$, 
\[
    \gamma \coloneqq \lambda_{\min}(\mathbb{E}_{x \sim \mathcal{D}}[xx^{\top}]) \geq \frac{\alpha \tau (1-\tau)}{L^2_g} > 0,
\]
that is, $\operatorname{supp}(\mathcal{D})$ spans $\mathbb{R}^d$. Conversely, if $\gamma >0$ and $g'(t) \geq c > 0$ on $\mathcal{I}$, then $\nabla^2 L(w^{\star}) \succeq 4c^2\gamma I$, using $g(1-g) \leq 1/4$; a lower bound on $g'$ is thus sufficient for condition 5 at $w^{\star}$ but is not assumed, being implied on the observed range by the condition 5 itself. The same two quantities govern the regularity conditions used throughout the generalized-linear-model literature in \citet{foster2018uniform}. Crucially, the identifiability content of condition 5 is verifiable on the pre-deployment sample: $\gamma$ is estimated by the $\lambda_{\min}$ of the empirical second-moment matrix $\frac{1}{m} \sum_{i=1}^m x_i x_i^{\top}$; values for each dataset are reported in Section~\ref{app:exp}.

\subsection{Sufficient conditions for condition 5}
\label{sec:F_5_Sufficient_condition_5} 
Condition 5 imposes an additional restriction beyond Assumption~\ref{ass:singleIndex}~: for a general non-decreasing link, the cross-entropy risk need not be convex in $w$. We give two sufficient conditions with explicit strong-convexity constants. Write $t \coloneqq w^{\top}x$ and suppress the arguments of $g, \eta$ and $\varphi_y$ where no ambiguity arises. Since $\mathbb{E}_{y \mid x}[\varphi''_y] = \eta \varphi''_1 + (1 - \eta)\varphi''_0$ and $\nabla^2 L(w) = \mathbb{E}_{x \sim \mathcal{D}}[ \mathbb{E}_{y \mid x}[\varphi''_y] xx^{\top}]$, it suffices to obtain a uniform lower bound on the conditional expected curvature. Let $c>0$ satisfy $g' \geq c$ on $\mathcal{I}$, and suppose $\gamma \coloneqq \lambda_{\min}(\mathbb{E}_{x \sim \mathcal{D}}[xx^{\top}]) > 0$; Section~\ref{sec:F_4_identifiability}~shows this is also necessary for Condition 5. \\[2pt]
\emph{$\log$-concave link.} Suppose $g$ and $1-g$ are $\log$-concave on $\mathcal{I}$, equivalently $\varphi''_1 \geq 0$ and $\varphi''_0 \geq 0$, as for the logistic and probit links. Weighting the two second derivatives by $g$ and $1-g$ cancels $g''$ terms, giving 
\[
    g \varphi''_1 + (1-g)\varphi''_0 = \frac{(g')^2}{g(1-g)} \geq 4c^2,
\]
where the inequality uses $g(1-g) \leq 1/4$. Non-negativity of the second derivatives then gives $\varphi''_1 + \varphi''_0 \geq 4c^2$. Since $\eta(x) = g({w^{\star}}^{\top}x) \in [\tau, 1-\tau]$ by conditions 3 and 4, we obtain $\mathbb{E}_{y \mid x}[\varphi''_y] \geq \tau(\varphi''_1 + \varphi''_0) \geq 4\tau c^2$, uniformly over $w \in \mathcal{W}_{\eta}$ and $x \in \mathcal{X}_0$. Thus condition 5 holds with $\alpha = 4\tau c^2 \gamma$; the factor $\tau$ is the price of a bound uniform in $w$ rather than one at $w^{\star}$ alone, where $\eta = g$ and the constant improves to $4c^2 \gamma$ (Section~\ref{sec:F_4_identifiability}). Strict $\log$-concavity is unnecessary: since $g \varphi''_1 + (1-g)\varphi''_0 = (g')^2/[g(1-g)] > 0$ when $g' > 0$, the two second derivatives cannot vanish simultaneously. \\ 
For the special case of logistic link $g = \sigma$, where $\sigma(t) = (1 + e^{-t})^{-1}$, the identity $\sigma' = \sigma(1 - \sigma)$ gives sharper constants. Specifically, $\varphi'_y(t) = \sigma (t) - y$ and $\varphi''_1 = \varphi''_0 = \sigma'$, so $\nabla^2_w \ell(w;(x,y)) = \sigma'(w^{\top}x)~xx^{\top}$, independently of the label. Since $| w^{\top} x| \leq r B_{\eta}$, we may take $\alpha = \sigma'(rB_{\eta})\gamma$. Moreover $|\varphi'_y| \leq 1$ and $| \varphi''_y| \leq 1/4$, so the conclusions of Lemma~\ref{lem:F.1_derivative_bounds} hold with $A_1$ and $A_2$ replaced by $1$ and $1/4$ respectively, with no dependence upon $\tau$. This is the case that applies to our experiments. \\[2pt] 
\emph{Dominated curvature.} The second sufficient condition follows by directly controlling the contribution of $g''$. We have 
\[
    \mathbb{E}_{y \mid x}[\varphi''_y] = (g')^2 \biggl[ \frac{\eta}{g^2} + \frac{1-\eta}{(1-g)^2} \biggr] + g'' \frac{g - \eta}{g(1-g)} \geq \frac{c^2}{(1-\tau)^2} - \frac{L_{g'} (1 - 2 \tau)}{\tau (1 - \tau)}.
\]
Indeed, $\eta(x), g(t) \in [\tau, 1- \tau]$ imply $|g - \eta| \leq 1- 2\tau$ and $g(1-g) \geq \tau (1- \tau)$; the expression in brackets is a convex combination of $1/g^2$ and $1/(1-g)^2$, each at least $(1-\tau)^{-2}$; and the bound then follows from $g' \geq c$ and $|g''| \leq L_{g'}$. It is positive whenever $c^2 > L_{g'} (1- 2\tau)(1 - \tau)/\tau$, in which case condition 5 holds with $\alpha$ equal to the displayed lower bound times $\gamma$. This condition also accommodates links that are not $\log$-concave, provided their second derivatives are sufficiently small relative to the lower bound on $g'$. \\[2pt]
In both cases, the resulting $\alpha$ is proportional to $\gamma$, consistent with Section~\ref{sec:F_4_identifiability}.

\subsection{Scope of the pre-deployment data}
\label{sec:F_6_Pre_deploy}
Assumption~\ref{ass:singleIndex} ~posits the single-index form on all of $\mathcal{X}$, hence on $\mathcal{A}$, and this premise is not testable from pre-deployment data; we do not claim otherwise. The analysis shows that no additional untestable assumption about off-support extrapolation is required. Conditional on Assumption~\ref{ass:singleIndex}~, the estimator and the target share the same link $g$ and differ off support only through the parameter $\widehat{w} - w^{\star}$, so the accuracy of $\widehat{\eta}$ at reported points is controlled entirely by $\| \widehat{w} - w^{\star} \|_2$ \textemdash{ } a quantity estimated from $S_{\eta} \subseteq \mathcal{X}_0$ (Lemmas~\ref{lem:F_3_Uniform_gradient_deviation}~and~\ref{lem:F_5_paramter_estimation}) and governed by condition 5, whose identifiability content is itself verifiable on the pre-deployment sample (Section~\ref{sec:F_4_identifiability}). The price of extrapolation is explicit and finite: the factor $L_g r$ in Theorem~\ref{thm:F_6_oracle_free_guarantee}~, with $r$ radius of the report set $\mathcal{A}$ rather than $\mathcal{X}_0$. Direct validation of the label model would instead require observing post-response outcomes, which the oracle-free setting posits are unavailable; the improvement rates reported in Section~\ref{sec:exp} therefore do not constitute an independent test of Assumption~\ref{ass:singleIndex}.

\section{Proxy for Improvement-Aware classifier}
\label{app:proxy}
\Proxy*
\begin{proof}
Let   $Z=w^{\star\top}X$. By Corollary~\ref{cor:f_s_relation with f*}, the strategic-optimal classifier is
parallel to the Non-strategic Bayes optimal classifier and has threshold
 $    b_s^\star=b^\star+\Delta $ where $
    \Delta:=\beta\max_{i\in[d]}\frac{w_i^\star}{\alpha_i}.$ 
 Thus we have 
  \[    f^\star(x)=\mathds 1\{w^{\star\top}x\ge b^\star\}  ~\text{and}~
    f_s^\star(x)=\mathds 1\{w^{\star\top}x\ge b^\star+\Delta\}. \]
Note that for any classifier of the form 
 $f_b(x)=\mathds 1\{w^{\star\top}x\ge b\}$,  
an agent with score \(z=w^{\star\top}x\) can reach the acceptance region if and only if $z+\Delta\ge b.$ Under the feasibility condition over \(\mathcal A\) and the canonical minimum-cost positive best-response rule, an agent with original score \(z={w^\star}^\top x\) either remains at \(x\) if \(z<b-\Delta\) or \(z\ge b\), and otherwise moves to a minimum-cost accepted report in \(\mathcal A\) with score exactly \(b\). Therefore the improvement-aware error of \(f_b\) as 
\begin{align*}
\texttt{err}_{\mathrm{Imp}}(f_b) &= \int_{-\infty}^{b-\Delta} g(z)\,dP_Z(z) \\
&\quad+ \int_{b-\Delta}^{b} \bigl(1-g(b)\bigr)\,dP_Z(z) \\
&\quad+ \int_{b}^{\infty} \bigl(1-g(z)\bigr)\,dP_Z(z).
\end{align*}
Let
$(I) \coloneqq \texttt{err}_{\mathrm{Imp}}(f_{b^\star}) - 
 \texttt{err}_{\mathrm{Imp}}(f_{b^\star+\Delta})$. Then,
\begin{align*}
(I) &=
\int_{[b^\star-\Delta,b^\star)}
    \Bigl(1-g(b^\star)-g(z)\Bigr)\,dP_Z(z) \\
& \quad +
\int_{[b^\star,b^\star+\Delta)}
    \Bigl(g(b^\star+\Delta)-g(z)\Bigr)\,dP_Z(z) \\ 
&\ge 
\int_{b^{\star} - \Delta}^{b^{\star}} \left( \frac{1}{2} - g(z) \right) dP_Z(z).
\end{align*}
By Assumption~\ref{ass:singleIndex}, \(g\) is continuous and non-decreasing and
\(b^\star=\inf\{t:g(t)\ge 1/2\}\) is finite; hence \(g(b^\star)=1/2\).  Also, since
\(g\) is non-decreasing, for every \(z<b^\star\), we have 
 $   g(z)\le g(b^\star)=1/2,
 $
and for every \(z\in[b^\star,b^\star+\Delta)\), 
 $   g(z)\le g(b^\star+\Delta).$ Hence both integrands are nonnegative, and therefore
$    \texttt{err}_{\mathrm{Imp}}(f_s^\star)
    \le
    \texttt{err}_{\mathrm{Imp}}(f^\star).$
\end{proof}
\begin{remark}
    Reducing the integral lower bound in the proof of Theorem~\ref{thm:proxy} to an explicit closed-form scalar rate would require a margin-type condition \citep{audibert2007fast} governing how quickly $g(z)$ moves away from $1/2$ near $b^{\star}$, together with a lower bound on the density $P_Z(z)$. We do not impose such conditions, as Theorem~\ref{thm:proxy} is designed to hold under Assumptions~\ref{ass:singleIndex} and~\ref{assump: cost function}.
\end{remark}

\section{Experiments: Supplementary Details}
\label{app:exp}


\paragraph{Tie-breaking.} When several coordinates attain $\max_{j \in [d]} w_j / \alpha_j$, the agent is indifferent among them, both in cost and in the classifier's decision. Moving along the coordinate $j$ at such a tie raises the qualification index by $w^{\star}_{j} \frac{(b-w^{\top}x)}{w_j} = \beta \frac{b-w^{\top}x}{S(w)} \cdot w^{\star}_j / \alpha_j$, an amount proportional to $w^{\star}_j / \alpha_j$, as $\eta(x) = g({w^{\star}}^{\top}x)$, so we index features in decreasing order of $w^{\star}_j / \alpha_j$: the rule $j^{\star} = \min \arg \max _{j \in [d]} w_j / \alpha_j$ then selects, among the equally cost-effective coordinates, the one yielding a largest improvement probability. In the plug-in algorithm, when $w^{\star}$ is unknown, we instead order features by $\widehat{w}_j / \alpha_j$, where $\widehat{\eta}(x) = g(\widehat{w}^{\top}x)$ is the estimate of $\eta$, fitted on an independent sample, and keep this ordering fixed throughout training and evaluation.

\paragraph{Training.} Because $q$ depends on $(w,b)$, it is recomputed at every step and gradients propagate through it. For a manipulating point, $x^f = x + \tfrac{b-w^{\top}x}{w_{j^{\star}}}e_{j^{\star}}$ (so that $w^{\top}x^f = b$), which gives $\partial x^f_{j^{\star}}/\partial b = 1 / w_{j^{\star}}$ and $\nabla_wx^{f}_{j^{\star}} = -x^f / w_{j^{\star}}$ (note $w_{j^{\star}} > 0$); the gradient of $q$ then follows by the chain rule through $\widehat{\eta}$, which we therefore require to be differentiable. The discrete components of Algorithm~\ref{alg:sim_oracle}  (the coordinate $j^{\star}$ and whether an agent manipulates) are piecewise constant in $(w,b)$ and are held fixed when differentiating, which gives the exact gradient on each region where they are constant. With bounded inputs and bounded input derivatives of $\widehat{\eta}$ on reachable responses, loss gradients within these regions are bounded uniformly in $(w,b)$ for fixed $\beta, \vec{\alpha}, \delta_g > 0$: the movement branch, the bound $ |\partial \ell / \partial q| \leq 2 M(x) \leq 2 \beta w_{j^{\star}}/\alpha_{j^{\star}}$ cancels the $1/w_{j^{\star}}$ factor; elsewhere, $q=y$ locally.  Across manipulation edges the loss is continuous: as $b - w^{\top}x \to 0^+$, $x^f \to x$ and $q$ is continuous, while $b-w^{\top}x = S(w)$ the shifted score $M(x) = 0$, where the two hinge terms in~(\ref{eq:exp_strat_imp_loss}) coincide. Across the tie hyperplanes $\{w_j/\alpha_j = w_k / \alpha_k\}$, by contrast, $j^{\star}$ switches and manipulating agents move along a different coordinate, so $q$, and with it the loss, can jump. This discontinuity is inherited from the agents' best response, since an agent at a tie is indifferent between coordinates that lead to different outcomes, and $\texttt{err}_{\mathrm{Imp}}$ itself jumps across the same hyperplanes; it is confined to at most $d(d-1)/2$ hyperplanes that do not depend on the data and the best-validation selection in Algorithm~\ref{alg:train_loop} guards against degradation from crossing them. Differentiating through $q$ distinguishes $\textsc{Strat-Imp-Aware}$ from stochastic methods for performative prediction \citep{mendler2020stochastic}, which hold the induced data fixed within each update and target performatively stable points; we instead minimize the improvement-aware objective itself, including its dependence on the labels the model induces. Each update is a projected (sub)gradient step onto $\mathbb{R}^d_{\geq0} \times \mathbb{R}$ (only $w$ is clipped, as required by Theorem~\ref{thm:strategic-shift}). After each epoch, improvement probabilities for $\mathcal{D}_{val}$ are computed under the current model, and the model with the lowest validation loss is returned. Algorithms~\ref{alg:sim_oracle} and~\ref{alg:strat_imp_error} are applied row-wise to a batch, and the per-point losses are averaged over the batch; this batch mean is $\mathcal{L}_{batch}$ in training and the per-batch increment to $\mathcal{L}_{val}$ in validation, so $\mathcal{L}_{val}$ is the average batch loss over $\mathcal{D}_{val}$. The operator $(\cdot)_+$ is applied entrywise in Algorithm~\ref{alg:train_loop}.

All experiments are produced by a single configuration-driven pipeline. Every
dataset-specific choice (label, features, split, agent environment, and
training hyperparameters) is stored in one YAML file per dataset
($\texttt{configs/retiring\_adult.yaml}$, $\texttt{configs/law\_school.yaml}$,
$\texttt{configs/heloc.yaml}$, $\texttt{configs/adult.yaml}$,
$\texttt{configs/synthetic.yaml}$); sweeps and comparison tables are configured
in $\texttt{configs/sweeps/}$. All outputs (splits, models, logs, tables) are
written to $\texttt{outputs/<dataset>/}$. Throughout, ``ACSIncome'' refers to the
dataset configured as $\texttt{retiring\_adult}$.

\subsection{Experimental Setup and Environment}
\label{app:setup}


All experiments were executed on an isolated compute node running \texttt{Ubuntu 20.04.6 LTS}. The machine is equipped with an \texttt{Intel(R) Xeon(R) Silver 4210R CPU @ 2.40GHz} and \texttt{256~GB} of system RAM. The GPU was not used: the code contains no
CUDA or MPS device placement and all tensors reside on the CPU; there is
therefore no dedicated VRAM requirement. 

\paragraph{Software.}
Python 3.11.15 with
PyTorch 2.3.1 (CPU), scikit-learn 1.4.2, NumPy 1.26.4, pandas 2.2.3,
SciPy 1.17.1, Matplotlib 3.10.9, PyYAML 6.0.3 and joblib 1.5.3. Our algorithm
and the frozen logistic regression are implemented in PyTorch in
\texttt{float64}; the neural/linear baselines of
\cite{attias2025pac} and \textsc{SERM}~\citep{levanon2022generalized} are
trained in \texttt{float32}; the logistic regression and the linear SVM use
scikit-learn.

\paragraph{Data partition (45/45/5/5).}
Each dataset is partitioned once, at random and without stratification, into
four disjoint subsets with fractions $0.45/0.45/0.05/0.05$
(\texttt{pipeline/data.py}, function \texttt{random\_split}; driver
\texttt{stage1\_split\_and\_profile.py}). A single permutation of the row
indices is drawn with \texttt{numpy.random.default\_rng(42)}; the first three
subsets receive $\lfloor 0.45n\rfloor$, $\lfloor 0.45n\rfloor$ and
$\lfloor 0.05n\rfloor$ rows and the last subset receives the remainder. The
row indices of every subset are stored in
\texttt{outputs/<dataset>/splits/split\_indices.json} together with the SHA-256
hash of the source file (\texttt{split\_summary.json}), and all later stages
read these fixed subsets. The roles are:
\begin{itemize}
  \item \texttt{lr\_train} (45\%): fits the logistic regression
        $\widehat{\eta}$ used as the differentiable estimate of
        $\Pr(y=1\mid x)$ in Algorithm~2; it is never seen by any classifier
        under comparison.
  \item \texttt{algo\_train} (45\%): trains \textsc{Strat-Imp-Aware} and all
        baselines.
  \item \texttt{algo\_val} (5\%): computes the validation loss used by
        Algorithm~1 for checkpoint selection, and is used for environment
        selection and for selecting baseline hyperparameters.
  \item \texttt{algo\_test} (5\%): unseen test set on which all reported
        numbers are computed.
\end{itemize}


\begin{table}[h]
\label{tab:splits}
\centering
\small
\begin{tabular}{lrrrrr}
\toprule
Dataset & $n$ & \texttt{lr\_train} & \texttt{algo\_train} & \texttt{algo\_val} & \texttt{algo\_test} \\
\midrule
ACSIncome  & 195{,}665 & 88{,}049 (0.408) & 88{,}049 (0.412) & 9{,}783 (0.411) & 9{,}784 (0.412) \\
Law School &  20{,}798 &  9{,}359 (0.891) &  9{,}359 (0.889) & 1{,}039 (0.868) & 1{,}041 (0.909) \\
HELOC      &   9{,}682 &  4{,}356 (0.477) &  4{,}356 (0.494) &   484 (0.488) &   486 (0.471) \\
Adult      &  48{,}842 & 21{,}978 (0.239) & 21{,}978 (0.239) & 2{,}442 (0.243) & 2{,}444 (0.239) \\
Synthetic  &  20{,}000 &  9{,}000 (0.423) &  9{,}000 (0.423) & 1{,}000 (0.433) & 1{,}000 (0.407) \\
\bottomrule
\end{tabular}
\caption{Subset sizes after preprocessing (from \texttt{outputs/<dataset>/splits/split\_summary.json}). $\Pr(y{=}1)$ in parentheses.}
\end{table}

\paragraph{Validity of the modelling assumptions.}
The following assumptions of \textsc{Strat-Imp-Aware} were checked empirically. 

\begin{enumerate}
  \item \emph{Monotone, non-negative linear structure ($w\in\mathbb{R}^d_{\ge 0}$).}
  For every dataset, features were oriented so that larger values are
  favourable, and the fitted logistic regression has strictly positive
  coefficients on all selected features
  (\texttt{outputs/<dataset>/models/logistic\_regression\_meta.json}).
  Standardized coefficients: ACSIncome (AGEP, WKHP, SCHL) $=(0.690, 0.945, 1.235)$;
  Law School (lsat, ugpa, zgpa) $=(0.624, 0.232, 1.305)$;
  HELOC (ExternalRiskEstimate, $-$NetFractionRevolvingBurden,
  $-$PercentTradesWBalance, PercentTradesNeverDelq) $=(0.807, 0.231, 0.124, 0.199)$;
  Adult (age, hours-per-week, educational-num) $=(0.651, 0.545, 0.892)$;
  Synthetic $(x_1,\dots,x_5)=(1.496, 1.200, 0.880, 0.610, 0.257)$.
  Consequently $\widehat{\eta}$ is increasing along every admissible
  improvement direction, so the improvement probability of Algorithm~2 is
  non-negative before clipping.
  \item \emph{Quality of the estimate $\widehat{\eta}$.}
  The logistic regression is calibrated on every split (mean predicted
  probability within $0.03$ of the empirical positive rate) with test AUC
  $0.834$ (ACSIncome), $0.844$ (Law School), $0.752$ (HELOC), $0.799$ (Adult)
  and $0.862$ (Synthetic). Per-feature profiles of $\widehat{\eta}$ against the
  empirical positive rate are saved in \texttt{outputs/<dataset>/plots/}
  (\texttt{stage3\_plot\_lr\_profiles.py}). They reveal that the linear logit
  over-predicts $\Pr(y{=}1)$ for age $>60$ and hours worked $>60$ on ACSIncome and
  Adult; all other profiles agree with the empirical rates.
  \item \emph{Exactness of the improvement probabilities (synthetic data).}
  On the synthetic dataset, the true $\eta(x)=\sigma(w^{\star\top}x+b^{\star})$
  is known. The fitted logistic regression recovers
  $w^\star=(1.5,1.2,0.9,0.6,0.3)$, $b^\star=-0.5$ as
  $(1.504, 1.189, 0.888, 0.607, 0.252)$, $-0.508$; on \texttt{algo\_test},
  $\mathrm{mean}|\widehat{\eta}-\eta|=0.006$, and for the agents that move,
  $\mathrm{mean}|\widehat{\pi}_{\mathrm{Imp}}-\pi_{\mathrm{Imp}}|=0.0017$
  (maximum $0.0068$, correlation $0.9998$)
  (\texttt{validate\_true\_eta.py},
  \texttt{outputs/synthetic/evaluation/true\_eta\_validation\_algo\_test.json}).
  The expected improvement rate is $0.8740$ with $\widehat{\eta}$ versus
  $0.8733$ with $\eta$, and the expected strategic improvement error is $0.0628$
  versus $0.0630$.
  \item \emph{Feasible best responses.} An agent moves only if the boundary
  point of its cheapest coordinate lies within the largest value of that
  feature observed in \texttt{lr\_train} (Section~\ref{sec:structural}).
  
  \item \emph{Correctness of the implementation.} The PyTorch copy of the
  logistic regression reproduces scikit-learn probabilities exactly (maximum
  absolute difference $0$). Agents that move land on the decision boundary
  (maximum $|w^\top x^f-b|=1.7\times10^{-16}$). The autograd gradient of the
  batch loss, including the path through $q$, matches central finite differences
  (step $10^{-6}$) to $9.0\times10^{-11}$.
  \item \emph{Positive agents keep their label.} In both training
  ($q=y+(1-y)\widehat{\pi}_{\mathrm{Imp}}$, \texttt{pipeline/strat\_imp.py})
  and evaluation (\texttt{pipeline/evaluation.py}, function
  \texttt{sample\_post\_labels}), only agents with $y=0$ can change label. It
  was verified on the saved per-agent records that no agent with $y=1$ is
  assigned a post-response label of $0$.
\end{enumerate}


\begin{table}[h]
\centering
\small
\setlength{\tabcolsep}{5pt}
\renewcommand{\arraystretch}{1.2}
\renewcommand{\theadfont}{\small\bfseries}
\begin{tabular}{C{2cm}C{2cm}C{2cm}C{2cm}C{2.2cm}}
\toprule
\textbf{Dataset} 
&  \textbf{Instances}
&  \thead{Total\\features} 
&  \thead{Imp.\\features} 
&  \textbf{Task}  \\
\midrule 
Adult 
& 48{,}842 
& 15 
& 5 
& Income $>$ 50K  \\

HELOC 
& 10{,}459 
& 24 
& 8 
& Credit risk  \\

Law School 
& 20{,}798 
& 12 
& 6 
& Bar exam   \\

ACS Income 
& 1,664,500 
& 10 
& 3 
& Income $>$ 50K  \\

Synthetic 
& 20{,}000 
& 8 
& 8 
& $+$ve label  \\
\bottomrule
\end{tabular}
\vspace{4mm}

\caption{Datasets used in experiments.}
\label{tab:datasets}
\end{table}




\subsection{Comparision with benchmark algorithms}
\label{sec:compare_algos}

Both strategic baselines were re-implemented from their public code and
papers (no third-party code was executed), so that every method is trained on
the same \texttt{algo\_train} rows and the same standardized features, and is
evaluated on the same \texttt{algo\_test} rows with the same agent model.

\paragraph{\citet{attias2025pac}.} 

We implement a baseline for improvement aware classifier following the framework of PAC Learning with Improvements \citep{attias2025pac}. The method models strategic agents as individuals who iteratively improve their feature vectors in order to obtain a positive classification from the deployed model.

\medskip
\noindent
\emph{Original setup:}
\begin{itemize}
  \item \emph{Classifier.} A two-layer network
        $d\!\to\!64\!\to\!\mathrm{ReLU}\!\to\!1\!\to\!\mathrm{sigmoid}$,
        trained with Adam (learning rate $10^{-3}$) for 100 epochs with batch size
        64 and no shuffling.
  \item \emph{Risk-averse loss.} A false-positive-weighted binary cross-entropy
        \[
          \mathcal{L}=-\tfrac{1}{n}\sum_i\big[w_{\mathrm{FP}}(1-y_i)\log(1-\hat y_i)
          +w_{\mathrm{FN}}\,y_i\log\hat y_i\big],
        \]
        with $w_{\mathrm{FP}}=2.0$, $w_{\mathrm{FN}}=1.33$ (default) and
        predictions clamped to $[10^{-6},1-10^{-6}]$.
  \item \emph{Decision rule.} Positive iff $h(x)>\tau$, with $\tau=0.9$ as the
        risk-averse default.
  \item \emph{Agent model and labels.} Agents respond by projected signed-gradient
        ascent in an $\ell_\infty$ ball, and labels come from a decision tree
        fitted on all data.
\end{itemize}

\medskip
\noindent
\emph{Our adaptation} (\texttt{pipeline/baselines\_pli.py}).
\begin{itemize}
  \item The architecture, loss, weights, optimizer, epochs, batch size and seed
        handling are reproduced exactly.
  \item The agent model and labels are replaced by ours. Agents best-respond to
        the decision rule $h(x)>\tau$ with the cost-based feasible response of
        Section~\ref{sec:structural} (computed numerically, since $h$ is non-linear),
        and post-response labels are drawn from $\widehat{\pi}_{\mathrm{Imp}}$.
  \item \emph{Threshold selection (HELOC and Adult only).} With $\tau=0.9$, the
        network never predicts a positive and no agent can reach the threshold:
        the maximum score on HELOC \texttt{algo\_test} is $0.864$, and at the
        best attainable feature profile it is $0.830$. For these two datasets we
        additionally report the same trained network with $\tau$ chosen on
        \texttt{algo\_val} from the threshold grid of \cite{attias2025pac},
        $\{0.5,0.7,0.8,0.9\}$, by minimal expected strategic improvement error.
        This selects $\tau=0.8$ on HELOC and $\tau=0.7$ on Adult
        (\texttt{experiment\_multiseed\_table.py}, option
        \texttt{select\_thresholds}).
  \item \emph{Other datasets.} For Synthetic, Law School and ACSIncome, the
        paper default $\tau=0.9$ is reported. On Synthetic, validation selection
        also returns $\tau=0.9$.
\end{itemize}

\paragraph{SERM \citep{levanon2022generalized}:}
We implement another strategic learning baseline that explicitly accounts for feature manipulation under $\ell_2$ costs. 

\medskip
\noindent
\emph{Original setup:} 
\begin{itemize}
  \item \emph{Classifier.} A linear model $f(x)=w^\top x+b$, positive iff
        $f(x)\ge 0$.
  \item \emph{Strategic hinge loss.}
        \[
          \tfrac{1}{n}\sum_i\max\{0,\,1-y_if(x_i)-2\,y_i\|w\|_2\}+\lambda\|w\|_2^2,
          \qquad y_i\in\{-1,+1\},
        \]
        i.e. users with $\ell_2$ cost and budget 2 who all prefer a positive
        prediction ($r=+1$).
  \item \emph{Optimization.} Adam with learning rate $0.05$ and
        $\lambda\in\{0.01,0.1,1\}$.
\end{itemize}

\medskip
\noindent
\emph{Our adaptation} (\texttt{pipeline/baselines\_gsc.py}).
\begin{itemize}
  \item The loss is reproduced exactly (shift $2\|w\|_2$).
  \item \emph{Training.} 100 epochs with shuffled mini-batches of size 128 and
        Adam (learning rate $0.05$); the model after the final epoch is returned.
  \item \emph{Regularization.} $\lambda\in\{0.1,1\}$ is selected per run on
        \texttt{algo\_val} by minimal expected strategic improvement error
        (\texttt{configs/sweeps/<dataset>/multiseed\_table.yaml},
        \texttt{configs/sweeps/gsc\_baseline.yaml}). $\lambda=0.1$ was selected
        in every run reported in Table~1 (in the additional Adult seeds 0 and
        1 the balanced variant selected $\lambda=1$).
  \item \emph{Class weighting (Law School and Adult only).} The hinge term of
        sample $i$ is multiplied by $n/(2n_{y_i})$, which averages to one. This
        is required because the unweighted objective collapses on these two Law school and Adult datasets.
  \item A variant whose shift equals the reach of our cost model,
        $\beta\max_j|w_j|/\alpha_j$, was also evaluated. It is stored in the
        result files but not reported in the main table.
\end{itemize}

\paragraph{Linear SVM.}
\texttt{sklearn.svm.LinearSVC} with hinge loss, $C=1.0$,
\texttt{max\_iter}$=200{,}000$ and \texttt{random\_state} equal to the run seed.
For Law School and Adult, \texttt{class\_weight="balanced"}
(\texttt{experiment\_svm\_baseline.py}, \texttt{experiment\_multiseed\_table.py}).

\subsection{Datasets and Preprocessing}
\label{app:data}

All loaders are implemented in \texttt{pipeline/data.py} (functions
\texttt{load\_raw}, \texttt{resolve\_feature\_types},
\texttt{prepare\_dataset}).

\paragraph{Feature typing.}
Each column is classified automatically as follows:
\begin{itemize}
  \item non-numeric or binary columns are \emph{nominal};
  \item integer-valued columns are \emph{continuous} if they have more than 30
        unique values and \emph{ordinal} otherwise (\texttt{max\_discrete\_unique: 30}).
\end{itemize}
Nominal features are
removed (\texttt{drop\_types: [nominal]}), and the logistic regression and all
classifiers use only the selected continuous and ordinal features. For features whose increase lowers the probability of a positive outcome, we take their negation.

\paragraph{Cleaning, normalization, label.}
\begin{itemize}
  \item \emph{Missing values.} Configured missing-value codes are mapped to
        \texttt{NaN}, and rows with a missing value in any selected feature are
        dropped before the split.
  \item \emph{Label.} The target is binarized as $y=\mathds{1}[\text{target}=\text{positive class}]$.
  \item \emph{Standardization.} Features are standardized with the mean and
        standard deviation of \texttt{lr\_train} (\texttt{StandardScaler} inside
        the logistic-regression pipeline, \texttt{pipeline/lr\_model.py}). The
        same affine map $u=(x-\mu)/\sigma$ is applied to all splits and used by
        every method (\texttt{FrozenLR} in \texttt{pipeline/strat\_imp.py}), so
        that costs $\alpha_j$ are expressed per standard deviation.
\end{itemize}

\subsection{Training Pipeline and Optimization}
\label{app:training}

\paragraph{Stages.} The pipeline consists of five scripts, each driven by
\texttt{--config configs/<dataset>.yaml}.
\begin{enumerate}
  \item \texttt{stage1\_split\_and\_profile.py}: feature typing, cleaning and
        the fixed 45/45/5/5 split.
  \item \texttt{stage2\_train\_lr.py}: logistic regression on
        \texttt{lr\_train}; coefficients on the standardized and the original
        scale, and metrics on every split.
  \item \texttt{stage3\_plot\_lr\_profiles.py}: for each selected feature, the
        mean $\pm$ one standard deviation of $\widehat{\eta}$ over the rows
        sharing each observed value (quantile bins if there are more than 100
        unique values), overlaid with the empirical positive rate.
  \item \texttt{stage4\_train\_strat\_imp.py}: Algorithm~\ref{alg:train_loop} on
        \texttt{algo\_train} with checkpoint selection on \texttt{algo\_val}.
  \item \texttt{stage5\_improvement\_rate.py}: improvement statistics on
        \texttt{algo\_test}.
\end{enumerate}
Comparison tables are produced by \texttt{experiment\_multiseed\_table.py}
(Law School, HELOC, Adult, Synthetic) and by \texttt{experiment\_svm\_baseline.py},
\texttt{experiment\_pli\_paper.py} and \texttt{experiment\_gsc\_baseline.py}
(ACSIncome).

\paragraph{Implementation of Algorithms 1--3} (\texttt{pipeline/strat\_imp.py},
\texttt{pipeline/strat\_imp\_clipped.py}).
\begin{itemize}
  \item \emph{Frozen estimate.} The logistic regression is copied into a
        PyTorch module whose parameters are registered as buffers. It therefore
        receives no updates, but gradients propagate through its input $x^f$.
  \item \emph{Algorithm~2 (vectorized).}
    \begin{itemize}
      \item The improvement direction is $j^\star=\min\arg\max_j w_j/\alpha_j$
            (first maximal index).
      \item Agents with $0<b-w^\top x\le S(w)$ move to
            $x^f=x+\frac{b-w^\top x}{w_{j^\star}}e_{j^\star}$.
      \item \emph{Feasible mode} additionally requires
            $x_{j^\star}+\frac{b-w^\top x}{w_{j^\star}}\le \bar x_{j^\star}$, the
            largest value of feature $j^\star$ in \texttt{lr\_train} (in
            standardized units).
      \item The improvement probability is
            $\widehat{\pi}_{\mathrm{Imp}}=\mathrm{clamp}_{[0,1]}\big((\widehat{\eta}(x^f)-\widehat{\eta}(x))/\max\{1-\widehat{\eta}(x),\delta_g\}\big)$.
      \item The branch condition, the choice of $j^\star$ and saturated clamps
            carry no gradient. A safe divisor avoids \texttt{NaN} gradients when
            $w_{j^\star}=0$.
    \end{itemize}
  \item \emph{Algorithm~3.} The per-sample loss is
        $q(1-s)_+ + (1-q)(1+s)_+$ with $q=y+(1-y)\widehat{\pi}_{\mathrm{Imp}}$ and
        $s=w^\top x-b+S(w)$. The gradient flows through $q$.
  \item \emph{Algorithm~1.} Each step updates $w\leftarrow(w-\gamma\nabla_w)_+$
        and $b\leftarrow b-\gamma\nabla_b$ (plain SGD, no momentum), with
        training batches reshuffled every epoch by a seeded generator.
  \item \emph{Validation loss and checkpointing.} The validation loss is the
        mean of per-batch mean losses over \texttt{algo\_val}, and the
        $(w,b)$ with the lowest validation loss is returned.
\end{itemize}

\paragraph{Class imbalance and the need for class-weighted baselines.}
Unweighted hinge-type objectives are minimized by a constant classifier on the
two imbalanced datasets.
\begin{itemize}
  \item \emph{Law School} ($89\%$ positive). The unweighted linear SVM and
        \textsc{SERM} converge to $w=0$ and accept every agent. On
        \texttt{algo\_train}, the hinge loss of the accept-all classifier
        ($0.2229$) is not improved by any rescaling of the logistic-regression
        direction, and doubling the training data
        ($18{,}718$ rows) does not change this.
  \item \emph{Adult} ($24\%$ positive). The unweighted linear SVM and
        \textsc{SERM} converge to the reject-all classifier.
\end{itemize}
In both cases no agent moves and the strategic improvement error equals the
minority-class share ($0.0913$ and $0.2390$, respectively). Balanced class
weights restore non-trivial boundaries, and these variants are reported (marked
$^\dagger$ in Table~\ref{tab:merged_cls}).

\paragraph{Variance and repeated runs.}
Table~\ref{tab:merged_cls} reports training seed 42 and label-sampling seed 42 for every method.
The following additional analyses quantify variability.
\begin{itemize}
  \item \emph{Label-sampling variance.} Stage~5 redraws the post-response
        labels 1{,}000 times and reports mean $\pm$ standard deviation, e.g.
        Synthetic improvement rate $0.873\pm0.022$ and Law School $0.645\pm0.053$
        (\texttt{outputs/<dataset>/evaluation/improvement\_rate\_algo\_test.json}).
  \item \emph{Training-seed variance.} Every method was retrained with seeds
        $\{42,0,1\}$ on Law School and Adult
        (\texttt{configs/sweeps/law\_school/multiseed\_table.yaml},
        \texttt{configs/sweeps/adult/multiseed\_table\_3seeds.yaml}). The seed
        affects initialization, batch order and label sampling. Mean $\pm$ sample
        standard deviation (\texttt{ddof=1}) of the strategic improvement error:
    \begin{itemize}
      \item \emph{Law School:} \textsc{Strat-Imp-Aware} $4.39\pm0.34\%$;
            \cite{attias2025pac} $5.41\pm0.11\%$; balanced SVM $5.35\pm0.29\%$;
            balanced \textsc{SERM} $27.76\pm2.03\%$.
      \item \emph{Adult:} \textsc{Strat-Imp-Aware} $15.89\pm0.24\%$; balanced
            \textsc{SERM} $15.52\pm0.14\%$; \cite{attias2025pac}
            ($\tau=0.7$) $17.85\pm0.05\%$; balanced SVM $63.09\pm0.61\%$.
    \end{itemize}
  \item \emph{Sample size.} On ACSIncome, Algorithm~\ref{alg:train_loop} was trained on nested
        random subsets of \texttt{algo\_train} of size
        $\{50,200,1000,3000,6000,88049\}$ with 10 repetitions each
        (\texttt{experiment\_improvement\_error.py},
        \texttt{configs/sweeps/improvement\_error.yaml}). The expected strategic
        improvement error on \texttt{algo\_test} decreases monotonically:
        $0.170\pm0.028$, $0.159\pm0.015$, $0.106\pm0.008$, $0.083\pm0.007$,
        $0.077\pm0.005$ and $0.064\pm0.0001$.
\end{itemize}
No cross-validation was used; all model and environment selection relies on
\texttt{algo\_val}, and \texttt{algo\_test} is used only for reporting.

\paragraph{Reproduction.}
For a dataset \texttt{<d>}, run in order:
\texttt{stage1\_split\_and\_profile.py}, \texttt{stage2\_train\_lr.py},
\texttt{stage3\_plot\_lr\_profiles.py}, \texttt{stage4\_train\_strat\_imp.py} and
\texttt{stage5\_improvement\_rate.py}, each with \texttt{--config configs/<d>.yaml}.
Then run
\texttt{experiment\_multiseed\_table.py --sweep configs/sweeps/<d>/multiseed\_table.yaml}
(Law School additionally with \texttt{--from-runs --seeds 42}).
Environment sweeps are run with
\texttt{experiment\_clipped\_improvement.py --sweep configs/sweeps/<d>\_env.yaml}.

\section{Improvement Probability Formulation} 
\label{sec:imp_probability_formulation}
Let $x \sim \mathcal{D}, y \mid x \sim  \text{Bernoulli}(\eta(x))$. We model improvement-aware strategic classification, such that the following condition holds: 
\[
    \mathbb{P}(y^{f} = 1 \mid x) = \eta(x^f)
\]
where $y^f$ denotes the realized label after moving to point $x^f$ through strategic manipulation, further:
\[
   \mathbb{P}(y^{f} = 1 \mid y = 1, x) = 1, \hspace{1cm} \mathbb{P}(y^{f} = 1 \mid y = 0, x) = p^f(x) \in [0,1]
\]
Now, by the law of total probability 
\begin{align*}
    \mathbb{P}(y^{f} = 1 | x) &= \mathbb{P}(y^{f} = 1, y =1 | x) + \mathbb{P}(y^{f} = 1, y = 0 | x)\\
    &= \mathbb{P}(y^f = 1 \mid y =1, x) \mathbb{P}(y =1 \mid x) + \mathbb{P}(y^f = 1 \mid y = 0, x) \mathbb{P}(y = 0 \mid x)\\
    &= \eta(x) + p^f(x) \cdot (1- \eta(x)) 
\end{align*}
From our modelling assumption, it follows that the improvement probability is \( p^f(x) =  \frac{\eta(x^f) - \eta(x)}{1 - \eta(x)} \). When $\eta(x) = 1$ the event $\{y=1\}$ has conditional probability one and $p^f(x)$ is set to $0$; since $x^f \geq x$ coordinate wise under Assumption~\ref{assump: cost function} and the tie-breaking rule of Section~\ref{sec:structural}, Assumption~\ref{ass:singleIndex} gives $\eta(x^f) \geq \eta(x)$, thus $\eta(x^f) = 1$, and the identity $\eta(x^f) = \eta(x) + p^f(x)(1-\eta(x))$ continues to hold. To simulate the improvement probability in Algorithm~\ref{alg:sim_oracle}, we utilize $p^f(x)$ as the probability that an unqualified agent $(y=0 \mid x)$ successfully flips their label, provided they strategically manipulated. The post-manipulation $y^{f}$ is determined by a Bernoulli trial with parameter $p^f(x)$. Formally: 
\[
    y^{f} | (y, x=0) \sim \text{Bernoulli}(p^f(x))
\]
This ensures that label flipping is intrinsically tied to the strategic transition from $x \to x^f$, maintaining the probabilistic consistency of the improvement-aware framework.

\section{Additional Theoretical Results}

    \begin{restatable}{proposition}{PropositionOne} Consider a scoring function $h(x) = \sum_{i=1}^d \alpha_i \psi(x_i)$ with $\alpha \in \mathbb{R}^d_{>0}$. Let $c_{sep}$ and $c_{dec}$ denote separable and decomposable cost functions with respect to $h$ and $\psi$, respectively. Let the feature transformation functions $\psi_i $ are  non-decreasing i.e. if $x_i'\geq x_i$ coordinate-wise then  $\emph{i.e.,}$ $\psi_i(x'_i) \geq \psi_i(x_i)$ for all $i \in [d]$. Then, for any $x, x' \in \mathcal{X}$, the costs satisfy $c_{sep}(x, x') = c_{dec}(x, x')$ .
\end{restatable}

\begin{proof}
    Let $x, x' \in \mathcal{X}$. Define the feature-wise deviation $\delta_i = \psi(x'_i) - \psi(x_i)$. 
    
    Recall the definition of the decomposable cost:
    \[
        c_{dec}(x, x') = \sum_{i=1}^d \alpha_i | \psi(x'_i) - \psi(x_i) |_+ = \sum_{i=1}^d \alpha_i \max(0, \delta_i)
    \]
    Similarly, recall the definition of the separable cost, substituting the scoring function structure: 
    \begin{align*}
        c_{sep}(x, x') &= \max \left(0, \sum_{i=1}^d \alpha_i \psi(x'_i) - \sum_{i=1}^d \alpha_i \psi(x_i)\right) \\
        &= \max \left( 0, \sum_{i=1}^d \alpha_i \delta_i \right)
    \end{align*}
    We now invoke the monotonicity condition, then $\psi(x'_i) \geq \psi(x_i)$ for all $i \in [d]$. This implies $\delta_i \geq 0$ for all $i$. \\
    Under this condition, we can clearly conclude that: \[
        c_{dec}(x, x') = \sum_{i=1}^d \alpha_i \delta_i = c_{sep}(x, x') 
    \]    
\end{proof}

\begin{restatable}{observation}{increasingX}\label{lem:x'_i>=x_i}  
Suppose that, for every $x \in \mathcal{X}_0$ and $z \in \mathcal{A}$, every point lying coordinate-wise between $x$ and $z$ belongs to $\mathcal{A}$. Let \(f_{w,b}\in\mathcal F\) be a linear classifier with \(w\in\mathbb R^d_{\ge0}\) and \(b\in\mathbb R\). Suppose a strategic best response exists for every $x \in \mathcal{X}_0$. Then, for every $x \in \mathcal{X}_0$, there exists a strategic best response $x^f$ satisfying $x^f_i \geq x_i$ for all $i \in [d]$. 
\end{restatable}

\begin{proof}
    For brevity, let $f = f_{w,b}$. Fix $x \in \mathcal{X}_0$, let $z \in \mathcal{A}$ be any strategic best response to $f$. Define $\tilde{z} \in \mathbb{R}^d$ by 
    \[
        \tilde{z}_i = \max \{x_i, z_i \}, \qquad i \in [d]
    \]

    \medskip
    \noindent
    \emph{$\tilde{z} \in \mathcal{A}$}: For each $i$, $\tilde{z}_i \in \{ x_i, z_i \}$, so $\tilde{z}_i$ lies between $x_i$ and $z_i$. Hence $\tilde{z}$ lies coordinate-wise between $x$ and $z$, and by the assumption on $\mathcal{A}$, $\tilde{z} \in \mathcal{A}$.

    \medskip
    \noindent
    \emph{$\tilde{z}$ has the same cost as $z$}: If $z_i \ge x_i$, then $\tilde{z}_i = z_i$, so $(\tilde{z}_i - x_i)_+ = (z_i - x_i)_+$. If $z_i < x_i$, then $\tilde{z}_i = x_i$, and both $(\tilde{z}_i - x_i)_+$ and $(z_i - x_i)_+$ equal to $0$. Thus $(\tilde{z}_i - x_i)_+ = (z_i - x_i)_+$ for every $i \in [d]$, and therefore 
    \begin{equation}
        \label{eq:same_cost_best_response}
        c(x, \tilde{z}) = \sum_{i=1}^d \alpha_i (\tilde{z}_i - x_i)_+  = \sum_{i=1}^d \alpha_i (z_i - x_i)_+ = c(x,z)
    \end{equation}

    \medskip
    \noindent 
    \emph{$\tilde{z}$ is classified as least as favorably as $z$}: Since $\tilde{z} \geq z$ coordinate-wise and $w \geq 0$, 
    \begin{equation}
        \label{eq:pred_compare_best_response} 
        w^{\top} \tilde{z} \geq w^{\top}z, \quad \text{hence } \quad f(\tilde{z}) \geq f(z)
    \end{equation}
    \medskip
    \noindent 
    \emph{$\tilde{z}$ is a best response}: Using $\beta > 0$ together with (\ref{eq:same_cost_best_response}) and (\ref{eq:pred_compare_best_response}), 
    \[
        u_f(x, \tilde{z}) = \beta f(\tilde{z}) - c(x, \tilde{z}) \geq \beta f(z) - c(x,z) = u_f(x,z).
    \]
    Since $z$ maximizes $u_f(x, \cdot)$ over $\mathcal{A}$ and $\tilde{z} \in \mathcal{A}$, the point $\tilde{z}$ is also a maximizer, $\emph{i.e.}$, a strategic best response. \\ 
    Finally, $\tilde{z}_i = \max \{x_i, z_i\} \geq x_i$ for all $i \in [d]$ by construction. Taking $x^f = \tilde{z}$ proves the claim.
\end{proof}

\end{document}